\documentclass{article}

\usepackage{layouts}
\usepackage{authblk}
\usepackage{graphicx}
\usepackage{amssymb,amsmath, amsthm,mathtools,amsfonts}

\usepackage[marginratio=1:1,height=600pt,width=460pt,tmargin=100pt]{geometry}

\usepackage{layout, color}
\usepackage{enumerate}
\usepackage{algorithm}
\usepackage{algpseudocode}
\usepackage{setspace}
\usepackage{enumitem}
\usepackage{array}
\usepackage{multirow}
\usepackage{hyperref}
\usepackage{url}
\usepackage{mdwlist}
\usepackage{multirow}
\usepackage{amsthm}
\usepackage[font=small,labelfont=bf]{caption}

\usepackage[labelfont=scriptsize,labelfont=scriptsize]{subcaption}

\usepackage{xcolor}
\usepackage{bbm}
\usepackage{graphicx}
\usepackage{cleveref}
\usepackage{float}
\usepackage{wrapfig}

\newtheorem{theorem}{Theorem}[section]
\newtheorem{proposition}[theorem]{Proposition}
\newtheorem{assumption}{Assumption}
\newtheorem{corollary}[theorem]{Corollary}
\newtheorem{lemma}[theorem]{Lemma}
\newtheorem{definition}[theorem]{Definition}

\theoremstyle{remark}
\newtheorem{remark}{Remark}

\newtheorem{example}{Example}
\newtheorem{conjecture*}{Conjecture}

\theoremstyle{plain}

\DeclareMathOperator*{\argmax}{arg\,max}

\newcommand{\R}{\mathbb{R}}
\newcommand{\E}{\mathbb{E}}

\newcommand{\CM}{\mathcal{M}}
\newcommand{\CN}{\mathcal{N}}
\newcommand{\CX}{\mathcal{X}}

\newcommand{\CP}{\mathcal{P}}

\newcommand{\norm}[1]{\left\lVert#1\right\rVert}
\newcommand{\abs}[1]{\left |#1\right |}
\newcommand*\diff{\mathop{}\!\mathrm{d}}

\graphicspath{ {./figures/} }

\title{Robust Generative Learning with Lipschitz-Regularized $\alpha$-Divergences Allows Minimal Assumptions on Target Distributions}

\begin{document}

\author[1]{Ziyu Chen\thanks{Email: ziyuchen@unc.edu}}
\author[1]{Hyemin Gu}
\author[1]{Markos A. Katsoulakis}
\author[1]{Luc Rey-Bellet}
\author[2]{Wei Zhu}
\affil[1]{Department of Mathematics and Statistics, University of Massachusetts Amherst}
\affil[2]{School of Mathematics, Georgia Institute of Technology}

\date{}

\maketitle

\begin{abstract}
This paper demonstrates the robustness of Lipschitz-regularized $\alpha$-divergences as objective functionals in generative modeling, showing they enable stable learning across a wide range of target distributions with minimal assumptions. We establish that these divergences remain finite under a mild condition—that the source distribution has a finite first moment—regardless of the properties of the target distribution, making them adaptable to the structure of target distributions. Furthermore, we prove the existence and finiteness of their variational derivatives, which are essential for stable training of generative models such as GANs and gradient flows. For heavy-tailed targets, we derive necessary and sufficient conditions that connect data dimension, $\alpha$, and tail behavior to divergence finiteness, that also provide insights into the selection of suitable $\alpha$'s. We also provide the first sample complexity bounds for empirical estimations of these divergences on unbounded domains. As a byproduct, we obtain the first sample complexity bounds for empirical estimations of these divergences and the Wasserstein-1 metric with group symmetry on unbounded domains. Numerical experiments confirm that generative models leveraging Lipschitz-regularized $\alpha$-divergences can stably learn distributions in various challenging scenarios, including those with heavy tails or complex, low-dimensional, or fractal support, all without any prior knowledge of the structure of target distributions.
\end{abstract}


\section{Introduction}
In generative modeling, the goal is to create new samples that resemble those from an unknown data distribution by designing algorithms that minimize a probability divergence or metric between the generated distribution and the target distribution. However, the diverse characteristics of real-world data distributions—such as heavy tails, low-dimensional structures, manifold constraints, or fractal-like supports—introduce significant challenges in the training of generative models. These challenges are manifested as instabilities, reduced robustness, and a need for specialized architectures, as standard generative frameworks struggle to adapt to complex data structures. Addressing these issues is essential for developing models that are not only accurate but also robust across a wide range of scenarios for the target distribution.

Features such as heavy-tailed distributions arise in various fields, including extreme events in ocean waves \cite{PhysRevE.100.033110}, floods \cite{merz2022understanding}, social sciences \cite{pareto1964cours, klebanov2023heavy}, human activities \cite{lotka1926frequency, zipf2016human}, biology \cite{lynn2024heavy}, and computer science \cite{sasaki2017programs}. Learning to generate heavy-tailed distributions has been explored with Generative Adversarial Networks (GANs). However, GANs based on Integral Probability Metrics (IPMs), such as the Wasserstein-1 metric, may struggle to learn these distributions without additional tail estimation strategies  \cite{feder2020nonlinear, huster2021pareto, allouche2022ev}. This limitation arises because the Wasserstein-1 metric between two distributions becomes infinite when one lacks a finite first moment, and accurately estimating tail behavior often requires extensive data from that tail, which may be difficult to obtain. Consequently, capturing discrepancies between distributions with a metric that remains finite, is stable to compute, and is less sensitive to the need for extensive tail data is essential for stable and effective learning.

%While estimating the tail behavior of a heavy-tailed distribution is important, selecting objectives that measure discrepancies between these distributions and facilitate stable learning is equally crucial.

On the other hand, many empirical results suggest that real-world data, such as images, exhibit low-dimensional structures \cite{pope2020intrinsic}. 
While there are theoretical guarantees for GANs to learn distributions with low-dimensional support \cite{liang2021well, huang2022error}, recent works on flow-based models, such as continuous normalizing flows (CNFs), neural ODEs, and score-based diffusion models, often rely on density assumptions \cite{chensampling, marzouk2024distribution}. These models can struggle to learn low-dimensional structures without additional regularization or specific architectures, such as autoencoders (see \Cref{sec:numerical}). This limitation arises because their performance is typically evaluated using the Kullback-Leibler (KL) or $f$-divergences, which require absolute continuity between probability measures. Thus, it is crucial to select a divergence that remains flexible and inherently compatible with the structure of the data distribution.

In this work, we demonstrate that the Lipschitz-regularized $\alpha$-divergence, as proposed in \cite{dupuis2022formulation, birrell2020f}, is a suitable objective functional for generative modeling with minimal assumptions on the target distribution, denoted by $Q$ from now on. First, we revisit the definition of the Lipschitz-regularized $\alpha$-divergence between two distributions $P$ and $Q$ defined as:
\begin{align}
	\label{eq:variational_formula}
	D_{\alpha}^L(P\|Q) \coloneqq \sup_{\gamma\in\text{Lip}_L(\mathbb{R}^d)}\left\{\E_{P}[\gamma]-\E_Q[f_\alpha^*(\gamma)]\right\},
\end{align}
where $\text{Lip}_L(\mathbb{R}^d)$ is the class of $L$-Lipschitz functions on $\R^d$; see more details in \Cref{sec:background}. 
In particular, we show that the Lipschitz-regularized $\alpha$-divergences are suitable for stably learning a broad range of distributions from three perspectives:

\begin{itemize}
	\item \textbf{Finiteness.} The  objective of generative modeling using \eqref{eq:variational_formula} can be formulated as  $\min_\theta D_\alpha^L(P_\theta\|Q)$, where $P_\theta$ is the generated distribution parametrized by $\theta$ and $Q$ is the target distribution. Thus, the divergence needs to be finite. On the contrary, an infinite or large divergence value can be an indicator of the divergence of an algorithm (see \Cref{tab:divergence:values} in \Cref{append:divergence_value}). We prove that these divergences remain finite whenever the generated distribution has a finite first moment, with no assumptions necessary on the target distribution $Q$. 
	When both distributions have power-law-decay densities, we provide sufficient and necessary conditions for the divergences to be finite. Notably, the Lipschitz-regularized KL divergences require minimal assumptions on both the tails of the generated and the target distributions. 
	
	\item \textbf{Existence of variational derivatives.} To find the optimal parameter  $\theta$ in the optimization $\min_\theta D_\alpha^L(P_\theta\|Q)$, one often uses gradient-based algorithms. Formally, the gradient of $D_\alpha^L(P_\theta\|Q)$ in terms of $\theta$ can be evaluated as
	\begin{equation}
		\label{eq:chain_rule}
		\nabla_\theta D_\alpha^L(P_{\theta}\| Q) = \int \frac{\delta D_\alpha^L(P\|Q)}{\delta P}(P_{\theta}(x)) \cdot \nabla_\theta P_{\theta}(x) \, \mathrm{d}x,
	\end{equation}
	therefore it is essential that the variational derivative $\frac{\delta D_\alpha^L(P\|Q)}{\delta P}$ is well-defined. We prove that these divergences have well-defined variational derivatives for any target distribution $Q$, given $P$ has a finite first moment. This is a crucial property for stable optimizations in generative learning and the associated gradient flows, and it illustrates that algorithms using this class of divergences can stably learn distributions without extensive prior knowledge of the tail behavior or density formulation of the target. In contrast, those using divergences without Lipschitz regularization generally can fail to learn (see \Cref{sec:numerical}).
	
	\item \textbf{Convergence of empirical estimations.} As distributions are only accessible through their finite samples, it is important to know how fast the divergence between their empirical measures converges to the true value of the divergence. We prove the first result of empirical estimations of this class of divergences on $\R^d$, and as a byproduct of the proof, we offer the first sample complexity bounds for empirical estimations of the Lipschitz-regularized $\alpha$-divergences and the Wasserstein-1 metric with group invariance on $\R^d$ with sub-Weibull assumptions. The key to these results is the Lipschitz regularization, without which we cannot prove such bounds.
\end{itemize}

The rest of the paper is organized as follows. We review and discuss some related work in \Cref{sec:relatedwork}. \Cref{sec:background} provides background and motivation for the proposed divergences. Finiteness results including the variational derivatives and their gradient flow for the Lipschitz-regularized $\alpha$-divergences are presented in \Cref{sec:theory}. \Cref{sec:finite-sample} provides the first convergence rate for finite-sample estimations of these divergences in $\R^d$. Based on the results and proofs from \Cref{sec:finite-sample}, in \Cref{sec:group_symmetry}, we provide the first sample complexity bounds for empirical estimations of the Lipschitz-regularized $\alpha$-divergences and the Wasserstein-1 metric with group symmetry in $\R^d$. Numerical experiments are detailed in \Cref{sec:numerical} including synthetic heavy-tailed distributions, distributions on a low-dimensional manifold, real keystroke data, and trajectories from the attractor of the Lorentz system, which is known to exhibit fractal properties. Finally, we conclude this paper in \Cref{sec:conclusion}.

\section{Related work}\label{sec:relatedwork}
\textbf{Generative models for heavy-tailed distributions.} 
Although heavy-tailed distributions are common, there are few results to date in their generative modeling, primarily using GANs. For example, \cite{wiese2020quant} generates heavy-tailed financial time series data by logarithmically transforming the data and then exponentiating the output, which produces distributions whose tails follow lognormal asymptotic rather than distributions with power-law tails considered in our paper. In a different approach, GANs are used for cosmological analysis \cite{feder2020nonlinear}, sharing a similarity with Pareto GANs \cite{huster2021pareto} in their use of a heavy-tailed latent variable. However, both papers require accurate estimations of the tail decay rate for each
marginal distribution. %However, our proposed $W_p$-$\alpha$-divergences allow significant flexibility of the input for the learning process to be stable as explained by our theoretical results.
EV-GANs \cite{allouche2022ev} use neural network approximations of the quantile function to encode the tail decay rate in an asymptotic sense, which is essentially also a tail estimation approach. We note that the focus of our work is to devise appropriate divergences as objective functionals for comparing and learning heavy-tailed distributions stably, \textit{without} prior knowledge of the tail behavior.

\noindent\textbf{Generative models for distributions with low-dimensional structures.} In \cite{liang2021well,huang2022error} it is rigorously shown that IPM-GANs are able to learn distributions with low-dimensional support. There are some other generative models that learn high-dimensional distributions from the low-dimensional latent space provided by auto-encoders \cite{vincent2008extracting,mcclelland1987parallel}, such as Bidirectional GANs \cite{donahue2016adversarial}, Variational Auto-Encoders \cite{kingma2013auto} and Generalized Denoising Auto-Encoders \cite{bengio2013generalized}. However, it is not clear if the low-dimensional latent space matches the low-dimensional structure of the data distribution and no convergence guarantees have been provided, and these results are largely empirical.

\noindent\textbf{Empirical estimations of divergences.}
\cite{rubenstein2019practical, moon2014multivariate, sreekumar2021non} estimate $f$-divergences using various assumptions and estimators, and \cite{ding2023empirical} considers in particular the $\alpha$-divergences. However, these studies either make additional structural assumptions or consider light tails or without establishing a convergence rate of the estimation. Recently, \cite{mena2019statistical, manole2024sharp} studied the convergence rate of entropic optimal transport and optimal transport with smooth costs. While our proof of the convergence rate of the empirical estimations of the Lipschitz-regularized $\alpha$-divergences is inspired by these works, the structure inherited from the $\alpha$-divergences in our study requires different, non-trivial treatment due to the nonlinear and asymmetric variational form, particularly as we consider even heavier tails. When the distributions are invariant to some group actions, \cite{chen2023sample} shows that empirical estimations of the Lipschitz-regularized $\alpha$-divergences and the Wasserstein-1 metric enjoy a faster convergence using symmetry-informed estimators on bounded domains of $\R^d$, and later \cite{tahmasebisample} extends the result to closed Riemannian manifolds with group symmetry only for Sobolev-IPMs that are symmetric.

\noindent\textbf{Lipschitz-regularized divergences.} The class of Lipschitz-regularized $f$-divergences was first proposed in \cite{dupuis2022formulation} in the context of Lipschitz-regularized KL-divergences with its first variation formula, under the assumptions that both the source and the target distributions have finite first moments. Later, \cite{birrell2020f} generalized it to the class of Lipschitz-regularized $f$-divergences and observed that GANs optimizing Lipschitz-regularized $f$-divergences outperform those optimizing either the Wasserstein-1 metric or the $f$-divergences in learning heavy-tailed distributions. 
In \cite{gu2022lipschitz}, under the assumption that $Q$ has a finite first moment, the gradient flows of the Lipschitz-regularized $\alpha$-divergences were introduced, using the variational derivatives to define a corresponding generative particle algorithm, outperforming other generative models in scarce and high-dimensional data regimes. In this paper, we provide the first theoretical explanations, not only for learning heavy-tailed distributions but also for learning distributions with manifold or fractal support, essentially making the generative modeling agnostic to the target data assumptions.

\section{Background}\label{sec:background}
Let $\CP(\R^d)$ be the space of probability measures on $\R^d$. A map $D:\CP(\R^d)\times\CP(\R^d)\to[0,\infty]$ is called a \textit{divergence} on $\CP(\R^d)$ if
\begin{equation}
	D(P,Q) = 0 \iff P = Q \in \CP(\R^d),
\end{equation}
hence providing a notion of ``distance'' between probability measures. In particular, the class of $\alpha$-divergences \cite{amari2000methods,havrda1967quantification}, denoted by $D_\alpha$, which is a sub-class of $f$-divergences \cite{csiszar1963information}, is defined as
\begin{equation}\label{eq:alphadiv}
	D_\alpha(P\|Q)\coloneqq \int_{\mathbb{R}^d}f_\alpha\left(\frac{\diff{P}}{\diff{Q}}\right)\diff{Q}, \quad \text{if}~P\ll Q,
\end{equation}
where $f_\alpha(x) = \frac{x^\alpha-1}{\alpha(\alpha-1)}$, with $\alpha > 0$ and $\alpha\neq 1$, and $P\ll Q$ means $P$ is absolutely continuous with respect to $Q$. When $P$ is not absolutely continuous with respect to $Q$, we write $D_\alpha(P\|Q)=\infty$. 
\begin{remark}
	Note that the $\alpha$-divergences can be equivalently defined as $D_\alpha(P\|Q) = \int \tilde{f}_\alpha(\frac{\diff{P}}{\diff{Q}})\diff{Q}$, where $\tilde{f}_\alpha(x)=\frac{x^\alpha-x}{\alpha(\alpha-1)}$ by noticing that $\int(f_\alpha-\tilde{f}_\alpha)(\frac{\diff{P}}{\diff{Q}})\diff{Q}=0$ for any $P\ll Q$. In the limiting case for $\tilde{f}_\alpha(x)$ when $\alpha\to 1$, we have $\lim_{\alpha\to 1} \frac{x^\alpha-x}{\alpha(\alpha-1)}=x\ln x$, recovering the Kullback–Leibler (KL) divergence. In this paper, we use $f_\alpha(x) = \frac{x^\alpha-1}{\alpha(\alpha-1)}$, and simply mean to replace $f_\alpha$ in \eqref{eq:alphadiv} by $f(x)=x\ln x$ whenever we refer to $\alpha=1$.
\end{remark}
The $\alpha$-divergence can be equivalently formulated in its dual form \cite{nguyen2010estimating,birrell2020f}  as
\begin{align}\label{eq:variation_alpha}
	D_\alpha(P\|Q) = \sup_{\gamma\in \mathcal{M}_b(\mathbb{R}^d)}\left\{\E_{P}[\gamma]-\E_Q[f_\alpha^*(\gamma)]\right\},
\end{align}
where $\mathcal{M}_b(\mathbb{R}^d)$ is the set of bounded measurable functions and $f_\alpha^*$ is the convex conjugate (Legendre transform) of $f_\alpha$,
\begin{align}\label{eq:falpha}
	f_\alpha^*(y) = \begin{cases} 
		\alpha^{-1}(\alpha-1)^{\frac{\alpha}{\alpha-1}}y^{\frac{\alpha}{\alpha-1}}\mathbf{1}_{y>0} + \frac{1}{\alpha(\alpha-1)}, & \alpha>1, \\
		\infty\mathbf{1}_{y\geq0}+\left(\alpha^{-1}(1-\alpha)^{-\frac{\alpha}{1-\alpha}}\abs{y}^{-\frac{\alpha}{1-\alpha}}-\frac{1}{\alpha(1-\alpha)}\right)\mathbf{1}_{y<0}, & \alpha\in(0,1).
	\end{cases}
\end{align}
Compared to \eqref{eq:variation_alpha}, the formulation of the Lipschitz-regularized $\alpha$-divergences in \eqref{eq:variational_formula} can be viewed as imposing Lipschitz regularization on the space of test functions in the variational form of $\alpha$-divergences. In our work, we focus on the case when $\alpha>1$ or $\alpha=1$ (corresponding to the KL divergence). It has been proved in \cite{birrell2020f} that the Lipschitz-regularized $\alpha$-divergence defined in \eqref{eq:variational_formula} has an equivalent primal formulation
\begin{equation}\label{def:W1proximal}
	D_{\alpha}^L(P\|Q)= \inf_{\eta\in\mathcal{P}(\mathbb{R}^d)}\{D_\alpha(\eta\|Q) + L\cdot W_1(P,\eta)\},
\end{equation}
where $W_1$ is the Wasserstein-1 metric. One can easily verify that $D_{\alpha}^L$ satisfies the conditions for being a divergence using \eqref{def:W1proximal}. \eqref{def:W1proximal} can be viewed as the infimal convolution between the $\alpha$-divergence and the Wasserstein-1 metric. Though \eqref{eq:variational_formula} is more often used in generative modeling as training objectives, its primal formulation is also theoretically very important.
For example, we have from \eqref{def:W1proximal} that 
\begin{equation}\label{eq:divergence_upperbound}
	D_{\alpha}^L(P\|Q)\leq \min\{D_\alpha(P\|Q) , L\cdot W_1(P,Q)\}.
\end{equation}
In practical tasks, such as in generative modeling, we estimate the divergence from finite samples of $P$ and $Q$, where the absolute continuity assumption in \eqref{eq:alphadiv} typically no longer holds. Meanwhile, $D_{\alpha}^L(P\|Q)$ is always finite if $P$ and $Q$ are discrete measures of finitely many points with possibly different support since $D_\alpha^L(P\|Q)\leq L\cdot W_1(P,Q)<\infty$ by \eqref{eq:divergence_upperbound}.

The following example shows that we can have a strict inequality in \eqref{eq:divergence_upperbound}.
\begin{example}\label{example1}
	Let $P$ and $Q$ be distributions on $\mathbb{R}$ such that 
	\begin{equation*}
		p(x) = (1+\delta)x^{-(2+\delta)}\mathbf{1}_{x\geq1},
		\quad
		q(x) = \frac{1}{2}\mathbf{1}_{0\leq x<1} + \frac{1}{x^2}\mathbf{1}_{x\geq2}.
	\end{equation*}
	Then neither $D_{\alpha}(P\|Q)$ nor $W_1(P,Q)$ is finite for any $\alpha>1,\delta>0$, while $D_{\alpha}^L(P\|Q)<\infty$.
\end{example}
\begin{proof}
	Since $P$ is not absolutely continuous with respect to $Q$, we have $D_{\alpha}(P\|Q) = \infty$; applying the cumulative distribution function formula for the 1-dimensional Wasserstein-1 distance, it is straightforward to see $W_1(P,Q) = \infty$ as $Q$ does not have a finite first moment. Consider the formula \eqref{def:W1proximal} and in particular, we design the intermediate probability measure as 
	\[d\eta = (1+\delta)2^{1+\delta}x^{-(2+\delta)}\mathbf{1}_{x\geq2}.\]
	Then we have 
	\begin{align*}
		D_{\alpha}(\eta\|Q) = \int_2^\infty \frac{(1+\delta)^\alpha 2^{\alpha(1+\delta)}x^{-\alpha\delta}-1}{\alpha(\alpha-1)}\cdot\frac{1}{x^2}\diff{x}<\infty,
	\end{align*}
	and
	\begin{align*}
		W_1(P,\eta) &= \int_{1}^{2}\int_{1}^y (1+\delta)x^{-(2+\delta)}\diff{x}dy\\
		&\quad + \int_2^\infty\abs{\int_{1}^y (1+\delta)x^{-(2+\delta)}\diff{x}-\int_{2}^y (1+\delta)2^{1+\delta}x^{-(2+\delta)}\diff{x}}dy\\
		&= \int_1^2 1-y^{-(1+\delta)}dy + \int_2^\infty\abs{(1-y^{-(1+\delta)})-(1-2^{1+\delta}y^{-(1+\delta)})}dy\\
		&= \int_1^2 1-y^{-(1+\delta)}dy + \int_2^\infty(2^{1+\delta}-1)y^{-(1+\delta)}dy<\infty.
	\end{align*}
	Therefore, $D_{\alpha}^L(P\|Q) \leq D_{\alpha}(\eta\|Q) + L\cdot W_1(P,\eta)<\infty.$
\end{proof}
\Cref{example1} is not a special example when $D_\alpha^L(P\|Q)$ is finite but neither $D_\alpha(P\|Q)$ nor $W_1(P,Q)$ is finite. In fact, $D_\alpha^L$ can be applied to much wider situations. As we will see in \Cref{thm:agnostic} and its proof, the Lipschitz regularization plays a key role. 

For the rest of the paper, we denote by $\CP_k(\R^d)$ the space of probability measures on $\R^d$ that have a finite $k$-th moment, $k\geq1$ and we assume that $k$ can be a non-integer; we also denote by $\CP_{<k}(\R^d)$ the space of probability measures on $\R^d$ that have a finite $s$-th moment for any $s<k$.

\section{Finiteness and variational derivatives of $D_\alpha^L$}
\label{sec:theory}
In generative modeling, the goal is to approximate a target data distribution $Q$ by a generated distribution $P_{g_\theta}$, where $g_\theta$ is typically a neural net parametrization. A specific divergence between the target and the generated distributions is often chosen as the loss function. We want to build the best approximation $P_{g_{\theta^*}}$ of $Q$ using the optimization of a probability divergence or metric:
\begin{equation}
	g_{\theta^*}=\arg\min_{g_\theta \in \mathcal{G}}D\left (P_{g_\theta},  Q\right ) \approx Q,
\end{equation}
where $\mathcal{G}$ is a family of neural nets with certain constraints on the parameters $\theta$. To optimize or minimize this loss, it is essential to ensure that the loss function or divergence is \textit{finite}. In \Cref{sec:minimal}, we first demonstrate that when $P$ has a finite first moment, $D_\alpha^L(P\|Q)$ remains finite without requiring any assumptions on $Q$. In \Cref{sec:finiteness}, assuming $P$ and $Q$ have densities and tails, we provide necessary and sufficient conditions for $D_\alpha^L(P\|Q)$ to be finite.

\subsection{Minimal assumptions on the target $Q$}\label{sec:minimal}
We make the following assumption on $P$ and $Q$ for this subsection.
\begin{assumption}\label{assumption0}
	Let $P$ and $Q$ be arbitrary probability measures on $\R^d$. In addition, we assume that $P$ has a finite first moment, that is $P\in\CP_1(\R^d)$. 
\end{assumption}
We show in \Cref{thm:agnostic} that $D_\alpha^L(P\|Q)$ is finite whenever $P\in\CP_1(\R^d)$ without any assumption on $Q$. This includes cases when $Q$ has heavy tails, even without a finite first moment, and when $Q$ is supported on a low-dimensional manifold and does not have a density. Before stating and proving the theorem, we need the following lemma for measures that are not necessarily probability measures that generalizes Lemma A.12 in \cite{chen2023sample}, and the proofs are the same in essence.

\begin{lemma}\label{lemma:bounded_div}
	For $\alpha>1$ and any non-negative measures $P$ and $Q$ defined on some bounded $\Omega\subset\mathbb{R}^d$ with non-zero integrals, $\Gamma = \text{Lip}_L(\Omega)$, we have 
	\begin{equation}\label{eq:gammaequalF}
		\sup_{\gamma\in\Gamma}\left\{\int_\Omega \gamma(x) \diff{P}-\int_\Omega f_\alpha^*[\gamma(x)]\diff{Q}\right\} = \sup_{\gamma\in\mathcal{F}}\left\{\int_\Omega \gamma(x) \diff{P}-\int_\Omega f_\alpha^*[\gamma(x)]\diff{Q}\right\},
	\end{equation}
	where 
	\[
	\mathcal{F} = \left\{\gamma\in\text{Lip}_{L}(\Omega):\norm{\gamma}_\infty\leq (\alpha-1)^{-1}\left(\frac{\int_\Omega \diff{P}}{\int_\Omega \diff{Q}}\right)^{\alpha-1} + L\cdot\text{diam}(\Omega) \right\}.
	\]
\end{lemma}
\begin{proof}[Proof of \Cref{lemma:bounded_div}]
	For any fixed $\gamma\in\Gamma$, define
	\[
	h(\nu) = \int_\Omega \left(\gamma(x)+\nu\right) \diff{P}-\int_\Omega f_\alpha^*[\gamma(x)+\nu]\diff{Q}.
	\]
	Since $\sup_{x\in\Omega}\gamma(x)- \inf_{x\in\Omega}\gamma(x)\leq L\cdot\text{diam}(\Omega)$, 
	interchanging the integration with differentiation is allowed by the dominated convergence theorem: 
	\[
	h'(\nu) = \int_\Omega \diff{P} -\int_\Omega f_\alpha^{*\prime}(\gamma+\nu)\diff{Q},
	\]
	where
	\begin{equation}\label{eq:falpha_derivative}
		f_\alpha^{*\prime}(y) = (\alpha-1)^{\frac{1}{\alpha-1}}y^{\frac{1}{\alpha-1}}\mathbf{1}_{y>0}.
	\end{equation}
	If $\inf_{x\in\Omega}\gamma(x)>(\alpha-1)^{-1}\left(\frac{\int \diff{P}}{\int \diff{Q}}\right)^{\alpha-1}$, then $h'(0)<0$. So there exists some $\nu_0<0$ such that $h(\nu_0)>h(0)$. This indicates the supremum on the left side of \eqref{eq:gammaequalF} is attained only if $\sup_{x\in\Omega}\gamma(x)\leq (\alpha-1)^{-1}\left(\frac{\int \diff{P}}{\int \diff{Q}}\right)^{\alpha-1} + L\cdot\text{diam}(\Omega)$. On the other hand, if $\sup_{x\in\Omega}\gamma(x)<0$, then there exists $\nu_0>0$ that satisfies $\sup_{x\in\Omega}\gamma(x)+\nu_0<0$ such that
	\begin{align*}
		\int_\Omega \left(\gamma(x)+\nu_0\right) \diff{P}-\int_\Omega f_\alpha^*[\gamma(x)+\nu_0]\diff{Q} &= \int_\Omega \left(\gamma(x)+\nu_0\right) \diff{P}\\
		&>\int_\Omega \gamma(x) \diff{P}\\
		&= \int_\Omega \gamma(x) \diff{P}-\int_\Omega f_\alpha^*[\gamma(x)]\diff{Q}.
	\end{align*}
	This indicates that the supremum on the left side of \eqref{eq:gammaequalF} is attained only if $\inf_{x\in\Omega}\gamma(x)\geq -L\cdot\text{diam}(\Omega)$. Therefore, we have that the supremum on the left side of \eqref{eq:gammaequalF} is attained only if $\norm{\gamma}_\infty\leq(\alpha-1)^{-1}\left(\frac{\int \diff{P}}{\int \diff{Q}}\right)^{\alpha-1} + L\cdot\text{diam}(\Omega)$.
\end{proof}

\begin{theorem}\label{thm:agnostic}
	Suppose $\alpha\geq1$ ($\alpha=1$ refers to the KL) and $P, Q$ satisfy \Cref{assumption0}, namely $P\in\CP_1(\R^d)$, then $D_\alpha^L(P\|Q)<\infty$.
\end{theorem}
The key is the Lipschitz regularization, without which the result will not be true; see the proof below. 
\begin{proof}
	We first prove the case when $\alpha>1$. Let $\Gamma = \text{Lip}_L(\mathbb{R}^d)$, and we have
	\begin{align*}
		D_{\alpha}^L(P\|Q) &= \sup_{\gamma\in\Gamma}\left\{\int \gamma(x) \diff{P}-\int f_\alpha^*[\gamma(x)]\diff{Q}\right\}\\
		&\leq \sup_{\gamma\in\text{Lip}_L(\norm{x}< R)}\left\{\int_{\norm{x}< R} \gamma(x)\diff{P}-\int_{\norm{x}< R} f_\alpha^*[\gamma(x)]\diff{Q}\right\}\\
		&\quad + \sup_{\gamma\in\text{Lip}_L(\norm{x}\geq R)}\left\{\int_{\norm{x}\geq R} \gamma(x) \diff{P}-\int_{\norm{x}\geq R} f_\alpha^*[\gamma(x)]\diff{Q}\right\}\\
		&\coloneqq I_1 + I_2.
	\end{align*}
	For $I_1$, by \Cref{lemma:bounded_div}, we have
	\begin{align*}
		I_1 &\leq C\int_{\norm{x}< R}\diff{P} + \left(\alpha^{-1}(\alpha-1)^{\frac{\alpha}{\alpha-1}}C^{\frac{\alpha}{\alpha-1}}+\alpha^{-1}(\alpha-1)^{-1}\right)\int_{\norm{x}< R}\diff{Q}<\infty,
	\end{align*}
	where $C = (\alpha-1)^{-1}\left(\frac{\int_{\norm{x}< R} \diff{P}}{\int_{\norm{x}< R} \diff{Q}}\right)^{\alpha-1} + 2LR$.
	
	Now we prove that $I_2<+\infty$. Let $M(\gamma) = \sup_{\norm{x}=R}\abs{\gamma(x)}$, where $\gamma\in\text{Lip}_L(\norm{x}\geq R)$. We show that there exists some $\overline{M}>0$ such that 
	\begin{equation}\label{eq:I2equivalent}
		I_2 = \sup_{\gamma\in\mathcal{G}}\left\{\int_{\norm{x}\geq R} \gamma(x) \diff{P}-\int_{\norm{x}\geq R} f_\alpha^*[\gamma(x)]\diff{Q}\right\},
	\end{equation}
	where
	\begin{equation}
		\mathcal{G} = \left\{\gamma\in\text{Lip}_L(\norm{x}\geq R):M(\gamma)\leq\overline{M} \right\}.
	\end{equation}
	Indeed, we have for any $\gamma\in\text{Lip}_L(\norm{x}\geq R)$,
	\begin{align*}
		&\int_{\norm{x}\geq R} \gamma(x) \diff{P}-\int_{\norm{x}\geq R} f_\alpha^*[\gamma(x)]\diff{Q}\\
		&= \int_{R\leq\norm{x}<2R}\gamma(x)\diff{P}-\int_{R\leq\norm{x}<2R}f_\alpha^*[\gamma(x)]\diff{Q}\\
		&\quad+ \int_{\norm{x}\geq2R}\gamma(x)\diff{P}-\int_{\norm{x}\geq2R}f_\alpha^*[\gamma(x)]\diff{Q}\\
		&\leq \int_{R\leq\norm{x}<2R}\gamma(x)\diff{P}-\int_{R\leq\norm{x}<2R}f_\alpha^*[\gamma(x)]\diff{Q} + \int_{\norm{x}\geq2R}\gamma(x)\diff{P}\\
		&\leq (M(\gamma)+LR)\int_{R\leq\norm{x}<2R}\diff{P} - \int_{R\leq\norm{x}<2R}f_\alpha^*(M(\gamma)-3LR)\diff{Q}\\
		&\quad+ \int_{\norm{x}\geq2R}\left(M(\gamma)+LR+L\norm{x}\right)\diff{P}\\
		&= LR\int_{\norm{x}\geq R}\diff{P} + L\int_{\norm{x}\geq 2R}\norm{x}\diff{P} + M(\gamma)\int_{\norm{x}\geq R}\diff{P}\\
		&\quad -f_\alpha^*(M(\gamma)-3LR)\int_{R\leq\norm{x}<2R}\diff{Q},
	\end{align*}
	where the last inequality is due to the fact that $\gamma(x)$ is $L$-Lipschitz and that for any $x:\norm{x}\geq R$, we have $\abs{\gamma(x)-M(\gamma)}\leq L(R+\norm{x})$. The first two terms are finite and are independent of $\gamma$ since $P\in\CP_1(\R^d)$. For the difference between the last two terms, we have
	\[
	\lim_{M(\gamma)\to+\infty} M(\gamma)\int_{\norm{x}\geq R}\diff{P} -f_\alpha^*(M(\gamma)-3LR)\int_{R\leq\norm{x}<2R}\diff{Q} = -\infty,
	\]
	since the exponent of $x$ in $f_\alpha^*(x)$ is $\frac{\alpha}{\alpha-1}>1$. This indicates that the supremum in $I_2$ should be taken over $\gamma$ such that $M(\gamma)\leq\overline{M}$ for some $\overline{M}>0$. Therefore,
	\begin{align*}
		I_2 &= \sup_{\gamma\in\mathcal{G}}\left\{\int_{\norm{x}\geq R} \gamma(x) \diff{P}-\int_{\norm{x}\geq R} f_\alpha^*[\gamma(x)]\diff{Q}\right\}\\
		&\leq \sup_{\gamma\in\mathcal{G}}\int_{\norm{x}\geq R} \gamma(x) \diff{P}\\
		&\leq \sup_{\gamma\in\mathcal{G}}\int_{\norm{x}\geq R} \left(LR+L\norm{x}+M(\gamma)\right) \diff{P}\\
		&\leq \int_{\norm{x}\geq R} \left(LR+L\norm{x}+\overline{M}\right) \diff{P}<\infty.
	\end{align*}
	For $\alpha=1$, we bound $I_1$ using a similar \Cref{lemma:bounded_div_KL} in \Cref{appendix:agnostic}, and the bound for $I_2$ can be derived exactly in the same way as for $\alpha>1$ by replacing $f_\alpha^*$ by $f_{\text{KL}}^*$.
\end{proof}
\begin{remark}
	\Cref{lemma:bounded_div} and \Cref{thm:agnostic} indeed work for any Lipschitz-regularized $f$-divergences, if $f^*$, the convex conjugate of $f$, is bounded below and superlinear, i.e., $\lim_{x\to\infty}\frac{f^*(x)}{x}=\infty$.
\end{remark}
\begin{remark}
	\Cref{thm:agnostic} has important implications in generative modeling that one can learn a data distribution $Q$, without any prior knowledge of whether $Q$ has heavy tails (even without a finite first moment) or lies on a low-dimensional manifold such that $Q$ does not have a density, whenever $P$ has a finite first moment, which is a very weak assumption; for example, $P$ can start with the Gaussian which is very easy to sample from. In this sense, the generative learning task can be agnostic to the structure of the data distribution using Lipschitz-regularized $\alpha$-divergences as the objective functionals. 
\end{remark}
In what follows, we discuss the applicability of two generative models based on \Cref{thm:agnostic}. Their numerical implementations can be seen in several numerical examples in \Cref{sec:numerical}.

\paragraph{Lip-$\alpha$-GANs}GANs based on the Lipschitz-regularized $\alpha$-divergences, abbreviated as Lip-$\alpha$-GANs, can be formulated as
\begin{equation}\label{eq:LipalphaGAN}
	\inf_{g\in\mathcal{G}}D_{\alpha}^L(g_\sharp P\|Q) = \inf_{g\in\mathcal{G}}\sup_{\gamma\in\text{Lip}_L(\mathbb{R}^d)}\left\{\E_{g_\sharp P}[\gamma]-\E_Q[f_\alpha^*(\gamma)]\right\},
\end{equation}
where $P$ is the initial source distribution, typically chosen as a Gaussian, and $Q$ is the target data distribution, and $\mathcal{G}$ is the class of generators, and $g_\sharp P$ is the push-forward measure of $P$ by the map $g$. \Cref{thm:agnostic} informs us that we can learn any probability measure $Q$ if $g_\sharp P\in\CP_1(\R^d)$; for example, the generator can be realized using a ReLU network with a Gaussian source distribution as $P$. Key to obtaining the optimal generator is calculating the gradient of the loss relative to generator parameters, shown by the chain rule (Regarding the chain rule calculation \eqref{eq:chain_rule}, we also refer to a related formal calculation in Sec. 3.3 of \cite{mroueh_sobolev_2019}):
\begin{equation}
	\label{eq:chain_rule}
	\nabla_\theta D_\alpha^L(P_{g_\theta}\| Q) = \int \frac{\delta D_\alpha^L(P\|Q)}{\delta P}(P_{g_\theta}(x)) \cdot \nabla_\theta P_{g_\theta}(x) \, \mathrm{d}x,
\end{equation}
where $\frac{\delta D_\alpha^L(P\|Q)}{\delta P}$ is the variational derivative or the first variation of $D_\alpha^L(P\|Q)$, formally defined in \Cref{thm:firstvariation}. Therefore, even with a well-designed neural network architecture for the generator $g_\theta$, a robust and well-defined variational derivative  $\frac{\delta D_\alpha^L(P\|Q)}{\delta P}(P_{g_\theta}(x))$ is crucial for stable and effective optimization in the parameter $\theta$ because it directly impacts the parameter gradient $ \nabla_\theta D_\alpha^L(P_{g_\theta}\| Q)$ via \eqref{eq:chain_rule}, otherwise computing $\nabla_\theta D_\alpha^L(P_{g_\theta}, Q)$ could become unstable, leading to erratic parameter updates that hinder convergence. While GANs use discriminators rather than explicit variational derivatives, \Cref{thm:firstvariation} shows that the finiteness of a variational derivative can provide mathematical insight into GAN training. On the other hand, it is worth noting that, in light of \eqref{def:W1proximal}, $D_\alpha^L$ offers advantages over both $D_\alpha$ and $W_1$:
\begin{itemize}
	\item The variational derivative does not exist in general for the Wasserstein-1 metric alone (as is used in WGANs). For example, let \( P = \delta_{x_1} \) and \( Q = \delta_{x_2} \) be two Dirac delta distributions centered at points \( x_1 \) and \( x_2 \) in $\mathbb{R}$ with the usual distance function. Then the variational derivative in the sense of \Cref{thm:firstvariation} has a discontinuity:
	\[
	\frac{\partial}{\partial \epsilon} \left| x_1 - x_2 + \epsilon v \right| \Big|_{\epsilon=0} = 
	\begin{cases} 
		v, & \text{if } x_1 - x_2 > 0, \\ 
		-v, & \text{if } x_1 - x_2 < 0.
	\end{cases}
	\]
	
	\item Unregularized $f$-divergences (such as the KL-divergence) may yield large variational derivatives when \( P_{g_\theta} \) and \( Q \) do not overlap significantly, potentially causing gradient spikes. This instability can lead to large, uncontrolled updates in \( \theta \), which might result in mode collapse or oscillations in GAN training. In contrast, the Lipschitz-regularized $\alpha$-divergences always have well-defined variational derivatives by \Cref{thm:firstvariation}. For example, let $P=\CN(\mu_1,\sigma)$ and $Q=\CN(\mu_2,\sigma)$ be two univariate Gaussians with different means but the same variance. Then through a direct calculation, we have $D_{\text{KL}}(P\|Q) = \frac{-(\mu_1-\mu_2)^2}{2\sigma^2}$, so that $\frac{\diff{D_{\text{KL}}(P\|Q)}}{\diff{\mu_1}} = \frac{-(\mu_1-\mu_2)}{\sigma^2}$. We can think of $P$ and $Q$ do not overlap significantly if $\mu_1-\mu_2$ has a large magnitude and $\sigma$ is small, so that both $D_{\text{KL}}(P\|Q)$ and its derivative in $\mu_1$ will have a large magnitude.
\end{itemize}

\paragraph{Gradient flows of $D_\alpha^L$.}To further illustrate the significance of \Cref{thm:agnostic}, we provide perspectives from the Wasserstein gradient flows of $D_{\alpha}^L$ for a feasible distribution learning task. As a particular case of the Lipschitz-regularized gradient flows proposed in \cite{gu2022lipschitz}, the Lipschitz-regularized $\alpha$-divergences can be used to construct gradient flows of the form
\begin{equation}\label{eq:gradientflow}
	\partial_t P_t=\text{div}\left(P_t\nabla\frac{\delta D_{\alpha}^L(P_t\|Q)}{\delta P_t}\right),
\end{equation}
for an initial source probability measure $P_0$ and a target measure $Q$, where $\frac{\delta D_{\alpha}^L(P\|Q)}{\delta P}$ is the first variation of $D_{\alpha}^L(P\|Q)$, defined in \Cref{thm:firstvariation}. This type of gradient flows was inspired by the gradient flows in the 2-Wasserstein space of probability measures in \cite{jordan1998variational,otto2001geometry}. 
In \cite{gu2022lipschitz}, the first variation form of $D_{\alpha}^L(P\|Q)$ is proved under the assumption that both $P,Q\in\CP_1(\R^d)$. In \Cref{thm:firstvariation}, we extend it to the case when we only require $P\in\CP_1(\R^d)$ but impose no assumptions on $Q$. This corresponds to the condition in \Cref{thm:agnostic}. The key to the extension is our \Cref{lemma:supnorm} and the proof can be found in \Cref{appendix:firstvariation}.
\begin{theorem}\label{thm:firstvariation}
	Under \Cref{assumption0}, namely $P\in\CP_1(\R^d)$ and $Q$ can be any probability measure, we define
	\begin{equation}
		\gamma^\star\coloneqq \argmax_{\gamma\in\text{Lip}_L(\R^d)}\left\{\E_P[\gamma]-\E_Q[f_\alpha^*(\gamma)]\right\}, 
	\end{equation}
	where the optimizer $\gamma^\star\in\text{Lip}_L(\R^d)$ exists, and is defined on $\text{supp}(P)\cup\text{supp}(Q)$, and is unique. Subsequently, we can extend $\gamma^\star$ to all of $\R^d$ as $\hat{\gamma}$ with the same Lipschitz constant. Let $\rho$ be a signed measure of total mass 0 and let $\rho = \rho_+-\rho_-$, where both $\rho_{\pm}\in\CP_1(\R^d)$ are nonnegative and mutually singular. If $P+\epsilon\rho\in\CP_1(\R^d)$ for sufficiently small $\epsilon>0$, then
	\begin{equation}
		\lim_{\epsilon\to0}\frac{1}{\epsilon}\left(D_{\alpha}^L(P+\epsilon\rho\|Q)-D_{\alpha}^L(P\|Q)\right)=\int\hat{\gamma}\diff{\rho},
	\end{equation}
	and we write
	\begin{equation}
		\frac{\delta D_{\alpha}^L(P\|Q)}{\delta P}(P)=\hat{\gamma}.
	\end{equation}
\end{theorem}
As a result, \Cref{thm:firstvariation} provides a reformulation of \eqref{eq:gradientflow} as in \cite{gu2022lipschitz}:
\begin{equation}\label{eq:gradientflow2}
	\begin{split}
		&\partial_tP_t + \text{div}(P_tv_t^L)= 0,\quad P_0 = P\in\CP_1(\R^d),\\
		&v_t^L = -\nabla\gamma^\star_t, \quad \gamma^\star_t = \argmax_{\gamma\in\text{Lip}_L(\R^d)}\left\{\E_{P_t}[\gamma]-\E_Q[f_\alpha^*(\gamma)]\right\}.
	\end{split}
\end{equation}
Moreover, Theorem 2 in \cite{gu2022lipschitz} tells us that if $P_t$ is sufficiently smooth, then we have 
\begin{equation}
	\frac{\diff{}}{\diff{t}}D_{\alpha}^L(P_t\|Q) = -I_{\alpha}(P_t\|Q)\leq 0,
\end{equation}
where $I_{\alpha}(P_t\|Q)$ is the Lipschitz-regularized Fisher Information:
\begin{equation*}
	I_{\alpha}(P_t\|Q)\coloneqq\E_{P_t}[\abs{\nabla\gamma^\star_t}^2].
\end{equation*}
Then for any $T\geq0$, we have 
\begin{equation}
	\label{eq:fisher:ineq}
	D_{\alpha}^L(P_T\|Q) = D_{\alpha}^L(P_0\|Q) - \int_0^T I_{\alpha}(P_s\|Q)\diff{s} \le D_{\alpha}^L(P_0\|Q) \, .
\end{equation}
Therefore, both the finiteness and the variational derivative of $D_{\alpha}^L(P_0\|Q)$ are crucial for the divergence to dissipate from the gradient flow perspective. While the convergence of the gradient flow is also important, we do not address its PDE theory in this work, but rather its feasibility to learn any distribution $Q$.

\subsection{When $P$ and $Q$ have densities and heavy tails}\label{sec:finiteness}
In this subsection, we show that $D_\alpha^L$ is applicable to comparing heavy-tailed distributions, by providing necessary and sufficient conditions that relate the tail behaviors of $P$ and $Q$ with $\alpha$. This also provides insights into the selection of suitable $\alpha$'s. For this purpose, including cases when $P \notin \CP_1(\R^d)$ --compare to \Cref{thm:agnostic}--we make the following assumptions on $P$ and $Q$.
\begin{assumption}\label{assumption1}
	Let $P$ and $Q$ be distributions on $\mathbb{R}^d$ whose densities $p(x)$ and $q(x)$ are absolutely continuous with respect to the Lebesgue measure. However, $P$ and $Q$ are not necessarily absolutely continuous with respect to each other on some bounded subset.
\end{assumption}
\begin{definition}\label{def:heavytail}
	For a pair of distributions $(P,Q)$ on $\mathbb{R}^d$, we say they are of heavy-tail $(\beta_1,\beta_2)$, $\beta_1,\beta_2>d$, if there exists some $R>0$, such that
	\[
	p(x)\asymp\norm{x}^{-\beta_1}, \quad q(x)\asymp\norm{x}^{-\beta_2},
	\]
	for $\norm{x}\geq R$. That is, there exist constants $0<c_{p,1}\leq c_{p,2}$ and $0<c_{q,1}\leq c_{q,2}$ such that 
	\[
	c_{p,1}\norm{x}^{-\beta_1}\leq p(x)\leq c_{p,2}\norm{x}^{-\beta_1}, \quad c_{q,1}\norm{x}^{-\beta_2}\leq q(x)\leq c_{q,2}\norm{x}^{-\beta_2},
	\]
	for $\norm{x}\geq R$.
\end{definition}
Then we prove the following necessary and sufficient conditions on the tail behaviors of $(P, Q)$ for $D_\alpha^L(P\|Q)$ to be finite. The proof makes extensive use of the variational formula \eqref{eq:variational_formula} and Lipschitz regularization and is provided in \Cref{appendix:finiteness}.
\begin{theorem}[Necessary and sufficient conditions for $D_{\alpha}^L<\infty$, $\alpha>1$]
	\label{thm:finite}
	Suppose $\alpha>1$, and $(P,Q)$ are distributions on $\mathbb{R}^d$ of heavy-tail $(\beta_1,\beta_2)$. Then 
	$D_{\alpha}^L(P\|Q)<\infty$ if and only if one of the following two conditions holds:\\
	(i) $d<\beta_1\leq d+1$ and $\beta_2-\beta_1<\frac{\beta_1-d}{\alpha-1}$;\\
	(ii) $\beta_1>d+1$.
\end{theorem}
\begin{remark}
	We can relax the assumption in  \Cref{def:heavytail} to allow different tail behavior in different directions as follows. Let $\Omega_k$ be a finite partition of the spherical coordinates $[0,\pi]^{d-2}\times [0,2\pi)$, where each $\Omega_k$ has non-zero Lebesgue measure of $[0,\pi]^{d-2}\times [0,2\pi)$. We can assume that $p(x)\asymp\norm{x}^{-\beta_{1,k}}$ and $q(x)\asymp\norm{x}^{-\beta_{2,k}}$ on each $\Omega_k$. Then the $D_{\alpha}^L(P\|Q)<\infty$ if and only if $\beta_{1,k}$ and $\beta_{2,k}$ satisfy one of the conditions of \Cref{thm:finite} on each $\Omega_k$. The proof is the same as that of \Cref{thm:finite} constrained on each $\Omega_k$. This relaxation can be adopted in the same way for \Cref{thm:finite_KL} and \Cref{cor:lowdimension}.
\end{remark}
For the Lipschitz-regularized KL-divergence, we have the following result whose proof can be found in \Cref{appendix:finiteness}.
\begin{theorem}[Necessary and sufficient conditions for $D_{\text{KL}}^L<\infty$]\label{thm:finite_KL}
	Suppose $\alpha=1$ (the KL case), and $(P,Q)$ are distributions on $\mathbb{R}^d$ of heavy-tail $(\beta_1,\beta_2)$, then
	$D_{\text{KL}}^L(P\|Q)<\infty$ for any $\beta_1,\beta_2>d$.
\end{theorem}
\begin{remark}\label{remark:KL}
	Since $\beta_1,\beta_2>d$ are the minimal assumptions for $P$ and $Q$ to be probability distributions, \Cref{thm:finite_KL} suggests that using the Lipschitz-regularized KL-divergence is the most robust choice, as it can be agnostic to both the tails of $P$ and $Q$, compared to the conditions in \Cref{thm:finite}.
\end{remark}
In cases where both $P$ and $Q$ lie on a low-dimensional submanifold, we have the following corollary. The proof can be found in \Cref{appendix:finiteness}.
\begin{corollary}[Necessary and sufficient conditions on embedded submanifolds]\label{cor:lowdimension}
	Let $\CM$ be a $d^*$-dimensional smooth embedded submanifold of $\mathbb{R}^d$ via an $L^*$-Lipschitz embedding $\varphi:\mathbb{R}^{d^*}\to\mathbb{R}^d$ with $\CM = \varphi(\mathbb{R}^{d^*})$ for $d^*<d$. Suppose $(P,Q)$ are of heavy-tail $(\beta_1,\beta_2)$ on $\mathbb{R}^{d^*}$, and let $p_{\CM}$ and $q_{\CM}$ be their push-forward distributions on $\CM$, i.e., $p_{\CM} = p\circ\varphi^{-1}$ and $q_{\CM} = q\circ\varphi^{-1}$. Then the Lipschitz-regularized $\alpha$-divergence between $p_{\CM}$ and $q_{\CM}$, defined as
	\begin{equation*}
		D_{\alpha}^L(p_{\CM}\|q_{\CM}) = \sup_{\gamma\in\text{Lip}_L(\mathbb{R}^d)}\left\{\E_{p_{\CM}}[\gamma]-\E_{q_{\CM}}[f_\alpha^*(\gamma)]\right\},
	\end{equation*} 
	is finite if and only if one of the following two conditions holds for $\alpha>1$:\\
	(i) $d^*<\beta_1\leq d^*+1$ and $\beta_2-\beta_1<\frac{\beta_1-d^*}{\alpha-1}$;\\
	(ii) $\beta_1>d^*+1$;
	
	\noindent and $D_{\alpha}^L(p_{\CM}\|q_{\CM})<\infty$ for any $\beta_1,\beta_2>d^*$ if $\alpha=1$.
\end{corollary}
\begin{remark}
	The Lipschitz condition on the embedding $\varphi$ is necessary to guarantee that the tails of $p_{\CM}$ and $q_{\CM}$ do not become heavier than those of $p$ and $q$.
\end{remark}

\section{Lipschitz regularization implies finite-sample estimation of $D_{\alpha}^L$ on $\R^d$}\label{sec:finite-sample}
In practice, we only have finite i.i.d. samples drawn from $P$ and $Q$. We denote by $X = \{x_1,\dots,x_m\}$ and $Y = \{y_1,\dots,y_n\}$ the i.i.d. samples from $P$ and $Q$, with empirical distributions $P_m=\frac{1}{m}\sum_{i=1}^m\delta_{x_i}$ and $Q_n=\frac{1}{n}\sum_{j=1}^n\delta_{y_j}$, respectively. Thus it is essential to provide guarantees for how fast $D_\alpha^L(P_m\|Q_n)$ converges to $D_\alpha^L(P\|Q)$ in average. This type of convergence rate for the Lipschitz-regularized $\alpha$-divergences has been proved in \cite{chen2023sample} on bounded domains of $\R^d$. Here, we derive the first result of the convergence of the finite-sample estimations on the unbounded domain $\R^d$, under certain tail conditions. The result for $d\geq3$ is stated below, with its proof deferred to \Cref{appendix:thm2}. The results for $d=1,2$ can be found as \Cref{prop:samplecomplexityd=2} and \Cref{prop:samplecomplexityd=1} in \Cref{appendix:thm2}. 
\begin{theorem}[Finite sample estimation of $D_{\alpha}^L$ on $\R^d$]
	\label{thm:samplecomplexity1}
	Assume $d\geq3$. For $\alpha>1$, let $P$ and $Q$ be probability measures on $\R^d$ such that $P\in\CP_{<\beta_1-d}(\R^d)$
	and $Q\in\CP_{<\beta_2-d}(\R^d)$, where $\beta_1>3d$ and $\beta_2>5d$. Suppose $\alpha$ satisfies $\frac{2d\alpha}{\alpha-1}<\beta_1-d$ and $\frac{2\alpha}{\alpha-1}<\frac{\beta_2}{d}-3$. Then we have
	\begin{equation}
		\E_{X,Y}\abs{D_{\alpha}^L(P_m\|Q_n)-D_{\alpha}^L(P\|Q)}\leq \frac{C_1}{m^{1/d}} + \frac{C_2}{n^{1/d}},
	\end{equation}
	where $C_1$ depends on $M_{\frac{d}{d-1}}(P)$ and $C_2$ depends on $M_{\frac{2d\alpha}{\alpha-1}}(P)$, $M_{\frac{2d\alpha}{\alpha-1}}(Q)$, and $M_{dr_2}(Q)$ for any $2+\frac{2\alpha}{\alpha-1}<r_2<\frac{\beta_2}{d}-1$. Here, we use $M_r(P)$ to denote the $r$-th moment of $P$. Both $C_1$ and $C_2$ are independent of $m,n$, but they depend on $L$ such that $C_1,C_2\to\infty$ when $L\to\infty$.
\end{theorem}
\begin{remark}
	The key to proving \Cref{thm:samplecomplexity1} is to leverage the Lipschitz condition of the test functions in the variational form \eqref{eq:variational_formula}.
\end{remark}

\section{Finite-sample estimations of $D_\alpha^L$ and $W_1$ with group symmetry on $\R^d$}\label{sec:group_symmetry}
Based on \Cref{thm:samplecomplexity1} and its proof, we are able to consider one special situation when the distributions are invariant with respect to some group symmetry and to provide convergence results for the empirical estimations of $D_\alpha^L$ with group symmetry in $\R^d$. Empirical estimations of divergences with group symmetry have been studied in \cite{chen2023sample, tahmasebisample} on bounded domains of $\R^d$ or on closed Riemannian manifolds. Here we provide the first sample complexity bound with group symmetry on unbounded domains, in particular, for $D_\alpha^L$ and later for $W_1$ in this section. Before presenting the theorems, we first briefly review the related concepts of group symmetry. Readers of interest can refer to \cite{birrell2022structure,chen2023sample,tahmasebisample} for more details. We leave all the proofs for this section in \Cref{appendix:group}.

A \textit{group} is a set $G$ equipped with a group product satisfying the axioms of associativity, identity, and invertibility. Given a group $G$ and a set $\CX\subset\R^d$, a map $\theta:G \times \CX\to\CX$ is called a \textit{group action on $\CX$} if $\theta_g\coloneqq \theta(g, \cdot): \CX\to\CX$ is an automorphism on $\CX$ for all $g\in G$, and $\theta_{g_2}\circ \theta_{g_1} = \theta_{g_2\cdot g_1}$, $\forall g_1, g_2\in G$. By convention, we will abbreviate $\theta(g, x)$ as $g x$.  We make the following assumptions on $G$.
\begin{assumption}\label{assumption:group}
	For any $g\in G$ and $x\in\R^d$, $\theta_g(x) = A_g \cdot x$, for some unitary matrix $A_g\in\R^{d\times d}$.
\end{assumption}

A function $\gamma:\CX\to \R$ is called \textit{$G$-invariant} if $\gamma\circ \theta_g = \gamma, \forall g\in G$. Let $\Gamma$ be a set of measurable functions $\gamma:\mathcal{X}\to\mathbb{R}$; its subset, $\Gamma_{G}$, of $G$-invariant functions is defined as
\begin{equation}
	\label{eq:invariant_function_space}
	\Gamma_{G} \coloneqq \{\gamma\in\Gamma:\gamma\circ \theta_g=\gamma,\forall g\in G\}.
\end{equation}
On the other hand, a probability measure $P\in \CP(\CX)$ is called \textit{$G$-invariant} if $P = (\theta_g)_\sharp P, \forall g\in  G$, where $(\theta_g)_\sharp P\coloneqq P\circ (\theta_g)^{-1}$ is the push-forward measure of $P$ under $\theta_g$. We denote the set of all $G$-invariant distributions on $\CX$ as $\CP_G(\CX)\coloneqq \{P\in\CP(\CX): P ~\text{is}~ G\text{-invariant}\}$. For $P,Q\in\CP_G(\CX)$, \cite{chen2023sample} proposes the following symmetry-informed estimator 
\begin{equation}\label{eq:variational_group}
	D_{\alpha}^{L,G}(P_m\|Q_n)\coloneqq \sup_{\gamma\in\text{Lip}_L^G(\R^d)}\{\E_{P_m}[\gamma]-\E_{Q_n}[f_\alpha^*(\gamma)]\}
\end{equation}
for $D_\alpha^L(P\|Q)$, where $\text{Lip}_L^G(\R^d)\subset\text{Lip}_L(\R^d)$ that consists of $G$-invariant $L$-Lipschitz functions. It is shown in Theorem 4.6 in \cite{birrell2022structure} that when $P_m,Q_n$ are replaced by $P,Q\in\CP_G(\CX)$ in \eqref{eq:variational_group}, we have $D_{\alpha}^{L,G}(P\|Q)= D_{\alpha}^{L}(P\|Q)$; that is, the divergence value between $P$ and $Q$ does not change if the supremum is taken over $\text{Lip}_L^G(\R^d)\subset\text{Lip}_L(\R^d)$ when both $P$ and $Q$ are $G$-invariant.

In particular, we consider the case when both $P$ and $Q$ are sub-Weibull, defined as follows.
\begin{definition}[sub-Weibull distributions]We call a distribution $P\in\CP(\R^d)$ sub-Weibull, if 
	\begin{equation}
		\Pr(x\sim P:\norm{x}
		\geq r)\leq a\exp(-br^{1/\theta})~\text{for all } r>0,~\text{for some } a,b,\theta>0.
	\end{equation}
\end{definition}
\begin{remark}
	Sub-Gaussian and sub-exponential distributions are special examples of sub-Weibull distributions.
\end{remark}
The following definition of intrinsic dimension is adopted from the capacity dimension from \cite{kegl2002intrinsic}.
\begin{definition}\label{def:intrinsic_dimension}
	The intrinsic dimension of a bounded $\mathcal{X}\subset\mathbb{R}^D$, denoted by $dim(\mathcal{X})$, is defined as
	\begin{equation}
		\dim(\mathcal{X})\coloneqq-\lim_{\epsilon\to0^+}\frac{\ln\mathcal{N}(\mathcal{X},\epsilon)}{\log\epsilon},
	\end{equation}
	where $\mathcal{N}(\mathcal{X},\epsilon)$ is the covering number of $\CX$ with $\epsilon$-balls in the standard Euclidean metric of $\R^d$.
\end{definition}
For example, if $\mathcal{X}\subset\mathbb{R}^D$ has nonempty interior, then $\dim(\mathcal{X})=D$; if $\mathcal{X}$ is a $d$-dimensional submanifold of $\mathbb{R}^D$, then $\dim(\mathcal{X})=d$. 

We have the following theorem for the empirical estimation of $D_\alpha^L$ with group symmetry on unbounded domains. 
\begin{theorem}[Finite sample estimation of $D_{\alpha}^L$ with finite group symmetry]
	\label{thm:samplecomplexity1_symmetry}
	For $\alpha>1$, let $P,Q\in\CP_G(\CX)$ for some $\CX\subset\R^d$, where $G$ satisfies \Cref{assumption:group}. Suppose the quotient space $\CX/G$ is connected, and for any bounded $\CX_0\subset\CX/G$ with nonempty interior with respect to the subspace topology ($\CX/G\xhookrightarrow{}\R^d$). Let $\abs{G}<\infty$ be the cardinality of $G$, and we further assume that both $P$ and $Q$ are sub-Weibull on $\R^d$. Then
	\begin{itemize}
		\item If $\dim(\CX_0)=d^*\geq 3$, we have
		\begin{equation}
			\E_{X,Y}\abs{D_{\alpha}^{L,G}(P_m\|Q_n)-D_{\alpha}^L(P\|Q)}\leq \frac{C_1}{(\abs{G}m)^{1/d^*}} + \frac{C_2}{(\abs{G}n)^{1/d^*}};
		\end{equation}
		\item If $\dim(\CX_0)=d^*= 2$, we have
		\begin{equation}
			\E_{X,Y}\abs{D_{\alpha}^{L,G}(P_m\|Q_n)-D_{\alpha}^L(P\|Q)}\leq \frac{C_1\ln m}{(\abs{G} m)^{1/2}} + \frac{C_2\ln n}{(\abs{G} n)^{1/2}};
		\end{equation}
		\item If $\dim(\CX_0)=d^*= 1$, we have
		\begin{equation}
			\E_{X,Y}\abs{D_{\alpha}^{L,G}(P_m\|Q_n)-D_{\alpha}^L(P\|Q)}\leq \frac{C_1}{(\abs{G} m)^{1/2}} + \frac{C_2}{(\abs{G} n)^{1/2}},
		\end{equation}
	\end{itemize}
	where $C_1$ and $C_2$ depend on $M_d(P)$, $M_d(Q)$. Both $C_1$ and $C_2$ are independent of $m,n$ and $G$.
\end{theorem}
When $G$ is a continuous group, we have the following theorem.
\begin{theorem}[Finite sample estimation of $D_{\alpha}^L$ with infinite group symmetry]
	\label{thm:samplecomplexity1_symmetry_infinite}
	For $\alpha>1$, let $P,Q\in\CP_G(\CX)$ for some $\CX\subset\R^d$, where $G$ satisfies \Cref{assumption:group}. Suppose the quotient space $\CX/G$ is connected, and for any bounded $\CX_0\subset\CX/G$ with nonempty interior with respect to the subspace topology ($\CX/G\xhookrightarrow{}\R^d$). Assume that both $P$ and $Q$ are sub-Weibull on $\R^d$. Then
	\begin{itemize}
		\item If $\dim(\CX_0)=d^{**}\geq 3$, we have
		\begin{equation}
			\E_{X,Y}\abs{D_{\alpha}^{L,G}(P_m\|Q_n)-D_{\alpha}^L(P\|Q)}\leq \frac{C_1}{m^{1/d^{**}}} + \frac{C_2}{n^{1/d^{**}}};
		\end{equation}
		\item If $\dim(\CX_0)=d^{**}= 2$, we have
		\begin{equation}
			\E_{X,Y}\abs{D_{\alpha}^{L,G}(P_m\|Q_n)-D_{\alpha}^L(P\|Q)}\leq \frac{C_1\ln m}{m^{1/2}} + \frac{C_2\ln n}{n^{1/2}};
		\end{equation}
		\item If $\dim(\CX_0)=d^{**}= 1$, we have
		\begin{equation}
			\E_{X,Y}\abs{D_{\alpha}^{L,G}(P_m\|Q_n)-D_{\alpha}^L(P\|Q)}\leq \frac{C_1}{m^{1/2}} + \frac{C_2}{n^{1/2}},
		\end{equation}
	\end{itemize}
	where $C_1$ and $C_2$ depend on $M_d(P)$, $M_d(Q)$. Both $C_1$ and $C_2$ are independent of $m,n$.
\end{theorem}
\begin{remark}
	If $\mathcal{X}$ is a $d^*$-dimensional connected submanifold of $\R^d$, and $G$ is a compact Lie group acting locally smoothly on $\mathcal{X}$, then $d^{**}=d-\dim(G)$, where $\dim(G)$ is the dimension of a principal orbit (i.e., the maximal
	dimension among all orbits) by Theorem IV 3.8 in \cite{bredon1972introduction}.
\end{remark}
The proofs of \Cref{thm:samplecomplexity1_symmetry} and \Cref{thm:samplecomplexity1_symmetry_infinite} also imply the convergence bound for the Wasserstein-1 distance with group symmetry on unbounded domains, since the variational form is shift-invariant with respect to the test function. We consider the symmetry-informed estimator for $P,Q\in\CP_G(\CX)$, proposed in \cite{chen2023sample,tahmasebisample}, defined as
\begin{equation}
	W_1^{G}(P_m,Q_n)\coloneqq \sup_{\gamma\in\text{Lip}_L^G(\R^d)}\{\E_{P_m}[\gamma]-\E_{Q_n}[\gamma]\}
\end{equation}
for $W_1(P,Q)$.
\begin{theorem}[Finite sample estimation of $W_1$ with finite group symmetry]
	\label{thm:samplecomplexity1_symmetry_W1}
	Let $P,Q\in\CP_G(\CX)$ for some $\CX\subset\R^d$, where $G$ satisfies \Cref{assumption:group}. Suppose the quotient space $\CX/G$ is connected, and for any bounded $\CX_0\subset\CX/G$ with nonempty interior with respect to the subspace topology ($\CX/G\xhookrightarrow{}\R^d$). Let $\abs{G}<\infty$ be the cardinality of $G$, and we further assume that both $P$ and $Q$ are sub-Weibull on $\R^d$. Then
	\begin{itemize}
		\item If $\dim(\CX_0)=d^*\geq 3$, we have
		\begin{equation}
			\E_{X,Y}\abs{W_1^G(P_m,Q_n)-W_1(P,Q)}\leq \frac{C_1}{(\abs{G}m)^{1/d^*}} + \frac{C_2}{(\abs{G}n)^{1/d^*}};
		\end{equation}
		\item If $\dim(\CX_0)=d^*= 2$, we have
		\begin{equation}
			\E_{X,Y}\abs{W_1^G(P_m,Q_n)-W_1(P,Q)}\leq \frac{C_1\ln m}{(\abs{G}m)^{1/2}} + \frac{C_2\ln n}{(\abs{G}n)^{1/2}};
		\end{equation}
		\item If $\dim(\CX_0)=d^*=1$, we have
		\begin{equation}
			\E_{X,Y}\abs{W_1^G(P_m,Q_n)-W_1(P,Q)}\leq \frac{C_1}{(\abs{G}m)^{1/2}} + \frac{C_2}{(\abs{G}n)^{1/2}},
		\end{equation}
	\end{itemize}
	where $C_1$ and $C_2$ depends on $M_d(P)$, $M_d(Q)$. Both $C_1$ and $C_2$ are independent of $m,n$ and $G$.
\end{theorem}
When $G$ is a continuous group, we have the following theorem.
\begin{theorem}[Finite sample estimation of $W_1$ with infinite group symmetry]
	\label{thm:samplecomplexity1_symmetry_infinite_W1}
	Let $P,Q\in\CP_G(\CX)$ for some $\CX\subset\R^d$, where $G$ satisfies \Cref{assumption:group}. Suppose the quotient space $\CX/G$ is connected, and for any bounded $\CX_0\subset\CX/G$ with nonempty interior with respect to the subspace topology ($\CX/G\xhookrightarrow{}\R^d$). Assume that both $P$ and $Q$ are sub-Weibull on $\R^d$. Then we have
	\begin{itemize}
		\item If $\dim(\CX_0)=d^{**}\geq 3$, we have
		\begin{equation}
			\E_{X,Y}\abs{W_1^G(P_m,Q_n)-W_1(P,Q)}\leq \frac{C_1}{m^{1/d^{**}}} + \frac{C_2}{n^{1/d^{**}}};
		\end{equation}
		\item If $\dim(\CX_0)=d^{**}=2$, we have
		\begin{equation}
			\E_{X,Y}\abs{W_1^G(P_m,Q_n)-W_1(P,Q)}\leq \frac{C_1\ln m}{m^{1/2}} + \frac{C_2\ln n}{n^{1/2}};
		\end{equation}
		\item If $\dim(\CX_0)=d^{**}=1$, we have
		\begin{equation}
			\E_{X,Y}\abs{W_1^G(P_m,Q_n)-W_1(P,Q)}\leq \frac{C_1}{m^{1/2}} + \frac{C_2}{n^{1/2}},
		\end{equation}
	\end{itemize}
	where $C_1$ and $C_2$ depend on $M_d(P)$, $M_d(Q)$. Both $C_1$ and $C_2$ are independent of $m,n$.
\end{theorem}
\begin{remark}
	Although the multiplicative constants in \Cref{thm:samplecomplexity1_symmetry_W1} and \Cref{thm:samplecomplexity1_symmetry_infinite_W1} are not optimal, but the rate is optimal compared to Theorem 1 in \cite{fournier2015rate} for $W_1$, when $d^*$ or $d^{**}$ are greater than or equal to three, or equal to one.
\end{remark}

\section{Numerical experiments}\label{sec:numerical}
In this section, we demonstrate how using the Lipschitz-regularized $\alpha$-divergences as objective functionals enables stable learning of heavy-tailed distributions and distributions with low-dimensional manifolds or fractal structures with various generative models. Note that the Lipschitz-regularized $\alpha$-divergences have an equivalent primal formulation in \eqref{def:W1proximal}, which can be viewed as $\alpha$-divergences with $W_1$-proximal regularization. One may consider replacing the $W_1$-proximal regularization with a $W_2$-proximal regularization, where $W_2$ is the Wasserstein-2 distance, as the $W_2$ distance and proximal regularization is widely used in generative modeling; for example, see \cite{onken2021otflow, wang2023efficient}. The $\alpha$-divergences with $W_2$-proximal regularization are defined as
\begin{equation}\label{def:W2proximal}
	D_{\alpha,2}^\lambda(P\|Q):= \inf_{\eta\in\mathcal{P}(\mathbb{R}^d)}\{D_\alpha(\eta\|Q) + \lambda\cdot W_2^2(P,\eta)\}.
\end{equation}
In \Cref{subsec:explanation:generative:models}, we introduce the generative models used and explain how their learning objectives relate to $\alpha$-divergences with $W_1$ or $W_2$ proximals. We illustrate our points with four examples. In \Cref{subsec:examples}, we compare the effects of incorporating $W_1$ or $W_2$ proximals in the learning objectives by training on a 2D Student-t distribution and on a real-world keystroke dataset. In \Cref{sec:attractor}, we show the importance of Lipschitz-regularized $\alpha$-divergences when learning distributions with low-dimensional structures with an example of learning a strange attractor from the Lorenz 63 model. In \Cref{sec:highdimension}, we present the task of learning an anisotropic heavy-tailed distribution embedded in a high-dimensional space and the results highlight that the Lipschitz-regularized $\alpha$-divergences make generative learning agnostic to heavy-tailed and manifold assumptions. 
We use Gaussian priors for all our experiments, and the implementation details including the network architectures can be found in the Supplementary Material.

\subsection{Generative models with different learning objectives}
\label{subsec:explanation:generative:models}

$W_1$ and $W_2$ proximals can be found, sometimes implicitly, in the learning objectives of several existing generative models. Below, we list various models based on $\alpha$-divergences used in our experiments and explain why some of them are (either implicitly or explicitly) regularized by Wasserstein proximal.

\noindent\textbf{(1) Generative models without proximal regularization:}
\begin{itemize}
	\item \textbf{$\alpha$-GAN}: GANs \cite{goodfellow2014generative,nowozin2016f} based on the variational representation of the $\alpha$-divergence \eqref{eq:variation_alpha};
	%adversarial training model where discriminator first maximizes a generic objective function and then generator minimizes the objective function; in particular, GANs minimizing $\alpha$-divergences falls into the category of $f$-GANs \cite{nowozin2016f};
	\item \textbf{$\alpha$-GPA}: Generative particle algorithm (GPA) based on the $\alpha$-divergence \cite{gu2022lipschitz};
	%particle method where the discriminators are trained as in GANs and then particles are transported by the gradients of the learned discriminators;
	\item \textbf{CNF}: Continuous normalizing flows by \cite{chen2018neural}, where the loss function is based on the KL divergence, a special case of the $\alpha$-divergence when $\alpha=1$.
	%\item \textbf{$\alpha$-flow GAN}: flow-based generative models based on adversarial training of the $\alpha$-divergence \cite{grover2018flow}. 
	%\WZ{Remove}
\end{itemize}

\noindent\textbf{(2) Generative models with $W_1$-proximal regularization:}
\begin{itemize}
	\item \textbf{\text{Lip}-$\alpha$-GAN} \cite{birrell2020f}: GANs using the Lipschitz-regularized $\alpha$-divergence \eqref{eq:variational_formula} as the objective function, with the Lipschitz constant set to $L=1$ in our experiment;
	\item \textbf{\text{Lip}-$\alpha$-GPA} \cite{gu2022lipschitz}: GPAs using the Lipschitz-regularized $\alpha$-divergence \eqref{eq:variational_formula} as the objective function, with the Lipschitz constant set to $L=1$ in our experiment. This is the implementation of the gradient flow formulation \eqref{eq:gradientflow2}.
\end{itemize}

\noindent\textbf{(3) Generative models with $W_2$-proximal regularization:} We consider the following class of flow-based models, which minimize $\alpha$-divergences with $W_2$ proximal \eqref{def:W2proximal} written as \eqref{eq:w2:dynamical:formulation} via the Benamou-Brenier formula,
\begin{equation}
	\label{eq:w2:dynamical:formulation}
	\inf_{v, \rho} \mathcal{F}(\rho(\cdot, T)) + C\int_0^T \frac{1}{2}|v(x,t)|^2 \rho(x,t)\diff{x} \diff{t}.
\end{equation}
Here, $\rho:\R^d\times [0, T]\to \R$ is the evolution of the probability measure via the (trainable) velocity field $v:\R^d\times [0, T]\to \R^d$, satisfying the Fokker–Planck equation:
\begin{equation}
	\rho_t + \nabla\cdot (\rho v) =  \frac{\sigma^2}{2}\Delta \rho, \quad \rho(\cdot, 0) = \rho_0~\text{is a tractable prior distribution, e.g., Gaussian.}
\end{equation}
\begin{itemize}
	\item \textbf{OT flow} \cite{onken2021otflow}: Optimal transport (OT) normalizing flows, which are equivalent to the $W_2$-proximal of CNFs, with $\mathcal{F}(\rho(\cdot, T)) = D_{\text{KL}}(Q \| \rho(\cdot, T)) $ and $\sigma=0$ in \eqref{eq:w2:dynamical:formulation};
	\item \textbf{VE-SGM} \cite{song2021scorebased}: Score-based generative models (SGM) with variance-exploding (VE) forward SDE \cite{song2021scorebased}. According to the mean-field game formulation by \cite{zhang2023meanfield}, it is equivalent to \eqref{eq:w2:dynamical:formulation} with stochastic dynamics ($\sigma>0$) and a cross-entropy terminal cost $\mathcal{F}(\rho(\cdot, T)) = -\mathbb{E}_{\rho(\cdot, T))} [\log Q]$, essentially also a $W_2$-proximal of CNFs.
\end{itemize}
We refer to \Cref{fig:experiment:design} for a visual illustration of the relationships among the models being compared.

\begin{figure}[t]
	\centering
	\includegraphics[width=.6\textwidth]{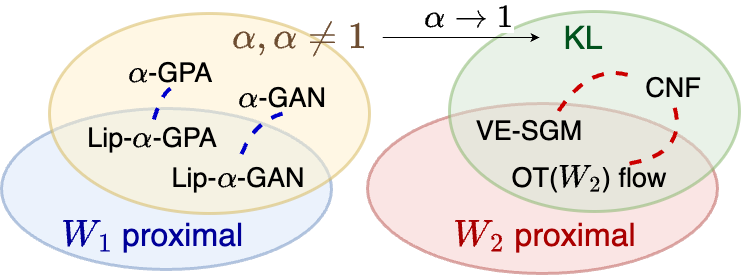}
	\caption{Generative models in the experiment and their relationship with the $\alpha$-divergences with $W_1$ or $W_2$ proximal regularization. See \Cref{subsec:explanation:generative:models} for detailed explanations of the models and notations.}
	\label{fig:experiment:design}
\end{figure}

\subsection{Learning heavy-tailed distributions}
\label{subsec:examples}

\begin{figure}[h]
	\centering
	
	\begin{subfigure}{\textwidth}
		\centering
		\includegraphics[width=.23\linewidth]{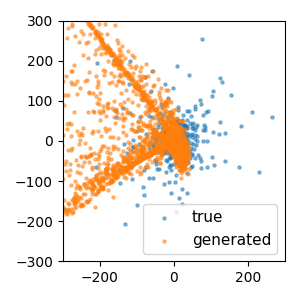}
		\includegraphics[width=.23\linewidth]{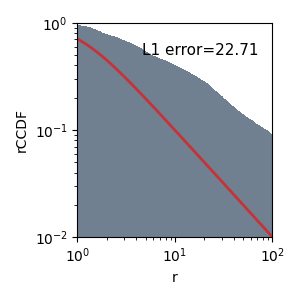}
		\includegraphics[width=.23\linewidth]{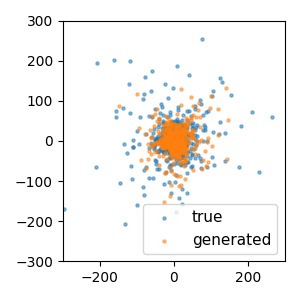}
		\includegraphics[width=.23\linewidth]{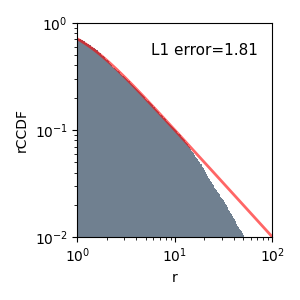}
		\caption{ $\alpha$-GAN (left) and its counterpart with $W_1$-proximal regularization, \text{Lip}-$\alpha$-GAN (right) }
		
	\end{subfigure}
	
	\begin{subfigure}{\textwidth}
		\centering
		\includegraphics[width=.23\linewidth]{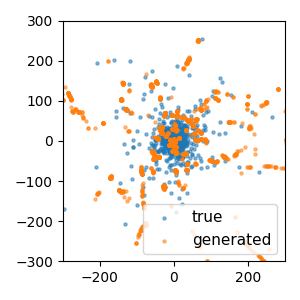}
		\includegraphics[width=.23\linewidth]{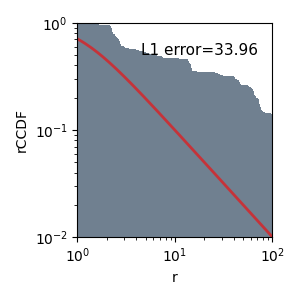}
		\includegraphics[width=.23\linewidth]{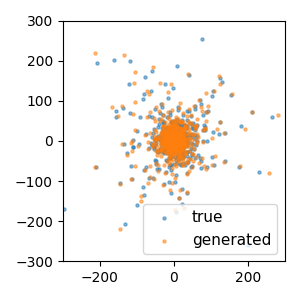}
		\includegraphics[width=.23\linewidth]{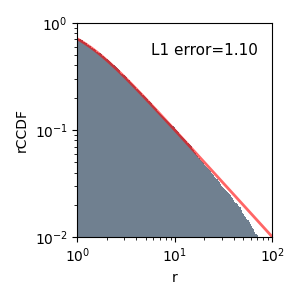}
		\caption{ $\alpha$-GPA (left) and its counterpart with $W_1$-proximal regularization, \text{Lip}-$\alpha$-GPA (right) }
		
	\end{subfigure}

	\caption{Learning a 2D isotropic Student-t with degree of freedom $\nu=1$ (tail index $\beta=3.0$) using generative models based on $\alpha$-divergences with $\alpha=2$ with or without Lipschitz regularization.
		Models with Lipschitz regularization (right) learn the heavy-tailed distribution significantly better than those without (left). See \Cref{subsec:explanation:generative:models} for detailed explanations of the models.
	}
	\label{fig:student-t:1.0:alpha:proximal:divergence}
\end{figure}

\paragraph{2D Student-t  example} We compare various generative models for learning a heavy-tailed 2D isotropic Student-t distribution with $\nu$ degrees of freedom, $q(x)\propto (1+ \frac{|x|^2}{\nu})^{\frac{\nu + 2}{2}}$. This synthetic example allows us to adjust the tail decay rate $\beta = \nu + 2$ by selecting different degrees of freedom $\nu$. In the main text, we present a heavy-tailed example with $\beta=3$ that does not have a finite first moment, while the relatively easier case of $\beta=5$ is deferred to \Cref{fig:student-t:3.0:alpha:proximal:divergence} and \Cref{fig:student-t:3.0:proximal:losses} in the Supplementary Material. We use 10,000 samples to train the models.

\Cref{fig:student-t:1.0:alpha:proximal:divergence} and \Cref{fig:student-t:1.0:proximal:losses} present the performance of various generative models. Each model is evaluated in two plots. First, a 2D scatter plot displays the generated samples (orange) and the true samples (blue), providing a visual assessment of the sample quality. Next, the tail behavior is assessed by plotting the ground truth Radial Complementary Cumulative Distribution Function (rCCDF) (red curve) and the histogram of the radii of generated samples (gray). The rCCDF is defined as $\text{rCCDF}(r) = 1- \text{CDF}(r)$, where CDF($r$) is the cumulative distribution function of the radius.  We then calculate the $L_1$ error between the ground truth rCCDF and the generated sample histogram. Generative models with Lipschitz regularization ($W_1$-proximal) significantly outperform the others in learning heavy-tailed distributions, corroborating our theoretical results in \Cref{sec:theory}.

\begin{figure}[h]
	\centering
	\begin{subfigure}{.48\linewidth}
		\includegraphics[width=.48\linewidth]{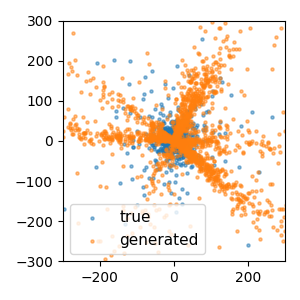}
		\includegraphics[width=.48\linewidth]{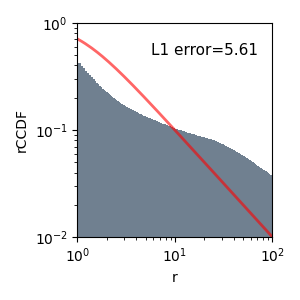}
		\caption{CNF}
	\end{subfigure}\\
	\begin{subfigure}{.48\linewidth}
		\includegraphics[width=.48\linewidth]{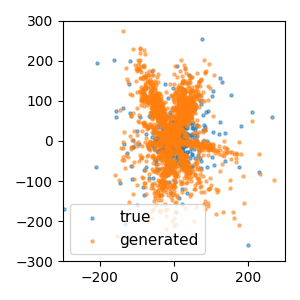}
		\includegraphics[width=.48\linewidth]{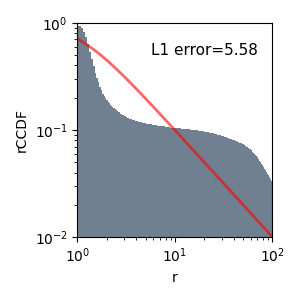}
		\caption{OT flow}
	\end{subfigure}
	\begin{subfigure}{.48\linewidth}
		\includegraphics[width=.48\linewidth]{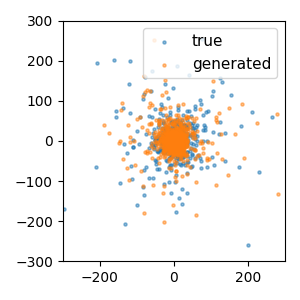}
		\includegraphics[width=.47\linewidth]{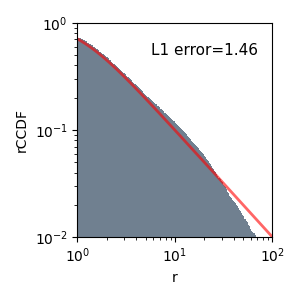}
		\caption{VE-SGM}
	\end{subfigure}
	\caption{Learning a 2D isotropic Student-t with degree of freedom $\nu=1$ (tail index $\beta=3.0$) using generative models based on $\alpha$-divergences with or without $W_2$-proximal regularization and $\alpha=2$.
		See \Cref{subsec:explanation:generative:models} for detailed explanations of the models.
	}
	\label{fig:student-t:1.0:proximal:losses}
\end{figure}

\begin{figure}[h]
	\centering
	\begin{subfigure}{.49\linewidth}
		\includegraphics[width=.49\linewidth]{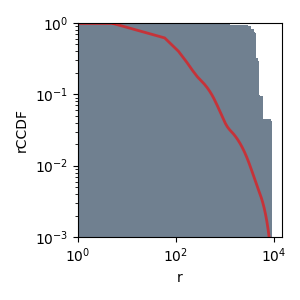}
		\includegraphics[width=.49\linewidth]{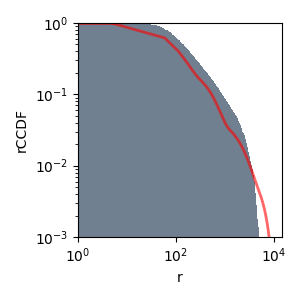}
		\caption{$\alpha$-GPA  (left),  $\alpha$-GAN (right)}
	\end{subfigure}
	\begin{subfigure}{.49\linewidth}
		\centering
		\includegraphics[width=.49\linewidth]{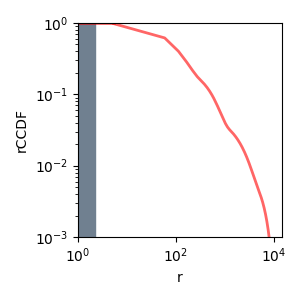}
		\caption{CNF}
	\end{subfigure}
	
	\caption{Sample generation of inter-arrival time between keystrokes. Generative models based on the $\alpha$-divergences with $\alpha=2$ (a), and the KL divergence (b). }
	\label{fig:real:example:without:w:proximals}
\end{figure}

\paragraph{Keystroke example} 
For a real-world heavy-tailed example, we consider learning the inter-arrival time between keystrokes from multiple users typing sentences \cite{interarrivaltime_heavytailedexample}.
The target dataset consists of 7,160 scalar samples, and we generated 10,000 samples using generative models with $W_1$ or $W_2$ proximal regularization.

We display the tail behavior by plotting the ground truth CCDF (red curve) and the corresponding histogram of the generated samples (gray). Unlike the previous synthetic example, the ground truth CCDF here is obtained by interpolating the heights of the histogram bins of the true samples.
In \Cref{fig:keystrokes:proximal:losses}, generative models with $W_1$-proximal regularization (\text{Lip}-$\alpha$-GPA and \text{Lip}-$\alpha$-GAN) outperform those regularized with $W_2$-proximals (OT flow and VE-SGM) in capturing the tails. 
This observation suggests that $W_1$-proximal algorithms can potentially handle heavier tails more effectively than $W_2$-proximal methods. In other words, algorithms based on the Lipschitz-regularized $\alpha$-divergences are more agnostic to heavy-tailed assumptions.

\begin{figure}[h]
	\centering
	\begin{subfigure}{.49\textwidth}
		\includegraphics[width=.49\linewidth]{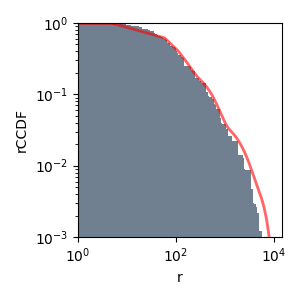}
		\includegraphics[width=.49\linewidth]{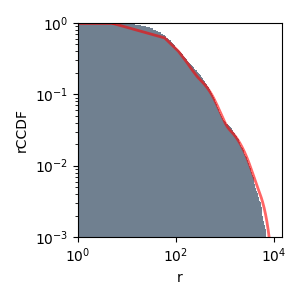}
		\caption{\text{Lip}-$\alpha$ GPA (left),  \text{Lip}-$\alpha$ GAN (right)}
	\end{subfigure}
	\begin{subfigure}{.49\textwidth}
		\includegraphics[width=.49\linewidth]{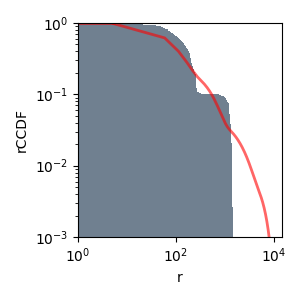}
		\includegraphics[width=.49\linewidth]{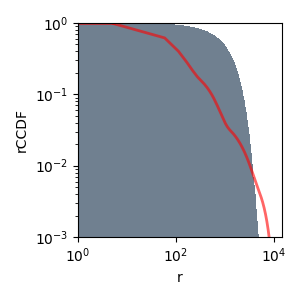}
		\caption{OT flow (left), VE SGM (right) }
	\end{subfigure}

	\caption{Sample generation of inter-arrival time between keystrokes. Generative models with $W_1$-proximal regularization, panel~(a), outperform those with $W_2$-proximal regularization, panel~(b), in capturing the tails. This observation suggests that $W_1$-proximal algorithms can potentially handle heavier tails more effectively than $W_2$-proximal methods.}
	\label{fig:keystrokes:proximal:losses}
\end{figure}

\subsection{Learning  attractors of chaotic dynamical systems}\label{sec:attractor}
\paragraph{Strange attractor from Lorenz 63 example}
The Lorenz 63 model is renowned for its strange attractor, which exhibits a complex fractal structure characterized by a non-integer Hausdorff dimension. In this example, we use various generative models to learn the geometric shape of the attractor, without accounting for its underlying dynamics.
The target dataset $\mathcal{T}$ for the generative models consists of $N=5000$ positions, defined as: 
$\mathcal{T} = \{\mathbf{x}(t_i) = (x_1(t_i), x_2(t_i), x_3(t_i)) : t_i \sim \text{Unif}([9900, 10000]) \}_{i=1}^{N}$ where $(x_1(t_i), x_2(t_i), x_3(t_i))$ is a numerically computed solution trajectory of the Lorenz 63 model with the standard parameter values $a=10, b=28, c=8.3$. The generated samples are represented as $\mathcal{G} = \{\mathbf{y}_i = (y_{1i}, y_{2i}, y_{3i})\}_{i=1}^M$, where $M$ is the number of generated points which does not necessarily match $N$. We use $M = 10000$ generated samples across various generative models for this example.

Because the generated samples lack time labels, the dynamics cannot be directly observed. 
Instead, we consider two standards: (a) measurement of how close the generated particles land on the attractor and (b) characteristic of the fractal structure. These standards are measured by corresponding metrics:
\begin{itemize}
	\item[(a)] \textbf{Mean square sum of the errors (MSE)} between generated samples $\mathbf{y}_i$ and their closest validation sample $\mathbf{v}_i^* = \rm{argmin}_{\mathbf{v}_j \in \mathcal{V}} |\mathbf{y}_i - \mathbf{v}_j|$ where the validation dataset is given as $\mathcal{V}=\{\mathbf{v}_j = (v_1(t_j), v_2(t_j), v_3(t_j)) : t_j = 9900 + 0.01 \cdot j\}_{j=1}^{10000}$ 
	\begin{equation}
		\label{eq:mse}
		\text{MSE} = \frac{1}{M}\sum_{i=1}^{M} |\mathbf{y}_i - \mathbf{v}_i^*|^2,
	\end{equation}
	which measures the deviation of generated samples from the attractor trajectory.
	
	\item[(b)] Adapted \textbf{Correlation dimension} for measuring dimensionality of the space occupied by point clouds of generated samples $\{\mathbf{y}_i\}_{i=1}^{M}$ without time information. Original correlation dimension is a characteristic measure to distinguish between deterministic chaos and random noise, to detect potential faults \cite{correlation_dimension}. Real correlation dimension for the attractor of Lorenz 63 should be 2.05. We obtained a reference value 2.04 by applying to our validation dataset $\mathcal{V}$ from a selection of the algorithm's parameter radius $r \in [0.7, 1.1]$.
\end{itemize}
The results can be found in \Cref{tab:lorenz63}. The results illustrate that 1) Lipschitz-regularized methods in general capture the attractor and its structure while those without Lipschitz regularization fail; 2) other methods such as OT flows, CNFs, and SGMs fail to accurately capture the attractor even they are trained for a longer time with more complicated network architecture. We additionally visualize generated samples in \Cref{fig:lorenz63}. Similar results when $N=1000$ and $M=2000$ can be found in the Supplementary Material.

\begin{figure}[htp]
	\centering
	\begin{subfigure}{\linewidth}
		\includegraphics[width=0.24\linewidth]{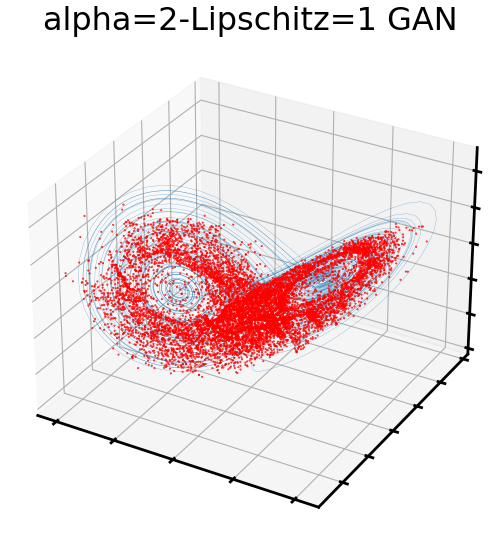}
		\includegraphics[width=0.24\linewidth]{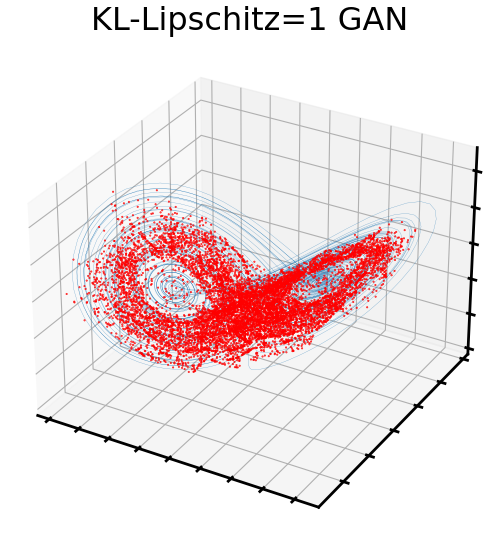}
		\includegraphics[width=0.24\linewidth]{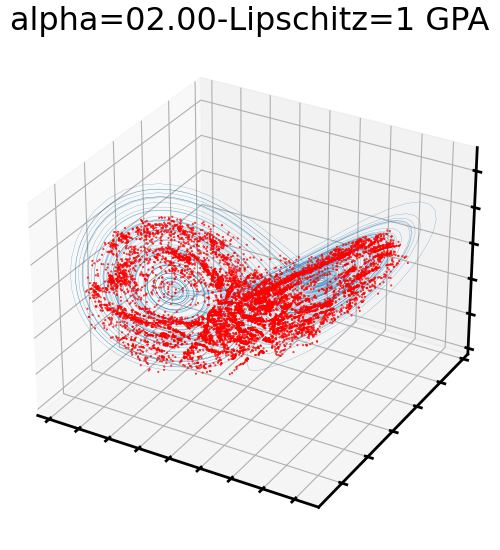}
		\includegraphics[width=0.24\linewidth]{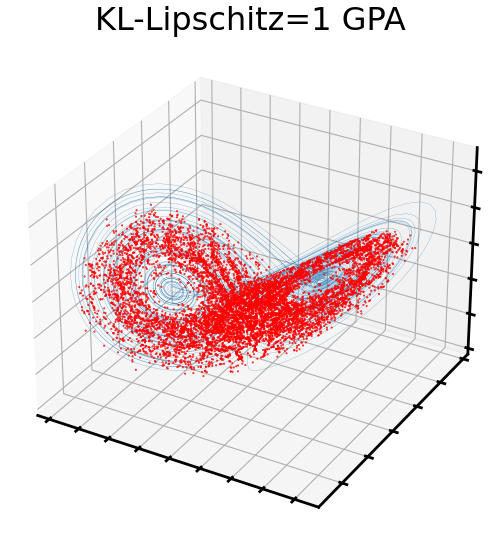}
		\caption{Generated samples from generative models with Lipschitz-regularized $\alpha$-divergences as learning objectives.
			$\alpha=2$-Lipschitz-1 GAN  (first), KL-Lipschitz-1 GAN (second), $\alpha=2$-Lipschitz-1 GPA  (third), KL-Lipschitz-1 GPA (fourth) }
	\end{subfigure}
	\begin{subfigure}{\linewidth}
		
		\includegraphics[width=0.24\linewidth]{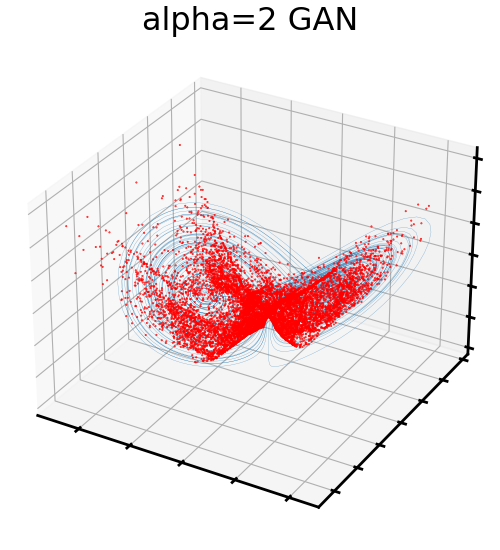}
		\includegraphics[width=0.24\linewidth]{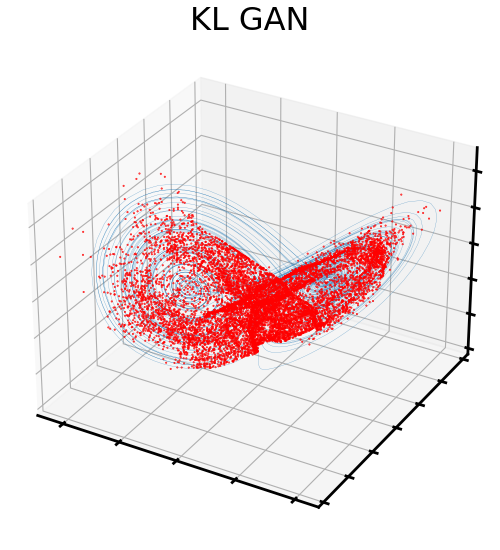}
		\includegraphics[width=0.24\linewidth]{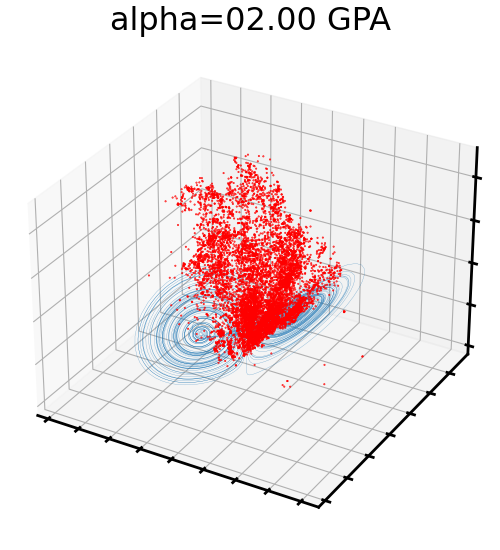}
		\includegraphics[width=0.24\linewidth]{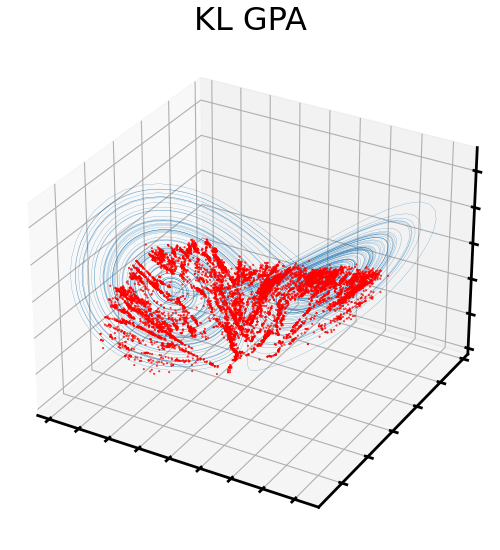}
		\caption{Generated samples from generative models with un-regularized $\alpha$-divergences as learning objectives. $\alpha=2$ GAN (first), KL GAN (second), $\alpha=2$ GPA (third), KL GPA (fourth). Snapshots from $\alpha=2$ and KL GPAs are transient and eventually blew up.}
	\end{subfigure}
	\begin{subfigure}{\linewidth}
		\centering
		\includegraphics[width=0.24\linewidth]{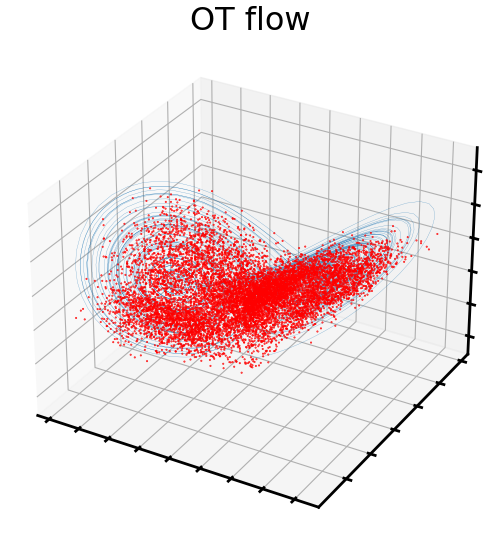}
		\includegraphics[width=0.24\linewidth]{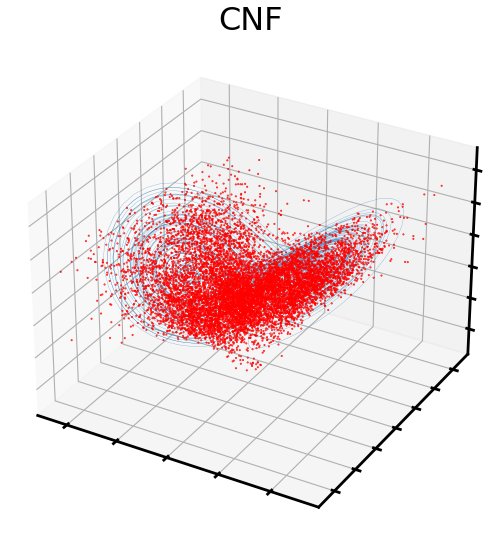}
		\includegraphics[width=0.24\linewidth]{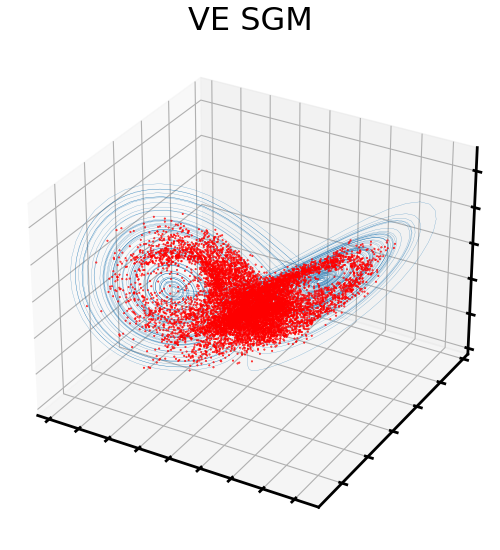}
		
		\caption{Generated samples from generative models with different learning objectives. 
			OT-flow: $W_2$-reverse KL divergence (left), CNF: reverse KL divergence (center), VE-SGM: $W_2$-proximal regularized cross-entropy (right)}
	\end{subfigure}
	
	\caption{Generated samples ($M=10000$) of the Lorenz 63 strange attractor from $N = 5000$ target samples. Lipschitz-regularized methods in general capture the attractor and its dimension while those without Lipschitz regularization fail. Other methods such as OT flows, CNFs, and SGMs cannot accurately capture the fractal structure. See \Cref{tab:lorenz63} for error metrics.}
	\label{fig:lorenz63}
\end{figure}

\begin{table}[h]
	\centering
	\begin{tabular}{c|r|r|r}
		Model & MSE & Correlation dimension &  Computation time (sec) \\ \hline\hline %MMD ($\gamma=0.01$)  \\\hline\hline
		\text{Lip}-$\alpha=2$ GAN & $\mathbf{0.1240}$  & $\mathbf{2.00}$  & 491.851 \\ \hline%0.0448 \\ \hline
		\text{Lip}-KL GAN & $\mathbf{0.1226}$  & $\mathbf{2.01}$  &  505.330 \\ \hline %0.0415 \\ \hline% 0.0406 \\ \hline
		$\alpha=2$ GAN & 0.945  & 1.99  & 336.272 \\ \hline%0.0743 \\ \hline
		KL GAN & 0.1612 & 1.99  & 486.941 \\ \hline%0.0418 \\ \hline
		\text{Lip}-$\alpha=2$ GPA & 0.2984 & 1.60 & 410.385 \\
		\hline %0.0410 \\ \hline
		\text{Lip}-KL GPA & 0.1369 & 1.91 & 398.344 \\ \hline
		$\alpha=2$ GPA & - & - & - \\ \hline%0.1071  \\ \hline
		KL GPA & - & - &  - \\ \hline%\textbf{0.0182}  \\ \hline
		
		OT($W_2$) flow & 0.6231 & 2.29 & $\geq 60000$ \\
		\hline%\textbf{0.0228}\\ \hline
		CNF & 1.2674 & 2.31 & $\geq 60000$ \\ \hline%\textbf{0.0335}\\ \hline
		VE SGM & $\mathbf{0.0791}$ & 2.31 & 2382.733\\ %0.1088 \\
		\hline
	\end{tabular}
	\caption{Performance metrics: (i) MSE \eqref{eq:mse} between generated samples and the validation dataset $\mathcal{V}$ that measures how close the generated particles land on the attractor, and (ii) Correlation dimension for $M=10000$ generated samples from different generative models. The ground truth correlation dimension measured on the validation dataset $\mathcal{V}$ is 2.04. A higher correlation dimension implies that noise dominates the shape of the attractor. A lower correlation dimension implies that the point clouds are more sparsely populated on the attractor; see for instance, Lip-$\alpha=2$ GPA compared to Lip-$\alpha=2$ GAN in \Cref{fig:lorenz63}. We do not report the MSE and Correlation dimension for $\alpha=2$ GPA and KL GPA (no Lipschitz regularization) since generated particles diverged in the early stage of training.
		Although SGM has the smallest MSE, it takes significant longer time to train, requiring much deeper network architecture (otherwise it does not converge), and it still significantly over-estimates the fractal dimension. See also \Cref{fig:lorenz63} for visualizations.\label{tab:lorenz63}}
\end{table}
\subsection{Learning distributions supported on low dimensional manifolds}\label{sec:highdimension}
\paragraph{10D heavy-tailed manifold embedded in 110D} 
We provide a high-dimensional example adapted from \cite{huster2021pareto}. In this example, a 10D heavy-tailed distribution is embedded in $\mathbb{R}^{110}$. Each of the first 10 axes is drawn from the standard Cauchy distribution $w_i \sim Cauchy$, then powered by a random exponent $t_i \sim \text{Unif}([0.5, 2])$, i.e.,
$x_i = sign(w_i) |w_i|^{t_i} \text{ for } i=1, \dots, 10$.
Values of the remaining axes are set to zero: $x_i=0$ for $i=11, \cdots, 110$. In our experiment, we fix the exponents $t_i, i=1, \cdots 10$, to  $(1.31, 0.91, 1.13, 1.76,  0.50, 0.68, 1.50, 1.73, 0.70, 1.36)$. We present two metrics similar to those used in the multivariate distributions example in \cite{huster2021pareto} to demonstrate (a) whether the algorithm can capture the heavy tails in the first 10 dimensions and (b) whether the generated distribution correctly lies on the 10-dimensional plane. For (a), we calculate the averaged $L_1$ error over the first 10 dimensions between the empirical rCCDF $F_v$ built from a validation dataset consisting of 100K target samples and the empirical rCCDF $F_g$ built from generated samples:
\begin{equation}\label{eq:L1error}
	L_1(F_v, F_g) = \sum_{i=1}^{20000}|F_v(z_i)-F_g(z_i)|  (z_{i+1}-z_i),
\end{equation}
where $z_i$ are sampled in equi-distance from the interval $[1, 5\times 10^6]$. For (b), we calculate the Euclidean distance of the generated samples to their projections on the first 10-dimensional subspace which is written as $\sum_{i=11}^{110}\mathbb{E}_{y_i} [\|y_i\|]$ where the orthogonal subspace is represented as zero $[0, \cdots, 0] \in \mathbb{R}^{100}$.

\begin{table}[h]
	\centering
	\begin{tabular}{c|c|c}
		Model & heavy-tailed subspace & orthogonal subspace  \\
		& avg $L_1$ error  & avg Euclidean distance \\
		\hline\hline
		\text{Lip}-$\alpha$ GPA & $\mathbf{3.1155 e + 02}$ & $3.4179e+00$ \\
		\hline
		$\alpha$ GPA & $4.9993e+06$ & $1.7150e+15$  \\
		\hline
		\text{Lip}-$\alpha$ GAN & $\mathbf{3.4645e+02}$ %($3.2283e+02$) 
		& $\mathbf{1.0990e-01}$ %($1.1713e-01$) 
		\\
		\hline
		$\alpha$ GAN & $4.4994e+06$ %($4.4994e+06$) 
		& $\mathbf{2.4480e-03}$ %($1.2395e-03$) 
		\\
		\hline
		OT($W_2$) flow & $4.9993e+06$ & inf \\
		\hline
		CNF & $4.9993e+06$ & inf \\
		\hline
		VE SGM & $3.6031e+02$ & $1.4441e+03$ \\
		\hline
	\end{tabular}
	\caption{Learning 10D heavy-tailed data embedded in $\mathbb{R}^{110}$ using 10K target samples. We report the  $L_1$ error defined in \eqref{eq:L1error} averaging over the first 10 dimensions. 
		Generative models without Lipschitz-regularized learning objectives, such as unregularized models or those using $W_2$-proximal regularization, either fail to capture the heavy tails or fail  to capture the manifold. In contrast, Lipschitz-regularized $\alpha$-divergence enables generative models to learn heavy-tailed distributions even when the tails exhibit different power-law behaviors, i.e.,  $Q(x_i) \sim |x_i|^{-\beta_i}$ for $i=1, \cdots, 10$.
		In addition, the Lipschitz-regularized $\alpha$-divergence encourages generated samples to lie near the data manifold. The unconstrained discriminator in $\alpha$-GAN produces large values outside the manifold, forcing the generator to map the source onto the 10D plane. However, the unconstrained $\alpha$-GAN fails to learn the tails. For further comparison of training objective function values for GANs and GPAs, see \Cref{tab:divergence:values} in \Cref{append:divergence_value}.\label{tab:high:dim}
	}
\end{table}
The results in \Cref{tab:high:dim} verify that models with the Lipschitz-regularized $\alpha$-divergences as objectives are more agnostic to both heavy-tailed and manifold assumptions.

\section{Conclusions and discussions}\label{sec:conclusion}

In this paper, we prove that Lipschitz-regularized $\alpha$-divergences, introduced in previous works, enable robust and stable learning for target distributions with minimal assumptions. In particular, we prove that these divergences are always finite and have a well-defined variational derivative when the first input distribution has a finite first moment. We also prove the sufficient and necessary conditions for the divergence to be finite when both distributions have power-law-decay tails. A first convergence rate of the finite-sample estimations of these divergences on $\R^d$ is proved. As a result, we derive the first sample complexity bounds for the empirical estimations of $D_\alpha^L$ and $W_1$ with group symmetry on $\R^d$. Numerical simulations further confirm the robustness of these divergences, showing that they significantly improve the learning process across a range of challenging scenarios, such as heavy-tailed distributions or distributions supported on low-dimensional manifolds or fractals.

Some future directions are unexplored in this work. First, it is not clear if there is an optimal $\alpha$ or if the $\alpha$ should be chosen adaptively to make the learning more efficient. Second, the PDE theory of the Lipschitz-regularized gradient flow is not established, and the convergence of the gradient flow is an important topic and may require some new functional inequalities. Lastly, \Cref{thm:samplecomplexity1} is not sharp, and a sharp convergence bound will help better understand this class of divergences and further derive better generalization bounds for algorithms based on this class of divergences.

\section*{Acknowledgement}
Z. Chen, H. Gu, M. Katsoulakis, L. Rey-Bellet are partially funded by AFOSR grant FA9550-21-1-0354. M.K. and L. R.-B. are partially funded by NSF DMS-2307115. H. G. and M.K.  are partially funded by NSF TRIPODS CISE-1934846. Z. Chen and W. Zhu are partially supported by NSF under DMS-2052525, DMS-2140982, and DMS-2244976. The authors would like to thank the anonymous reviewers for their careful reading and constructive feedback, which helped improve the manuscript.

\section*{Data Availability Statement}
All codes in \Cref{sec:numerical} can be found at: \url{https://github.com/HyeminGu/Proximal_generative_models}. The implementation details can be found in the Supplementary Material. The keystroke data is available in [Observations on Typing from 136 Million Keystrokes], at \url{ http://userinterfaces.aalto.fi/136Mkeystrokes}. All the other datasets can be generated on a local computer.

\section*{Funding}
This work was supported by the National Science Foundation [DMS-2307115 to M.K. and L.R.-B., TRIPODS CISE-1934846 to H.G. and M.K., DMS-2052525 to Z.C. and W.Z., DMS-2140982 to Z.C. and W.Z., DMS-2244976 to Z.C. and W.Z.]; and the Air Force Office of Scientific Research [FA9550-21-1-0354 to Z.C., H.G., M.K., and L.R.-B.].
%%%%%%%%%%%%%%%%%%%%%%%%%%%%%%%%%%%%%%%%%%%%%%%%%%%%%%%%%%%%

\bibliographystyle{abbrv}
\bibliography{IMAbibfile}

%\newpage

\appendix

\section{Notation for the proofs}We denote by $A\lesssim B$ if there are some $c,d>0$, such that $A\leq cB+d$; and $A\asymp B$ if both $A\lesssim B$ and $B\lesssim A$ hold. For a bounded set $\Omega\subset\mathbb{R}^d$, $\text{diam}(\Omega) = \sup_{x,y\in\Omega}\norm{x-y}_2$, where $\norm{\cdot}_2$ is the Euclidean norm on $\mathbb{R}^d$. Moreover, given a probability density $p(x)$, we use $M_r(p)$ to denote the $r$-th moment of $p(x)$. For convenience, we will abuse notation and use symbols $p,q$ and $P,Q$, to represent probability distributions as well as the density functions associated with them. Whether
a character refers to a probability distribution or a density should be clear from the context.

\section{Additional lemma of \Cref{thm:agnostic}}\label{appendix:agnostic}
For the Lipschitz-regularized KL-divergence, we have the following lemma similar to \Cref{lemma:bounded_div}.
\begin{lemma}\label{lemma:bounded_div_KL}
	For the KL case, i.e., $f_{\text{KL}}^*(y) = e^{y-1}$ and any non-negative measures $P$ and $Q$ defined on some bounded $\Omega\subset\mathbb{R}^d$ with non-zero integrals, $\Gamma = \text{Lip}_L(\Omega)$, we have 
	\begin{equation}\label{eq:gammaequalF_KL}
		\sup_{\gamma\in\Gamma}\left\{\int_\Omega \gamma(x) \diff{P}-\int_\Omega f_{\text{KL}}^*[\gamma(x)]\diff{Q}\right\} = \sup_{\gamma\in\mathcal{F}}\left\{\int_\Omega \gamma(x)\diff{P}-\int_\Omega f_{\text{KL}}^*[\gamma(x)]\diff{Q}\right\},
	\end{equation}
	where 
	\[
	\mathcal{F} = \left\{\gamma\in\text{Lip}_{L}(\Omega):\ln\frac{\int_\Omega \diff{P}}{\int_\Omega \diff{Q}}+1-L\cdot\text{diam}(\Omega)\leq\gamma\leq \ln\frac{\int_\Omega\diff{P}}{\int_\Omega \diff{Q}}+1+L\cdot\text{diam}(\Omega)\right\}.
	\]
\end{lemma}
\begin{proof}
	For any fixed $\gamma\in\Gamma$, define
	\[
	h(\nu) = \int_\Omega \left(\gamma(x)+\nu\right) \diff{P}-\int_\Omega f_{\text{KL}}^*[\gamma(x)+\nu]\diff{Q}.
	\]
	Since $\sup_{x\in\Omega}\gamma(x)- \inf_{x\in\Omega}\gamma(x)\leq L\cdot\text{diam}(\Omega)$, 
	interchanging the integration with differentiation is allowed by the dominated convergence theorem: 
	\[
	h'(\nu) = \int_\Omega \diff{P} -\int_\Omega f_{\text{KL}}^{*\prime}(\gamma+\nu)\diff{Q}.
	\]
	If $\inf_{x\in\Omega}\gamma(x)>\ln\frac{\int \diff{P}}{\int\diff{Q}}+1$, then $h'(0)<0$. So there exists some $\nu_0<0$ such that $h(\nu_0)>h(0)$. This indicates the supremum on the left side of \eqref{eq:gammaequalF_KL} is attained only if $\sup_{x\in\Omega}\gamma(x)\leq \ln\frac{\int \diff{P}}{\int \diff{Q}}+1 + L\cdot\text{diam}(\Omega)$. On the other hand, if $\sup_{x\in\Omega}\gamma(x)<\ln\frac{\int \diff{P}}{\int \diff{Q}}+1$, then $h'(0)>0$. So there exists some $\nu_0>0$ such that $h(\nu_0)>h(0)$.
	This indicates that the supremum on the left side of \eqref{eq:gammaequalF_KL} is attained only if $\inf_{x\in\Omega}\gamma(x)\geq \ln\frac{\int \diff{P}}{\int \diff{Q}}+1-L\cdot\text{diam}(\Omega)$.
\end{proof}

\section{Proof of \Cref{thm:firstvariation}}\label{appendix:firstvariation}
\begin{proof}[Proof of \Cref{thm:firstvariation}]
	The existence and uniqueness of $\gamma^\star$ follow from Theorem 4.9 in \cite{dupuis2022formulation} and Theorem 25 in \cite{birrell2020f}. We extend $\gamma^\star$ from $\text{supp}(P)\cup\text{supp}(Q)$ to all of $\R^d$ by
	\begin{equation}\label{eq:Lip_extension}
		\hat{\gamma}(y)= \sup_{x\in\text{supp}(P)\cup\text{supp}(Q)}\{\gamma^\star(x)+L\abs{x-y}\}.
	\end{equation}
	And it is a well-known result (e.g., see the proof of Lemma 2.3 in \cite{gu2022lipschitz}) that $\hat{\gamma}$ is $L$-Lipschitz continuous on $\R^d$ and 
	\begin{equation}\label{eq:lipextension_sup}
		\hat{\gamma} = \sup_h\{h(x):h\in\text{Lip}_L(\R^d),h(y) = \gamma^\star(y), \forall y\in\text{supp}(P)\cup\text{supp}(Q)\}.
	\end{equation}
	We need to show that 
	\begin{equation}
		\liminf_{\epsilon\to0+}\frac{1}{\epsilon}\left(D_{\alpha}^L(P+\epsilon\rho\|Q)-D_{\alpha}^L(P\|Q)\right)\geq\int\hat{\gamma}\diff{\rho},
	\end{equation}
	and
	\begin{equation}\label{ineq:lessthan}
		\limsup_{\epsilon\to0+}\frac{1}{\epsilon}\left(D_{\alpha}^L(P+\epsilon\rho\|Q)-D_{\alpha}^L(P\|Q)\right)\leq\int\hat{\gamma}\diff{\rho}.
	\end{equation}
	If $P+\epsilon \rho\in\CP_1(\R^d)$, then by \Cref{thm:agnostic}, $D_{\alpha}^L(P+\epsilon \rho\|Q)<\infty$ and thus we have
	\begin{align*}
		D_{\alpha}^L(P+\epsilon \rho\|Q) &= \sup_{\gamma\in\text{Lip}_L(\R^d)}\left(\E_{P+\epsilon \rho}[\gamma]-\E_Q[f_\alpha^*(\gamma)]\right)\\
		&\geq \E_{P+\epsilon \rho}[\hat{\gamma}]-\E_Q[f_\alpha^*(\hat{\gamma})]\\
		&= \epsilon\int_{\R^d}\hat{\gamma}\diff{\rho} + \E_{P}[\hat{\gamma}]-\E_Q[f_\alpha^*(\hat{\gamma})]\\
		&= \epsilon\int_{\R^d}\hat{\gamma}\diff{\rho} + D_{\alpha}^L(P\|Q).
	\end{align*}
	Thus, we have
	\begin{equation}\label{eq:lowersemicontinuous}
		\liminf_{\epsilon\to0+}\frac{1}{\epsilon}\left(D_{\alpha}^L(P+\epsilon\rho\|Q)-D_{\alpha}^L(P\|Q)\right)\geq\int\hat{\gamma}\diff{\rho}.
	\end{equation}
	To prove the other direction, we define $F(\epsilon) = D_{\alpha}^L(P+\epsilon \rho\|Q)$. Then by Theorem 18 in \cite{birrell2020f}, $F(\epsilon)$ is convex, lower semi-continuous
	and finite on $[0,\epsilon_0]$ for some $\epsilon_0>0$. Due to the convexity of $F$, it is differentiable on $(0,\epsilon_0)$ except for a countable number of points. If $\hat{\gamma}_\epsilon$ is the optimizer for $D_{\alpha}^L(P+\epsilon \rho\|Q)$, similar to \eqref{eq:lowersemicontinuous}, we have for $\delta>0$ sufficiently small
	\begin{equation}
		D_{\alpha}^L(P+(\epsilon+\delta)\rho\|Q)-D_{\alpha}^L(P+\epsilon\rho\|Q)
		\geq 
		\delta\int\hat{\gamma}_\epsilon\diff{\rho},
	\end{equation}
	and
	\begin{equation}
		D_{\alpha}^L(P+(\epsilon-\delta)\rho\|Q)-D_{\alpha}^L(P+\epsilon\rho\|Q)
		\geq 
		-\delta\int\hat{\gamma}_\epsilon\diff{\rho}.
	\end{equation}
	If $F$ is differentiable at $\epsilon$, this implies that
	\begin{align*}
		\int\hat{\gamma}_\epsilon\diff{\rho} &\leq \lim_{\delta\to0}\frac{1}{\delta}\left(D_{\alpha}^L(P+(\epsilon+\delta)\rho\|Q)-D_{\alpha}^L(P+\epsilon\rho\|Q)\right)\\
		&= F'(\epsilon)\\
		&=\lim_{\delta\to0}\frac{1}{\delta}\left(D_{\alpha}^L(P+\epsilon\rho\|Q)-D_{\alpha}^L(P+(\epsilon-\delta)\rho\|Q)\right)\\
		&\leq \int\hat{\gamma}_\epsilon\diff{\rho}.
	\end{align*}
	Consequently,
	\begin{equation}
		F'(\epsilon)=\int\hat{\gamma}_\epsilon\diff{\rho}.
	\end{equation}
	Let $F'_+(0)$ be the right derivative at $\epsilon=0$, i.e., $F'_+(0) = \lim_{\epsilon\to0^+}\frac{1}{\epsilon}(F(\epsilon)-F(0))$. By convexity, for any sequence $\epsilon_n$ such that $F$ is differentiable at $\epsilon_n$ and $\epsilon_n\searrow 0$, we have
	\begin{equation}
		F'_+(0) = \lim_{n\to\infty} F'(\epsilon_n) = \lim_{n\to\infty} \int\hat{\gamma}_{\epsilon_n}\diff{\rho}.
	\end{equation}
	We write $\R^d = \cup_{m\in\mathbb{N}} K_m$ with $K_m\subset\R^d$ being a compact set and $K_m\subset K_{m+1}$. The optimizers $\hat{\gamma}_{\epsilon_n}$ are unique. Moreover, by \Cref{lemma:supnorm}, they satisfy $\abs{\hat{\gamma}_{\epsilon_n}(x)}\leq L(\abs{x}+R)+M_n$, where 
	\begin{equation}\label{inf_uniform}
		M_n = \inf_M\left\{(M+LR)+L\int_{\R^d}\abs{x}\diff{P}+\epsilon_nL\int_{\R^d}\abs{x}\diff{\rho}<f_\alpha^*(M-3LR)\int_{\abs{x}<2R}\diff{Q}\right\}
	\end{equation}
	where $R>0$ is fixed for all $n$ such that $\int_{\abs{x}<2R}\diff{Q}>0$. Thus, by the linear dependence on $\epsilon_n$ on the left side inside the infimum in \eqref{inf_uniform}, we have $M_n\leq \overline{M}$ for all sufficiently large $n$. Therefore, the sequence $\{\hat{\gamma}_{\epsilon_n}\}$ is  equibounded and equicontinuous on $K_m$. By the Arzel\`a-Ascoli theorem, there exists a subsequence of $\hat{\gamma}_{\epsilon_n}$ that converges uniformly in $K_m$. Using diagonal argument, by taking subsequences sequentially along $\{K_m\}_{m\in\mathbb{N}}$ we
	conclude there exists a subsequence such that $\hat{\gamma}_{\epsilon_{n_k}}$ converges uniformly in any $K_m$ and thus $\hat{\gamma}_{\epsilon_{n_k}}$ converges pointwise in $\R^d$. Let $\hat{\gamma}_{0}$ be the limit, then $\hat{\gamma}_{0}$ is $L$-Lipschitz due to the uniform convergence of $L$-Lipschitz functions. For simplicity, we also denote by $\hat{\gamma}_{\epsilon_n}$ the convergent subsequence. Thus, given $\rho_\pm\in\CP_1(\R^d)$, we have  by the dominated convergence theorem,
	\begin{equation}
		F_+'(0) = \lim_{n\to\infty}\int_{\R^d}\hat{\gamma}_{\epsilon_n}\diff{\rho} = \int_{\R^d}\hat{\gamma}_0\diff{\rho}.
	\end{equation}
	By the lower semi-continuity of $D_{\alpha}^L(\cdot\|Q)$, we have
	\begin{align*}
		D_{\alpha}^L(P\|Q) &\leq \liminf_{n\to\infty} D_{\alpha}^L(P+\epsilon_n\rho\|Q)\\
		&= \liminf_{n\to\infty}\left\{\E_{P+\epsilon_n\rho}[\hat{\gamma}_{\epsilon_n}]-\E_Q[f_\alpha^*(\hat{\gamma}_{\epsilon_n})]\right\}\\
		&= \lim_{n\to\infty}\E_{P+\epsilon_n\rho}[\hat{\gamma}_{\epsilon_n}]-\limsup_{n\to\infty}\E_Q[f_\alpha^*(\hat{\gamma}_{\epsilon_n})]\\
		&= \E_{P}[\hat{\gamma}_{0}]-\limsup_{n\to\infty}\E_Q[f_\alpha^*(\hat{\gamma}_{\epsilon_n})]\\
		&\leq \E_{P}[\hat{\gamma}_{0}]-\E_Q[f_\alpha^*(\hat{\gamma}_{0})]\\
		&\leq D_{\alpha}^L(P\|Q),
	\end{align*}
	where in the third equality we use the dominated convergence theorem, and in the second-to-last inequality we apply the Fatou's lemma. Thus, we have $\hat{\gamma}_{0} = \hat{\gamma}$ $P,Q$-- a.s., and $\hat{\gamma}_{0}\leq \hat{\gamma}$ for all $x\in\R^d$. The latter is true since $\hat{\gamma}$ is the Lipschitz extension of $\gamma_{0}^\star$ by \eqref{eq:Lip_extension}, and \eqref{eq:lipextension_sup} guarantees that $\hat{\gamma}$ is the supremum of all the $L$-Lipschitz functions whose restriction on $\text{supp}(P)\cup\text{supp}(Q)$ is equal to $\hat{\gamma}_{0}$. It can be shown (as in the beginning of the proof of Theorem 1 in \cite{gu2022lipschitz}) that $\rho_-$ is absolutely continuous with respect to $P$, then we have
	\begin{equation}
		F'_+(0)= \int\gamma_0^\star\diff{\rho}=\int\gamma_0^\star\diff{\rho_+}-\int\gamma_0^\star\diff{\rho_-}=\int\gamma_0^\star\diff{\rho_+}-\int\hat{\gamma}\diff{\rho_-}\leq \int\hat{\gamma}\diff{\rho}.
	\end{equation}
	Thus, \eqref{ineq:lessthan} is proved.
\end{proof}

\section{Proofs of \Cref{thm:finite} and \Cref{thm:finite_KL}}\label{appendix:finiteness}

\begin{proof}[Proof of \Cref{thm:finite}]
	\textbf{1. Sufficiency}.\\
	Let $\Gamma = \text{Lip}_L(\mathbb{R}^d)$, and we have
	\begin{align*}
		D_{\alpha}^L(P\|Q) &= \sup_{\gamma\in\Gamma}\left\{\int \gamma(x) p(x)\diff{x}-\int f_\alpha^*[\gamma(x)]q(x)\diff{x}\right\}\\
		&\leq \sup_{\gamma\in\text{Lip}_L(\norm{x}< R)}\left\{\int_{\norm{x}< R} \gamma(x) p(x)\diff{x}-\int_{\norm{x}< R} f_\alpha^*[\gamma(x)]q(x)\diff{x}\right\}\\
		&\quad + \sup_{\gamma\in\text{Lip}_L(\norm{x}\geq R)}\left\{\int_{\norm{x}\geq R} \gamma(x) p(x)\diff{x}-\int_{\norm{x}\geq R} f_\alpha^*[\gamma(x)]q(x)\diff{x}\right\}\\
		&\coloneqq I_1 + I_2.
	\end{align*}
	For $I_1$, by Lemma~\ref{lemma:bounded_div}, we have
	\begin{align*}
		I_1 &\leq C\int_{\norm{x}< R}p(x)\diff{x} + \left(\alpha^{-1}(\alpha-1)^{\frac{\alpha}{\alpha-1}}C^{\frac{\alpha}{\alpha-1}}+\alpha^{-1}(\alpha-1)^{-1}\right)\int_{\norm{x}< R}q(x)\diff{x}<\infty,
	\end{align*}
	where $C = (\alpha-1)^{-1}\left(\frac{\int_{\norm{x}< R} p(x)\diff{x}}{\int_{\norm{x}< R} q(x)\diff{x}}\right)^{\alpha-1} + 2LR$.
	
	For $I_2$, we have
	\begin{align*}
		\int_{\norm{x}\geq R} \gamma(x) p(x)\diff{x}-\int_{\norm{x}\geq R} f_\alpha^*[\gamma(x)]q(x)\diff{x} &= \int_{\norm{x}\geq R}  p(x)\left(\gamma(x) - f_\alpha^*[\gamma(x)]\frac{q(x)}{p(x)}\right)\diff{x}.
	\end{align*}
	(i) If $d<\beta_1\leq d+1$ and $\beta_2-\beta_1<\frac{\beta_1-d}{\alpha-1}$:\\
	Note that the set of bounded $L$-Lipschitz functions on $\{x:\norm{x}\geq R\}$ is a subset of $\CM_b(x:\norm{x}\geq R)$, and the supremum over all the $L$-Lipschitz functions can be bounded by taking the supremum over all the measurable functions. Moreover, we can solve for the optimal $\hat{\gamma}(x)$ that maximizes $\gamma(x) - f_\alpha^*[\gamma(x)]\frac{q(x)}{p(x)}$ within the class of measurable functions: the stationary point of $\gamma(x) - f_\alpha^*[\gamma(x)]\frac{q(x)}{p(x)}$ in $\gamma$ for every $x$ provides $\hat{\gamma}(x) = \frac{1}{\alpha-1}\left(\frac{p(x)}{q(x)}\right)^{\alpha-1}$. Therefore, we have
	\begin{align*}
		&\sup_{\gamma\in\text{Lip}_L(x:\norm{x}\geq R)}\int_{\norm{x}\geq R}  p(x)\left(\gamma(x) - f_\alpha^*[\gamma(x)]\frac{q(x)}{p(x)}\right)\diff{x}\\
		&\leq \int_{\norm{x}\geq R}  p(x)\left(\hat{\gamma}(x) - f_\alpha^*[\hat{\gamma}(x)]\frac{q(x)}{p(x)}\right)\diff{x}\\
		&= \int_{\norm{x}\geq R} \frac{1}{\alpha(\alpha-1)}\left(\left[\frac{p(x)}{q(x)}\right]^\alpha-1\right)q(x)\diff{x}\\
		&\asymp\int_{\norm{x}\geq R}\norm{x}^{\alpha(\beta_2-\beta_1)-\beta_2}\diff{x}<\infty,
	\end{align*}
	since $\alpha(\beta_2-\beta_1)-\beta_2 = (\alpha-1)(\beta_2-\beta_1)-\beta_1<-d$.\\
	(ii) If $\beta_1>d+1$: the proof follows that of \Cref{thm:agnostic}.\\
	
	\textbf{2. Necessity}.\\
	Suppose $\beta_1\leq d+1$ and $\beta_2-\beta_1\geq\frac{\beta_1-d}{\alpha-1}$. We split $\beta_2-\beta_1\geq\frac{\beta_1-d}{\alpha-1}$ into two cases.\\
	(i) If $\beta_2-\beta_1\geq \frac{1}{\alpha-1}$:\\
	Let $\widehat{\gamma}(x) = \tau\norm{x}$, where $\tau\in(0,L]$ is to be determined. Then we have $\widehat{\gamma}\in\text{Lip}_L(\mathbb{R}^d)$. Using this $\widehat{\gamma}$, we have
	\begin{align*}
		D_{\alpha}^L(P\|Q) &\geq \int \widehat{\gamma}(x) p(x)\diff{x}-\int f_\alpha^*[\widehat{\gamma}(x)]q(x)\diff{x}\\
		&= \int_{\norm{x}<R}\widehat{\gamma}(x) p(x) - f_\alpha^*[\widehat{\gamma}(x)]q(x) \diff{x} + \int_{\norm{x}\geq R}\widehat{\gamma}(x) p(x) - f_\alpha^*[\widehat{\gamma}(x)]q(x)\diff{x}.
	\end{align*}
	It is straightforward that the first integral over $\norm{x}<R$ is finite. For the latter one, we have 
	\begin{align*}
		\int_{\norm{x}\geq R}\widehat{\gamma}(x) p(x)\diff{x} - \int_{\norm{x}\geq R}f_\alpha^*[\widehat{\gamma}(x)]q(x)\diff{x} &\gtrsim \int_{\norm{x}\geq R}\left(\tau\norm{x}^{1-\beta_1} - \tau^{\frac{\alpha}{\alpha-1}}\norm{x}^{\frac{\alpha}{\alpha-1}-\beta_2}\right)\diff{x}.
	\end{align*}
	We need to show the right-hand side is infinite. First, since $\frac{\alpha}{\alpha-1}>1$, we can choose $\tau$ sufficiently small such that $\tau>\tau^{\frac{\alpha}{\alpha-1}}$. Moreover, by the assumption, we have $1-\beta_1\geq-d$ and $\frac{\alpha}{\alpha-1}-\beta_2\leq1-\beta_1$, so that we have 
	\begin{align*}
		\int_{\norm{x}\geq R}\left(\tau\norm{x}^{1-\beta_1} - \tau^{\frac{\alpha}{\alpha-1}}\norm{x}^{\frac{\alpha}{\alpha-1}-\beta_2}\right)\diff{x} = \infty,
	\end{align*}
	and thus $D_{\alpha}^L(P\|Q)=\infty$.\\
	(ii) If $\frac{\beta_1-d}{\alpha-1}\leq\beta_2-\beta_1<\frac{1}{\alpha-1}$:\\
	Define
	\begin{equation*}
		\widehat{\gamma}(x)=
		\begin{cases}
			\tau R^{(\alpha-1)(\beta_2-\beta_1)}, & \text{if } \norm{x} < R ;\\
			\tau\norm{x}^{(\alpha-1)(\beta_2-\beta_1)}, & \text{if } \norm{x} \geq R,
		\end{cases}
	\end{equation*}
	where $\tau\in(0,L]$ is to be determined. Since in this case we have $(\beta_2-\beta_1)(\alpha-1)<1$, we have $\widehat{\gamma}(x)\in\text{Lip}_L(\mathbb{R}^d)$ if we pick $R$ sufficiently large which is independent of $\tau\leq L$. Using this $\widehat{\gamma}(x)$, we have
	\begin{align*}
		D_{\alpha}^L(P\|Q) &\geq \int \widehat{\gamma}(x) p(x)\diff{x}-\int f_\alpha^*[\widehat{\gamma}(x)]q(x)\diff{x}\\
		&= \int_{\norm{x}<R}\widehat{\gamma}(x) p(x) - f_\alpha^*[\widehat{\gamma}(x)]q(x) \diff{x} + \int_{\norm{x}\geq R}\widehat{\gamma}(x) p(x) - f_\alpha^*[\widehat{\gamma}(x)]q(x)\diff{x}.
	\end{align*}
	By the definition of $\widehat{\gamma}$, we know that the first integral over $\norm{x}<R$ is finite. For the latter one, we have in this case
	\begin{align*}
		\int_{\norm{x}\geq R}\widehat{\gamma}(x) p(x)\diff{x} &- \int_{\norm{x}\geq R}f_\alpha^*[\widehat{\gamma}(x)]q(x)\diff{x}\\ &\gtrsim \int_{\norm{x}\geq R}\left(\tau\norm{x}^{(\alpha-1)(\beta_2-\beta_1)-\beta_1} - \tau^{\frac{\alpha}{\alpha-1}}\norm{x}^{(\alpha-1)(\beta_2-\beta_1)-\beta_1}\right)\diff{x}.
	\end{align*}
	We show the right-hand side is infinite. Again, we can choose $\tau$ sufficiently small such that $\tau>\tau^{\frac{\alpha}{\alpha-1}}$. On the other hand, by the assumption in this case, we have $(\alpha-1)(\beta_2-\beta_1)-\beta_1\geq -d$, so that we have 
	\begin{align*}
		\int_{\norm{x}\geq R}\left(\tau\norm{x}^{(\alpha-1)(\beta_2-\beta_1)-\beta_1} - \tau^{\frac{\alpha}{\alpha-1}}\norm{x}^{(\alpha-1)(\beta_2-\beta_1)-\beta_1}\right)\diff{x} = \infty,
	\end{align*}
	hence $D_{\alpha}^L(P\|Q)=\infty$.
\end{proof}

\begin{proof}[Proof of \Cref{thm:finite_KL}]
	Same as in the beginning of the proof of \Cref{thm:finite}, we can split $D_{\text{KL}}^L(P\|Q)$ into $I_1$ and $I_2$, where $I_1$ is bounded by \Cref{lemma:bounded_div_KL} with appropriate $R$. 
	
	For $I_2$, we have
	\begin{align*}
		&\sup_{\gamma\in\text{Lip}_L(x:\norm{x}\geq R)}\int_{\norm{x}\geq R} \gamma(x) p(x)\diff{x}-\int_{\norm{x}\geq R} f_{\text{KL}}^*[\gamma(x)]q(x)\diff{x}\\
		&\leq \sup_{\gamma\in\CM_b(x:\norm{x}\geq R)}\int_{\norm{x}\geq R} \gamma(x) p(x)\diff{x}-\int_{\norm{x}\geq R} f_{\text{KL}}^*[\gamma(x)]q(x)\diff{x}\\
		&= \int_{\norm{x}\geq R} \ln\frac{p(x)}{q(x)}p(x)\diff{x}\\
		&\asymp\int_{\norm{x}\geq R}\norm{x}^{-\beta_1}\ln\norm{x}\diff{x}<\infty,
	\end{align*}
	since $\beta_1>d$ and the equality is due to the dual formula of KL divergence.
\end{proof}

\begin{proof}[Proof of \Cref{cor:lowdimension}]
	Note the change-of-variable formula
	\begin{align*}
		\int_{\mathbb{R}^d}\gamma(y)\diff{p_{\CM}}(y) = \int_{\mathbb{R}^{d^*}}(\gamma\circ\varphi)(x)\cdot p(x)\diff{x},\,\text{(similarly for $q_{\CM}$ and $q$)}
	\end{align*}
	and $\gamma\circ\varphi$ is an $LL^*$-Lipschitz function on $\mathbb{R}^{d^*}$ for any $\gamma\in\text{Lip}_L(\mathbb{R}^d)$. Then the proof of Theorem~\ref{thm:finite} can be followed.
\end{proof}

\section{Proofs of results in \Cref{sec:finite-sample}}\label{appendix:thm2}
To prove \Cref{thm:samplecomplexity1}, we need a few lemmas. Let $x_1,x_2,\dots,x_m\in\R^d$ be i.i.d. samples of distribution $P$, and $P_m$ be the corresponding empirical distributions. We define $L_2(P_m)$ the metric between any functions $f,g$ as $L_2(P_m)(f,g) = \sqrt{\frac{1}{m}\sum_{i=1}^m \abs{f(x_i)-g(x_i)}^2}$.
\begin{lemma}[Metric entropy with empirical measures]\label{lemma:metric_entropy}
	Let $\mathcal{F}$ be a class of real-valued functions on $\R^d$ and $0\in\mathcal{F}$. Let $\xi = \{\xi_1,\xi_2,\dots,\xi_m\}$ be a set of independent random variables that take values on $\{-1,1\}$ with equal probabilities (also known as Rademacher variables). Suppose $X = \{x_1,x_2,\dots,x_m\}\subset\R^d$ are i.i.d. samples of distribution $P$, then we have 
	\begin{align*}
		\E_\xi\sup_{f\in\mathcal{F}}\abs{\frac{1}{m}\sum_{i=1}^m\xi_i f(x_i)}\leq \inf_{0<\theta<M_X}\left(4\theta+\frac{12}{\sqrt{m}}\int_{\theta}^{M_{X}}\sqrt{\ln\CN(\mathcal{F},\delta,L_2(P_m))}\diff{\delta}\right),
	\end{align*}
	where $M_{X} = \sup_{f\in\mathcal{F}}\sqrt{\frac{1}{m}\sum_{i=1}^m \abs{f(x_i)}^2}$.
\end{lemma}
\begin{proof}
	Let $N\in\mathbb{N}$ be an arbitrary positive integer and $\delta_k = M_{X}\cdot 2^{-(k-1)}$, $k=1,\dots,N$, with $M_{X} = \sup_{f\in\mathcal{F}}\sqrt{\frac{1}{m}\sum_{i=1}^m \abs{f(x_i)}^2}$. Let $V_k$ be the cover achieving $\CN(\mathcal{F},\delta_k,L_2(P_m))$, and denote $\abs{V_k} = \CN(\mathcal{F},\delta_k,L_2(P_m))$. For any $f\in\mathcal{F}$, let $\pi_k(f)\in V_k$, such that 
	\begin{equation}
		\sqrt{\frac{1}{m}\sum_{i=1}^m\abs{f(x_i)-\pi_k(f)(x_i)}^2}\leq \delta_k.
	\end{equation}
	We have
	\begin{align*}      &\E_\xi\sup_{f\in\mathcal{F}}\abs{\frac{1}{m}\sum_{i=1}^m\xi_i f(x_i)}\\
		&\leq \E_\xi\sup_{f\in\mathcal{F}}\abs{\frac{1}{m}\sum_{i=1}^m\xi_i\left( f(x_i)-\pi_N(f)(x_i)\right)} + \sum_{j=1}^{N-1}\E_\xi\sup_{f\in\mathcal{F}}\abs{\frac{1}{m}\sum_{i=1}^m\xi_i\left( \pi_{j+1}(f)(x_i)-\pi_j(f)(x_i)\right)}\\
		&\quad + \E_\xi\sup_{f\in\mathcal{F}}\abs{\frac{1}{m}\sum_{i=1}^m\xi_i \pi_1(f)(x_i)}.
	\end{align*}
	For the third term, observe that it suffices to take $V_1=\{0\}$ so that $\pi_1(f)$ is the zero function and the third term vanishes. The first term can be bounded using Cauchy-Schwartz inequality as
	\begin{align*}
		\E_\xi\sup_{f\in\mathcal{F}}\abs{\frac{1}{m}\sum_{i=1}^m\xi_i\left( f(x_i)-\pi_N(f)(x_i)\right)} &\leq \frac{1}{m}\sqrt{\sum_{i=1}^m\E_{\xi}(\xi_i)^2}\sqrt{\sup_{f\in\mathcal{F}}\sum_{i=1}^m\left(f(x_i)-\pi_N(f)(x_i)\right)^2}\\
		&\leq \delta_N.
	\end{align*}
	
	To handle the middle term, for each $j$, let $W_j = \{\pi_{j+1}(f)-\pi_j(f):f\in\mathcal{F}\}$. We have $\abs{W_j}\leq\abs{V_{j+1}}\abs{V_j}\leq\abs{V_{j+1}}^2$, then
	\begin{align*}
		\sum_{j=1}^{N-1}\E_\xi\sup_{f\in\mathcal{F}}\abs{\frac{1}{m}\sum_{i=1}^m\xi_i\left( \pi_{j+1}(f)(x_i)-\pi_j(f)(x_i)\right)} = \sum_{j=1}^{N-1}\E_\xi\sup_{w\in W_j}\abs{\frac{1}{m}\sum_{i=1}^m\xi_i w(x_i)}.
	\end{align*}
	Moreover, we have
	\begin{align*}
		&\sup_{w\in W_j}\sqrt{\sum_{i=1}^m w(x_i)^2}\\
		&= \sup_{f\in\mathcal{F}}\sqrt{\sum_{i=1}^m\left(\pi_{j+1}(f)(x_i)-\pi_j(f)(x_i)\right)^2}\\
		&\leq \sup_{f\in\mathcal{F}}\sqrt{\sum_{i=1}^m\left(\pi_{j+1}(f)(x_i)-f(x_i)\right)^2} + \sup_{f\in\mathcal{F}}\sqrt{\sum_{i=1}^m\left(f(x_i)-\pi_j(f)(x_i)\right)^2}\\
		&\leq \sqrt{m}\delta_{j+1} + \sqrt{m}\cdot\delta_{j}\\
		&=3\sqrt{m}\delta_{j+1}.
	\end{align*}
	By the Massart finite class lemma (see, e.g. \cite{mohri2018foundations}), we have
	\begin{align*}
		\E_\xi\sup_{w\in W_j}\abs{\frac{1}{m}\sum_{i=1}^m\xi_i w(x_i)}\leq\frac{3\sqrt{m}\delta_{j+1}\sqrt{2\ln\abs{W_j}}}{m}\leq \frac{6\delta_{j+1}\sqrt{\ln\abs{V_{j+1}}}}{\sqrt{m}}.
	\end{align*}
	Therefore,
	\begin{align*}
		\E_\xi\sup_{f\in\mathcal{F}}\abs{\frac{1}{m}\sum_{i=1}^m\xi_i f(x_i)} &\leq \delta_N +\frac{6}{\sqrt{m}}\sum_{j=1}^{N-1}\delta_{j+1}\sqrt{\ln\mathcal{N}(\mathcal{F},\delta_{j+1},L_2(P_m))}\\
		&\leq \delta_N +\frac{12}{\sqrt{m}}\sum_{j=1}^{N}(\delta_j-\delta_{j+1})\sqrt{\ln\mathcal{N}(\mathcal{F},\delta_{j},L_2(P_m))}\\
		&\leq\delta_N + \frac{12}{\sqrt{m}}\int_{\delta_{N+1}}^{M_X}\sqrt{\ln\mathcal{N}(\mathcal{F},\delta,L_2(P_m))}\diff{\delta}.
	\end{align*}
	Finally, select any $\theta\in(0,M_X)$ and let $N$ be the largest integer with $\delta_{N+1}>\theta$, (implying $\delta_{N+2}\leq \theta$ and $\delta_N=4\delta_{N+2}\leq 4\theta$), so that
	\[
	\delta_N + \frac{12}{\sqrt{m}}\int_{\delta_{N+1}}^{M_X}\sqrt{\ln\mathcal{N}(\mathcal{F},\delta,L_2(P_m))}\diff{\delta}\leq 4\theta + \frac{12}{\sqrt{m}}\int_{\theta}^{M_X}\sqrt{\ln\mathcal{N}(\mathcal{F},\delta,L_2(P_m))}\diff{\delta}.
	\]
\end{proof}

\begin{lemma}\label{lemma:empirical_moments}
	Suppose $P_m$ is the empirical distribution of $P\in\CP_1(\R^d)$, and $\Lambda = \frac{1}{m}\sum_{i=1}^m\norm{x}^{\hat{\beta}}$ with $1\leq\hat{\beta}<\beta-d$, then for $1\leq z\leq\hat{\beta}$, we have 
	\begin{align*}
		\E_{P_m}\norm{x}^z \leq \Lambda +1.
	\end{align*}
\end{lemma}
\begin{proof}
	Note that $\norm{x}^z\leq\max\{1,\norm{x}^{\hat{\beta}}\}\leq 1 +\norm{x}^{\hat{\beta}}$, so we have the bound.
\end{proof}

We provide the following lemma that sets up a landmark for the magnitude of the Lipschitz functions under the supremum.
\begin{lemma}\label{lemma:supnorm}
	Suppose $\alpha>1$, and $P\in\CP_1(\R^d)$. Let $M(\gamma) = \sup_{\norm{x}=R}\abs{\gamma(x)}$, then there exists $\overline{M}$ that depends on $P,Q,L$ and $R$, such that
	\begin{equation*}
		D_{\alpha}^L(P\|Q) = \sup_{\substack{\gamma\in\text{Lip}_L(\mathbb{R}^d) \\ M(\gamma)\leq\overline{M}}}\left\{\E_P[\gamma]-\E_Q[f_\alpha^*(\gamma)]\right\},
	\end{equation*}
	where
	\begin{equation*}
		\overline{M}=\inf\left\{\hat{M}:(M(\gamma)+LR)\int \diff{P} + L\int \norm{x}\diff{P} - f_\alpha^*(M(\gamma)-3LR)\int_{\norm{x}<2R}\diff{Q}<0, \forall M(\gamma)>\hat{M}\right\}.
	\end{equation*}
\end{lemma}
\begin{proof}
	For any $\gamma\in\text{Lip}_L(
	\mathbb{R}^d)$, let 
	\begin{equation*}
		J_1 \coloneqq \int_{\norm{x}<R} \gamma(x) \diff{P}-\int_{\norm{x}<R} f_\alpha^*[\gamma(x)]\diff{Q},\quad J_2 \coloneqq \int_{\norm{x}\geq R} \gamma(x)\diff{P}-\int_{\norm{x}\geq R} f_\alpha^*[\gamma(x)]\diff{Q},
	\end{equation*}
	then
	\begin{align*}
		\int \gamma(x) \diff{P}-\int f_\alpha^*[\gamma(x)]\diff{Q} = J_1 + J_2.
	\end{align*}
	We have for any $\gamma\in\text{Lip}_L(
	\mathbb{R}^d)$,
	\begin{align*}
		J_1 &\leq \int_{\norm{x}<R} (M(\gamma)+LR)\diff{P} - \int_{\norm{x}<R}f_\alpha^*(M(\gamma)-3LR)\diff{Q}\\
		&= (M(\gamma)+LR)\cdot\int_{\norm{x}<R}\diff{P} - f_\alpha^*(M(\gamma)-3LR)\cdot\int_{\norm{x}<R}\diff{Q}.
	\end{align*}
	On the other hand, by the same argument in the proof of \Cref{thm:agnostic} (for proving $I_2<\infty$ therein), we have
	\begin{align*}
		J_2&\leq LR\int_{\norm{x}\geq R}\diff{P} + L\int_{\norm{x}\geq 2R}\norm{x}\diff{P} + M(\gamma)\int_{\norm{x}\geq R}\diff{P}\\
		&\quad-f_\alpha^*(M(\gamma)-3LR)\int_{R\leq\norm{x}<2R}\diff{Q},
	\end{align*}
	Both the upper bounds for $J_1$ and $J_2$ tend to $-\infty$ as $M(\gamma)\to\infty$. Thus, there exists such $\overline{M}$ as claimed. Moreover, we have 
	\begin{align*}
		J_1+J_2 &\leq (M(\gamma)+LR)\int \diff{P} + L\int \norm{x}\diff{P}\\
		&\quad - f_\alpha^*(M(\gamma)-3LR)\int_{\norm{x}<2R}\diff{Q}.
	\end{align*}
	Therefore, we can pick $\overline{M}>0$ as 
	\begin{equation*}
		\inf\left\{\hat{M}:(M(\gamma)+LR)\int \diff{P} + L\int \norm{x}\diff{P} - f_\alpha^*(M(\gamma)-3LR)\int_{\norm{x}<2R}\diff{Q}<0, \forall M(\gamma)>\hat{M}\right\},
	\end{equation*}
	and it is obvious that $\overline{M}>0$ only depends on $P,Q$ and $R$.
\end{proof}

Let $\overline{M}_{m,n}$ be the quantity in Lemma~\ref{lemma:supnorm} where $(P,Q)$ are replaced by their empirical counterparts $(P_m,Q_n)$, then $\overline{M}_{m,n}$ is a random variable. We have the following lemma to estimate the expectation of the 
$r$-th moment ($r\geq1$) of $\overline{M}_{m,n}$. The proof is different from that for Lemma~\ref{lemma:supnorm}.

\begin{lemma}\label{lemma:supnorm_random}
	Suppose $\alpha>1$, and $(P,Q)$ are distributions on $\mathbb{R}^d$ of heavy-tail $(\beta_1,\beta_2)$ with $\beta_1,\beta_2>d+r$ for some $r\geq1$. Let $M(\gamma) = \sup_{\norm{x}=R}\abs{\gamma(x)}$, then there exists $\overline{M}_{m,n}$ that depends on $P_m,Q_n$ and $R$, such that
	\begin{equation*}
		D_{\alpha}^L(P_m\|Q_n) = \sup_{\substack{\gamma\in\text{Lip}_L(\mathbb{R}^d) \\ M(\gamma)\leq\overline{M}_{m,n}}}\left\{\E_{P_m}[\gamma]-\E_{Q_n}[f_\alpha^*(\gamma)]\right\},
	\end{equation*}
	Moreover, we have 
	\begin{align*}
		\E_{X,Y}\left[\overline{M}_{m,n}^r\right] \leq M_{p,q,r},
	\end{align*}
	where $M_{p,q,r}$ depends on $\alpha,L,R, M_r(p)$ and $M_r(q)$, and is independent of $m,n$.
\end{lemma}
\begin{proof}
	We have
	\begin{align*}
		\E_{P_m}[\gamma]-\E_{Q_n}[f_\alpha^*(\gamma)] &\leq \sum_{i=1}^m\frac{M(\gamma)+L\abs{\norm{x_i}-R}}{m} - \sum_{j=1}^n\frac{f_\alpha^*\left(M(\gamma)-2LR-L\abs{\norm{y_j}-R}\right)}{n}.
	\end{align*}
	Hence $\overline{M}_{m,n}$ can be taken as
	\begin{align*}
		\overline{M}_{m,n} = \inf\left\{z: \sum_{i=1}^m\frac{s+L\abs{\norm{x_i}-R}}{m}<\sum_{j=1}^n\frac{f_\alpha^*\left(s-2LR-L\abs{\norm{y_j}-R}\right)}{n}, \forall s>z\right\}.
	\end{align*}
	Moreover, by Jensen's inequality, we have
	\begin{align*}
		\sum_{j=1}^n\frac{f_\alpha^*\left(s-2LR-L\abs{\norm{y_j}-R}\right)}{n}\geq f_\alpha^*\left( s-2LR-L\sum_{j=1}^n\frac{\abs{\norm{y_j}-R}}{n}\right),
	\end{align*}
	since the convex conjugate $f_\alpha^*$ is convex, and so that 
	\begin{align*}
		\overline{M}_{m,n}&\leq \inf\left\{z: \sum_{i=1}^m\frac{s+L\abs{\norm{x_i}-R}}{m}<f_\alpha^*\left( s-2LR-L\sum_{j=1}^n\frac{\abs{\norm{y_j}-R}}{n}\right), \forall s>z\right\}\\
		&\coloneqq \widetilde{M}_{m,n}.
	\end{align*}
	It is obvious that $\widetilde{M}_{m,n}$ solves the following equation in variable $z$:
	\begin{equation}\label{eq:falpha=linear}
		f_\alpha^*(z-c_1) = z + c_2,
	\end{equation}
	where
	\begin{align*}
		c_1 &= 2LR+L\sum_{j=1}^n\frac{\abs{\norm{y_j}-R}}{n},\\
		c_2&= \sum_{i=1}^m\frac{L\abs{\norm{x_i}-R}}{m}.
	\end{align*}
	Equation~\eqref{eq:falpha=linear} can be reformulated as to find $y^*$ that solves:
	\begin{equation}\label{eq:falpha=linear2}
		f_\alpha^*(y) - y =  c_1 + c_2,
	\end{equation}
	where $z-c_1=y$. We derive an upper bound for $y^*$ as follows. Let $g(y) = f_\alpha^*(y) - y$, then 
	\begin{align*}
		g'(y) = (\alpha-1)^{\frac{1}{\alpha-1}}y^{\frac{1}{\alpha-1}}\mathbf{1}_{y>0} - 1,
	\end{align*}
	such that $g'(y)\geq 1$ for $y>2^{\alpha-1}(\alpha-1)^{-1}$. Given that $g\left(2^{\alpha-1}(\alpha-1)^{-1}\right) = \frac{2^\alpha}{\alpha}+\frac{1}{\alpha(\alpha-1)}-\frac{2^{\alpha-1}}{\alpha-1}$, we can take $y^*\leq 2^{\alpha-1}(\alpha-1)^{-1} + c_1+c_2 + \abs{\frac{2^\alpha}{\alpha}+\frac{1}{\alpha(\alpha-1)}-\frac{2^{\alpha-1}}{\alpha-1}}$. Therefore, we have 
	\begin{align*}
		\overline{M}_{m,n}\leq \widetilde{M}_{m,n}=y^*+c_1\leq 2^{\alpha-1}(\alpha-1)^{-1} + 2c_1+c_2 + \abs{\frac{2^\alpha}{\alpha}+\frac{1}{\alpha(\alpha-1)}-\frac{2^{\alpha-1}}{\alpha-1}}.
	\end{align*}
	The claim follows since by Jensen's inequality, $\E_X\left[\left(\sum_{i=1}^m\frac{\norm{x_i}}{m}\right)^r\right]\leq \E_X\left[\sum_{i=1}^m\frac{\norm{x_i}^r}{m}\right]=M_r(p)$. (Similarly for $\E_Y\left[\left(\sum_{j=1}^n\frac{\norm{y_j}}{n}\right)^r\right]$.)
\end{proof}

\begin{proof}[Proof of \Cref{thm:samplecomplexity1}]
	Without loss of generality, we assume that both 
	\begin{align*}
		\int_{\norm{x}\leq1}p(x)\diff{x}>0,\quad
		\int_{\norm{x}\leq1}q(x)\diff{x}>0.
	\end{align*}
	Let $\Omega_{0} = \{x\in\mathbb{R}^d:\norm{x}\leq 1\}$ and $\Omega_k =\{x\in\mathbb{R}^d:2^{k-1}<\norm{x}\leq2^k\}$ for $k\geq 1$. For each $k\in\mathbb{N}$, the Lebesgue measure of $\{x:\diff(x,\Omega_k)\leq 1\}$ is bounded by $C_d2^{kd}$ for some $C_d>0$. Let $\Lambda_2 = \frac{1}{n}\sum_{j=1}^n\norm{y_j}^{\hat{\beta}_2}$, where $2+\frac{2\alpha}{\alpha-1}<\hat{\beta}_2<\frac{\beta_2}{d}-1$. By Markov's inequality, the mass or proportion of $Q_n$ that lies in $\Omega_k$ is bounded by 
	\begin{align*}
		\Pr(x\sim Q_n: \norm{x}>2^{k-1})&=\Pr(x\sim Q_n: \norm{x}^{\hat{\beta}_2}>2^{(k-1)\hat{\beta}_2})\\
		&\leq \frac{\E_{Q_n}\norm{x}^{\hat{\beta}_2}}{2^{(k-1)\hat{\beta}_2}} = \Lambda_2 2^{-(k-1)\hat{\beta}_2}.
	\end{align*}
	Let $M = \max(\overline{M},\overline{M}_{m,n})$, where $\overline{M}$ is the quantity in Lemma~\ref{lemma:supnorm} with $R=1$, and $\overline{M}_{m,n}$ is the random counterpart for $(P_m,Q_n)$ as defined in Lemma~\ref{lemma:supnorm_random}. $M$ is a random variable since $\overline{M}_{m,n}$ is random. Let $\mathcal{F}_M$ be the following class of functions
	\begin{equation}
		\mathcal{F}_{\alpha,M} = \left\{f_\alpha^*(\gamma):\gamma\in\text{Lip}_L(\mathbb{R}^d), \sup_{\norm{x}=1}\abs{\gamma(x)}\leq M\right\}.
	\end{equation}
	By formulas \eqref{eq:falpha} and \eqref{eq:falpha_derivative} , functions in $\mathcal{F}_{\alpha,M}$ have H\"older norm on $\Omega_k$ bounded by $C_\alpha(M^{\frac{\alpha}{\alpha-1}}+L^{\frac{\alpha}{\alpha-1}}2^{\frac{\alpha k}{\alpha-1}})$ for some $C_\alpha>0$ that only depends on $\alpha$. By Corollary 2.7.4 in \cite{wellner1996weak} with $V=d$ and $r=2$, we have 
	\begin{align*}
		&\ln(\mathcal{F}_{\alpha,M},\delta,L_2(Q_n))\\
		&\leq K\delta^{-d}\left(\sum_{k=0}^{\infty} (C_d2^{kd})^{\frac{2}{d+2}} \left(C_\alpha(M^{\frac{\alpha}{\alpha-1}}+L^{\frac{\alpha}{\alpha-1}}2^{\frac{\alpha k}{\alpha-1}})\right)^{\frac{2d}{d+2}}(\Lambda_2 2^{-(k-1)\hat{\beta}_2})^{\frac{d}{d+2}}\right)^{\frac{d+2}{2}}\\
		&\leq K\delta^{-d}(M+L)^{\frac{d\alpha}{\alpha-1}}\Lambda_2^{d/2} \left(\sum_{k=0}^{\infty}  2^{\frac{2kd}{d+2}+\frac{2\alpha kd}{(\alpha-1)(d+2)}-\frac{\hat{\beta}_2d(k-1)}{d+2}}\right)^{\frac{d+2}{2}}\\
		&\leq K\delta^{-d}(M+L)^{\frac{d\alpha}{\alpha-1}}\Lambda_2^{d/2}.
	\end{align*}
	where the constant $K$ can vary from line to line and does not depend on $n$, and the last step follows as the choice of $\hat{\beta}_2$ such that the series is summable over $k$ independent of $Q_n$. Then we have
	\begin{align*}
		&\E_{X,Y}\abs{D_{\alpha}^L(P_m\|Q_n)-D_{\alpha}^L(P\|Q)}\\
		&=\E_{X,Y}\abs{\sup_{\substack{\gamma\in\text{Lip}_L(\mathbb{R}^d) \\ M(\gamma)\leq\overline{M}_{m,n}}}\left\{\E_{P_m}[\gamma]-\E_{Q_n}[f_\alpha^*(\gamma)]\right\} - \sup_{\substack{\gamma\in\text{Lip}_L(\mathbb{R}^d) \\ M(\gamma)\leq\overline{M}}}\left\{\E_P[\gamma]-\E_Q[f_\alpha^*(\gamma)]\right\}}\\
		&\leq \E_{X,Y}\sup_{\substack{\gamma\in\text{Lip}_L(\mathbb{R}^d) \\ M(\gamma)\leq M}}\abs{\E_{P_m}[\gamma]-\E_{Q_n}[f_\alpha^*(\gamma)] - \left(\E_P[\gamma]-\E_Q[f_\alpha^*(\gamma)]\right)}\\
		&\leq\E_{X}\sup_{\gamma\in\text{Lip}_L(\mathbb{R}^d)}\abs{\E_P[\gamma]-\E_{P_m}[\gamma]} + \E_{X,Y}\sup_{\substack{\gamma\in\text{Lip}_L(\mathbb{R}^d) \\ M(\gamma)\leq M}}\abs{\E_{Q}[f_\alpha^*(\gamma)]-\E_{Q_n}[f_\alpha^*(\gamma)]}\\
		&\leq \E_{X}\sup_{\gamma\in\text{Lip}_L(\mathbb{R}^d)}\abs{\E_P[\gamma]-\E_{P_m}[\gamma]} + \E_{X}\E_{Y}\E_{Y'}\E_{\xi}\sup_{\substack{\gamma\in\text{Lip}_L(\mathbb{R}^d) \\ M(\gamma)\leq M}}\abs{\frac{1}{n}\sum_{j=1}^n\xi_i\left(f_\alpha^*[\gamma(y_j)]-f_\alpha^*[\gamma(y_j')]\right)}\\
		&\leq \E_{X}\sup_{\gamma\in\text{Lip}_L(\mathbb{R}^d)}\abs{\E_P[\gamma]-\E_{P_m}[\gamma]} + 2\E_{X}\E_{Y}\E_{\xi}\sup_{\substack{\gamma\in\text{Lip}_L(\mathbb{R}^d) \\ M(\gamma)\leq M}}\abs{\frac{1}{n}\sum_{j=1}^n\xi_i f_\alpha^*[\gamma(y_j)]}\\
		&\leq \E_{X}\sup_{\gamma\in\text{Lip}_L(\mathbb{R}^d)}\abs{\E_P[\gamma]-\E_{P_m}[\gamma]} + 2\E_{X,Y}\inf_{\theta>0}\left(4\theta+\frac{12}{\sqrt{n}}\int_{\theta}^{\infty}\sqrt{\ln\CN(\mathcal{F}_{\alpha,M},\delta,L_2(Q_n))}\diff{\delta}\right),
	\end{align*}
	where $\xi_i$'s are the Rademacher variables. 
	
	First note that the first term $\E_{X}\sup_{\gamma\in\text{Lip}_L(\mathbb{R}^d)}\abs{\E_P[\gamma]-\E_{P_m}[\gamma]}$ is the convergence rate of the Wasserstein-1 distance and the bound follows the result of Theorem 1 in \cite{fournier2015rate}:
	\begin{align*}
		\E_{X}\sup_{\gamma\in\text{Lip}_L(\mathbb{R}^d)}\abs{\E_P[\gamma]-\E_{P_m}[\gamma]}\leq \frac{CM_r^{1/r}(p)}{m^{1/d}},
	\end{align*}
	with $r=\frac{d}{d-1}$.
	For the second term, we have
	\begin{align*}
		&\E_{X,Y}\inf_{\theta>0}\left(4\theta+\frac{12}{\sqrt{n}}\int_{\theta}^{\infty}\sqrt{\ln\CN(\mathcal{F}_{\alpha,M},\delta,L_2(Q_n))}\diff{\delta}\right)\\
		&\leq \E_{X,Y}\inf_{\theta>0}\left(4\theta+\frac{12}{\sqrt{n}}K(M+L)^{\frac{d\alpha}{2(\alpha-1)}}\Lambda_2^{d/4}\int_{\theta}^{\infty}\delta^{-\frac{d}{2}}\diff{\delta}\right)\\
		&\leq \E_{X,Y}\inf_{\theta>0}\left(4\theta+\frac{12}{\sqrt{n}}K(M+L)^{\frac{d\alpha}{2(\alpha-1)}}\Lambda_2^{d/4}\cdot\frac{2}{2-d}\theta^{1-d/2}\right)\\
		&\leq \E_{X,Y}\left(4n^{-\frac{1}{d}}+12K(M+L)^{\frac{d\alpha}{2(\alpha-1)}}\Lambda_2^{d/4}\cdot\frac{2}{2-d}n^{-\frac{1}{d}}\right)\\
		&= 4n^{-\frac{1}{d}}+\frac{24K}{2-d}n^{-\frac{1}{d}}\cdot\E_{X,Y}\left[(M+L)^{\frac{d\alpha}{2(\alpha-1)}}\Lambda_2^{d/4}\right]
	\end{align*}
	where we pick $\theta = n^{-\frac{1}{d}}$. By the Cauchy-Schwartz inequality, we have
	\begin{align*}
		\E_{X,Y}\left[(M+L)^{\frac{d\alpha}{2(\alpha-1)}}\Lambda_2^{d/4}\right] &\leq \sqrt{\E_{X,Y}(M+L)^{\frac{d\alpha}{(\alpha-1)}}}\sqrt{\E_{Y}\Lambda_2^{d/2}}.
	\end{align*}
	Notice that $\E_{X,Y}(M+L)^{\frac{d\alpha}{(\alpha-1)}}$ is bounded by Lemma~\ref{lemma:supnorm_random} and the bound depends on $M_{\frac{d\alpha}{\alpha-1}}(p)$ and $M_{\frac{d\alpha}{\alpha-1}}(q)$. By Jensen's inequality, we have $\E_{Y}\Lambda_2^{d/2}\leq(\E_{Y}\Lambda_2^{d})^{1/2}$. And we have
	\begin{align*}
		\E_{Y}\Lambda_2^{d} = \E_{Y}\left(\frac{1}{n}\sum_{j=1}^n\norm{y_j}^{\hat{\beta}_2}\right)^{d}\leq \E_{Y}\left(\frac{1}{n}\sum_{j=1}^n\norm{y_j}^{\hat{\beta}_2d}\right)=M_{\hat{\beta}_2d}(q),
	\end{align*}
	where the inequality follows Jensen's inequality. Combining all these bounds, we obtain the result as in the statement of the theorem.
\end{proof}

\begin{proposition}\label{prop:samplecomplexityd=2}
	For $d=2$. Assume $(P,Q)$ are distributions on $\mathbb{R}^d$ of heavy-tail $(\beta_1,\beta_2)$, where $\beta_1>10$ and $\beta_2>18$. Suppose $\alpha$ satisfies $\frac{4\alpha}{\alpha-1}+4<\beta_1-2$ and $\frac{8\alpha}{\alpha-1}<\beta_2-10$, then if $m$ and $n$ are sufficiently large, we have
	\begin{equation}
		\E_{X,Y}\abs{D_{\alpha}^L(P_m\|Q_n)-D_{\alpha}^L(P\|Q)}\leq \frac{C_1\ln m}{m^{1/2}} + \frac{C_2\ln n}{n^{1/2}},
	\end{equation}
	where $C_1$ depends on $M_{r_1}(p)$ for any $r_1>2$ and $C_2$ depends on $M_{\frac{4\alpha}{\alpha-1}+4}(p)$, $M_{\frac{4\alpha}{\alpha-1}+4}(q)$ and $M_{dr_2}(q)$ for any $2+\frac{2\alpha}{\alpha-1}<r_2<\frac{\beta_2-2}{4}$; both $C_1$ and $C_2$ are independent of $m,n$.
\end{proposition}
\begin{proposition}\label{prop:samplecomplexityd=1}
	For $d=1$. Assume $(P,Q)$ are distributions on $\mathbb{R}^d$ of heavy-tail $(\beta_1,\beta_2)$, where $\beta_1>7$ and $\beta_2>13$. Suppose $\alpha$ satisfies $\frac{2\alpha}{\alpha-1}+4<\beta_1-1$ and $\frac{6\alpha}{\alpha-1}<\beta_2-7$, then if $m$ and $n$ are sufficiently large, we have
	\begin{equation}
		\E_{X,Y}\abs{D_{\alpha}^L(P_m\|Q_n)-D_{\alpha}^L(P\|Q)}\leq \frac{C_1}{m^{1/2}} + \frac{C_2}{n^{1/2}},
	\end{equation}
	where $C_1$ depends on $M_{2}(p)$ and $C_2$ depends on $M_{\frac{2\alpha}{\alpha-1}+4}(p)$, $M_{\frac{2\alpha}{\alpha-1}+4}(q)$ and $M_{dr_2}(q)$ for any $2+\frac{2\alpha}{\alpha-1}<r_2<\frac{\beta_2-1}{3}$; both $C_1$ and $C_2$ are independent of $m,n$.
\end{proposition}
\begin{proof}
	The only difference from the proof of \Cref{thm:samplecomplexity1} is that we need to bound the random metric entropy differently since $\sqrt{\ln\CN(\mathcal{F}_{\alpha,M},\delta,L_2(Q_n))}$ is no longer integrable at infinity, and the upper limit of the integral in \Cref{lemma:metric_entropy} cannot be relaxed to $\infty$. Instead, we have 
	\begin{align*}
		&\E_{X,Y}\inf_{0<\theta<M_Y}\left(4\theta+\frac{12}{\sqrt{n}}\int_{\theta}^{M_{Y}}\sqrt{\ln\CN(\mathcal{F}_{\alpha,M},\delta,L_2(Q_n))}\diff{\delta}\right)\\
		&\leq \E_{X,Y}\inf_{0<\theta<M_Y}\left(4\theta+\frac{12}{\sqrt{n}}K(M+L)^{\frac{d\alpha}{(\alpha-1)}}\Lambda_2^{d/2}\int_{\theta}^{M_Y}\delta^{-\frac{d}{2}}\diff{\delta}\right),
	\end{align*}
	where $M_Y = \sup_{\gamma\in\mathcal{F}_{\alpha,M}}\sqrt{\frac{1}{n}\sum_{j=1}^n \abs{\gamma(y_j)}^2}\leq\sqrt{\frac{1}{n}\sum_{j=1}^n(M+L+ L\norm{y_j})^2}$.
	
	For $d=2$, we have $\int_{\theta}^{M_Y}\delta^{-\frac{d}{2}}\diff{\delta}=\ln M_y - \ln \theta$, and we can pick $\theta = \frac{\ln n}{\sqrt{n}}$, and use the inequality $\ln M_y\leq M_y-1$ and combine it with \Cref{lemma:empirical_moments} and \Cref{lemma:supnorm_random} as in the proof of \Cref{thm:samplecomplexity1}.
	
	For $d=1$, we have $\int_{\theta}^{M_Y}\delta^{-\frac{d}{2}}\diff{\delta}=\frac{\sqrt{M_y}-\sqrt{\theta}}{2}$, and we can pick $\theta = \frac{1}{\sqrt{n}}$ to balance the two terms.
\end{proof}

\section{Proofs of results in \Cref{sec:group_symmetry}}\label{appendix:group}
\begin{proof}[Proof of \Cref{thm:samplecomplexity1_symmetry}]The proof is very similar to that of \Cref{thm:samplecomplexity1}, therefore we only outline the improvement we can obtain. First, same as the beginning of the proof of Theorem 4.8 in \cite{chen2023sample}, we can restrict the domain from $\CX$ to $\CX/G$ by invariance, so that we focus on Lipschitz functions on $\CX/G$. Indeed, we have
	\begin{align*}
		&\E_{X,Y}\abs{D_{\alpha}^{L,G}(P_m\|Q_n)-D_{\alpha}^L(P\|Q)}\\
		&=\E_{X,Y}\abs{\sup_{\substack{\gamma\in\text{Lip}_L^G(\CX) \\ M(\gamma)\leq\overline{M}_{m,n}}}\left\{\E_{P_m}[\gamma]-\E_{Q_n}[f_\alpha^*(\gamma)]\right\} - \sup_{\substack{\gamma\in\text{Lip}_L^G(\CX) \\ M(\gamma)\leq\overline{M}}}\left\{\E_P[\gamma]-\E_Q[f_\alpha^*(\gamma)]\right\}}\\
		&\leq \E_{X,Y}\sup_{\substack{\gamma\in\text{Lip}_L^G(\CX) \\ M(\gamma)\leq M}}\abs{\frac{1}{m}\sum_{i=1}^m\gamma(x_i)-\frac{1}{n}\sum_{j=1}^n f_\alpha^*[\gamma(y_j)]-\left(\E_P[\gamma]-\E_Q[f_\alpha^*(\gamma)]\right)}\\
		&= \E_{X,Y}\sup_{\substack{\gamma\in\text{Lip}_L^G(\CX) \\ M(\gamma)\leq M}}\abs{\frac{1}{m}\sum_{i=1}^m\gamma(T_G(x_i))-\frac{1}{n}\sum_{j=1}^n f_\alpha^*[\gamma(T_G(y_j))]-\left(\E_P[\gamma]-\E_Q[f_\alpha^*(\gamma)]\right)}\\
		&\leq \E_{X,Y}\sup_{\substack{\gamma\in\text{Lip}_L^G(\CX/G) \\ M(\gamma)\leq M}}\abs{\frac{1}{m}\sum_{i=1}^m\gamma(T_G(x_i))-\frac{1}{n}\sum_{j=1}^n f_\alpha^*[\gamma(T_G(y_j))]-\left(\E_{P_{\CX/G}}[\gamma]-\E_{Q_{\CX/G}}[f_\alpha^*(\gamma)]\right)}.
	\end{align*}
	where $T_G:\CX\to\CX/G$ is the quotient map, and $P_{\CX/G}, Q_{\CX/G}$ are restrictions of $P,Q$ on $\CX/G$ since both $P,Q$ are $G$-invariant, and $T_G(x_i)$ and $T_G(y_j)$ can be viewed as i.i.d. samples drawn from $P_{\CX/G}, Q_{\CX/G}$. Compared to the proof of \Cref{thm:samplecomplexity1}, we have some minor differences. First, in the sub-Weibull setting, the bound provided by Markov's inequality has Weibull-type decay in $k$, and we can simply choose $\hat{\beta}_2=1$. Therefore, the summation in bounding $\ln(\mathcal{F}_{\alpha,M},\delta,L_2(Q_n))$ is summable in $k$. Moreover, to bound $\ln(\mathcal{F}_{\alpha,M},\delta,L_2(Q_n))$, we have $\delta^{-d}$ improved to $\delta^{-d^*}$ due to the intrinsic dimension assumption. Due to \Cref{assumption:group} on the group and the partition $\Omega_k$'s are circular about the origin, the Lebesgue measure induces a reduction by a factor of $1/\abs{G}$ by working on $\CX/G$ compared to $\CX$, which then makes a reduction by $1/\abs{G}$ in the bound of $\ln(\mathcal{F}_{\alpha,M},\delta,L_2(Q_n))$, and it eventually contributes to the factor $\abs{G}$ in the bound in \Cref{thm:samplecomplexity1_symmetry}. On the other hand, we bound $\E_{X}\sup_{\gamma\in\text{Lip}_L(\mathbb{R}^d)}\abs{\E_P[\gamma]-\E_{P_m}[\gamma]}$ using the same procedure using metric entropy instead. Since the magnitude of Lipschitz functions grows slower than $\mathcal{F}_{\alpha,M}$, the procedure is straightforward. This finally creates a factor of $\abs{G}$ in front of $m$ in the final bound. For cases when the intrinsic dimension is 1 or 2, we can apply proofs of \Cref{prop:samplecomplexityd=2} and \Cref{prop:samplecomplexityd=1} after the above treatment.
\end{proof}

\begin{proof}[Proof of \Cref{thm:samplecomplexity1_symmetry_infinite}]
	Compared to the proof of that of \Cref{thm:samplecomplexity1_symmetry}, we do not need to make a factor $\abs{G}$. Instead, in bounding $\ln(\mathcal{F}_{\alpha,M},\delta,L_2(Q_n))$, we have $\delta^{-d^*}$ improved to $\delta^{-d^{**}}$ due to the intrinsic dimension assumption.
\end{proof}

\begin{proof}[Proof of \Cref{thm:samplecomplexity1_symmetry_W1} and \Cref{thm:samplecomplexity1_symmetry_infinite_W1}]
	Since the variational form of $W_1$ is shift-invariant to $\gamma\in\text{Lip}_L(\R^d)$, we can always assume $\gamma(0)=0$. Thus, \Cref{lemma:supnorm} and \Cref{lemma:supnorm_random} are not useful. Compared to the proof of \Cref{thm:samplecomplexity1}, we can pick $\hat{\beta}_2 = 1$ and $M$ can be set to 0. Finally, it is the limiting case of $\alpha\to\infty$.
\end{proof}

\section{Values of different training objectives in \Cref{sec:highdimension}\label{append:divergence_value}}
\paragraph{Training objective function values for Lipschitz-regularized and standard $\alpha$-divergence} 

\Cref{tab:divergence:values} summarizes the training objective (divergence) values for the standard and Lipschitz-regularized $\alpha$-divergence in different GPA and GAN models in the experiment \Cref{sec:highdimension}.

\begin{table}[h]
	\centering
	\begin{tabular}{c|c}
		Model & Objective function (divergence) value \\ \hline
		Lip-$\alpha$ GPA & $0.0129187$ \\
		$\alpha$ GPA & $3.0554032e+26$ \\
		Lip-$\alpha$ GAN &  $0.358602$ \\
		$\alpha$ GAN & $3.25298e+7$ 		
	\end{tabular}
	\caption{Final values of the training objective for GANs and GPAs under Lipschitz-regularized ($L=1$) and standard $\alpha$-divergences with $\alpha=2$. See the example in \Cref{sec:highdimension}.\label{tab:divergence:values}}
\end{table}

%%%%%%%%%%%%%%%%%%%%%%%%%%%%%%%%%%%%%%%%%%%%%%%%%%%%%%%%%%%%
\clearpage
\noindent\textbf{Supplementary Material}
\setcounter{page}{1}
\setcounter{section}{0}
\section{Computational details}\label{append:resource}
Our experiment is computed using personal computer in the environment:  \texttt{Apple M2 8 cores} and \texttt{Apple M2 24 GB - Metal 3}.

\paragraph{Access and preprocess for the real-world Keystroke example}
We downloaded a dataset provided by \cite{interarrivaltime_heavytailedexample} which contains the keystroke times of 20 different individuals from the same script with a total character count equal 717. For each individual, we obtained 716 samples of inter-arrival time between keystrokes by subtracting \texttt{release time} of first to the second last characters in the script from \texttt{press time} of second to last characters in the script. Assuming that the random variable is i.i.d., we gathered samples from 10 individuals with examinee \texttt{id}s corresponding to 27252, 36718, 56281, 64663, 67159, 97737, 145007, 159915, 264420, 271802. In total, we use 7160 samples from inter-arrival time between keystrokes as target samples.

\paragraph{Neural network architectures and hyper-parameters}
Full batch is used for training different models. \texttt{batchsize} equals 10000 for 2D student-t example and 7160 for Keystroke example. Neural network architectures of the models are specified in \Cref{table:nn architecture gpa,table:nn architecture gan,table:nn architecture sgm,table:nn architecture otflow}.

\begin{table}[h]
	
	\begin{minipage}{.49\linewidth}
		\centering
		\begin{tabular}{c}
			
			\hline
			Lip-$\alpha$-GPA \\
			\hline
			$W^1 \in \mathbb{R}^{\ell_1 \times d}$, $b^1 \in \mathbb{R}^{\ell_1}$ \\
			Spectral Normalization on $W^1$ \\
			ReLU\\
			\hline
			$W^2 \in \mathbb{R}^{\ell_2 \times \ell_1}$, $b^2 \in \mathbb{R}^{\ell_2}$ \\
			Spectral Normalization on $W^2$ \\
			ReLU\\
			\hline
			$W^3 \in \mathbb{R}^{\ell_3 \times \ell_2}$, $b^3 \in \mathbb{R}^{\ell_3}$ \\
			Spectral Normalization on $W^3$ \\
			ReLU\\
			\hline
			$W^4 \in \mathbb{R}^{1 \times \ell_3}$, $b^4 \in \mathbb{R}$ \\
			Spectral Normalization on $W^4$ \\
			Linear\\
			\hline
		\end{tabular}%\subcaption{2D image data (MNIST)}
		\medskip
	\end{minipage}
	\begin{minipage}{.49\linewidth}
		\centering
		\begin{tabular}{c}
			
			\hline
			$\alpha$-GPA \\
			\hline
			$W^1 \in \mathbb{R}^{\ell_1 \times d}$, $b^1 \in \mathbb{R}^{\ell_1}$ \\
			ReLU\\
			\hline
			$W^2 \in \mathbb{R}^{\ell_2 \times \ell_1}$, $b^2 \in \mathbb{R}^{\ell_2}$ \\
			ReLU\\
			\hline
			$W^3 \in \mathbb{R}^{\ell_3 \times \ell_2}$, $b^3 \in \mathbb{R}^{\ell_3}$ \\
			ReLU\\
			\hline
			$W^4 \in \mathbb{R}^{1 \times \ell_3}$, $b^4 \in \mathbb{R}$ \\
			Linear\\
			\hline
		\end{tabular}%\subcaption{2D image data (MNIST)}
		\medskip
	\end{minipage}
	
	\caption{Neural network architectures of GPA discriminator $\gamma: \mathbb{R}^d \rightarrow \mathbb{R}$. (Lipschitz regularized) $\alpha$-divergences are estimated by variational representation formula \eqref{eq:variational_formula}. In particular, Spectral normalization \cite{miyato2018sn} technique is used for ensuring the Lipschitz continuity. For 2D student-t example, Keystroke example and Lorenz63 example, $\ell = [32, 32, 32]$. For heavy-tail data embedded in $\mathbb{R}^{110}$ example, $\ell = [128, 128, 128]$. }
	\label{table:nn architecture gpa}
\end{table}

\begin{table}[h]
	
	\begin{minipage}{.49\linewidth}
		\centering
		\begin{tabular}{c}
			
			\hline
			discriminator $\gamma$ \\
			\hline
			$W^1 \in \mathbb{R}^{\ell_1 \times d}$, $b^1 \in \mathbb{R}^{\ell_1}$ \\
			ReLU\\
			\hline
			$W^2 \in \mathbb{R}^{\ell_2 \times \ell_1}$, $b^2 \in \mathbb{R}^{\ell_2}$ \\
			ReLU\\
			\hline
			$W^3 \in \mathbb{R}^{\ell_3 \times \ell_2}$, $b^3 \in \mathbb{R}^{\ell_3}$ \\
			ReLU\\
			\hline
			$W^4 \in \mathbb{R}^{1 \times \ell_3}$, $b^4 \in \mathbb{R}$ \\
			Linear\\
			\hline
		\end{tabular}%\subcaption{2D image data (MNIST)}
		\medskip
	\end{minipage}
	\begin{minipage}{.49\linewidth}
		\centering
		\begin{tabular}{c}
			
			\hline
			generator $g$ \\
			\hline
			$W^1 \in \mathbb{R}^{\ell_1 \times d}$, $b^1 \in \mathbb{R}^{\ell_1}$ \\
			ReLU\\
			\hline
			$W^2 \in \mathbb{R}^{\ell_2 \times \ell_1}$, $b^2 \in \mathbb{R}^{\ell_2}$ \\
			ReLU\\
			\hline
			$W^3 \in \mathbb{R}^{\ell_3 \times \ell_2}$, $b^3 \in \mathbb{R}^{\ell_3}$ \\
			ReLU\\
			\hline
			$W^4 \in \mathbb{R}^{d \times \ell_3}$, $b^4 \in \mathbb{R}^d$ \\
			Linear\\
			\hline
		\end{tabular}%\subcaption{2D image data (MNIST)}
		\medskip
	\end{minipage}
	
	\caption{Neural network architectures of GAN discriminator $\gamma: \mathbb{R}^d \rightarrow \mathbb{R}$ and generator $g: \mathbb{R}^{d'} \rightarrow \mathbb{R}^d$. (Lipschitz regularized) $\alpha$-divergences are estimated by variational representation formula \eqref{eq:variational_formula}. In particular, gradient penalty \cite{arjovsky2017wgan} technique is used for ensuring the Lipschitz continuity. For 2D student-t example, Keystroke example and Lorenz63 example, $\ell = [64, 32, 16]$. For heavy-tail data embedded in $\mathbb{R}^{110}$ example, $\ell = [128, 128, 128]$. }
	\label{table:nn architecture gan}
\end{table}

\begin{table}[h]
	\centering
	\begin{tabular}{c}
		
		\hline
		score function $s$ \\
		\hline
		$W^1 \in \mathbb{R}^{\ell_1 \times (d+1)}$, $b^1 \in \mathbb{R}^{\ell_1}$ \\
		GeLU\\
		\hline
		$W^2 \in \mathbb{R}^{\ell_2 \times \ell_1}$, $b^2 \in \mathbb{R}^{\ell_2}$ \\
		GeLU\\
		\hline
		$W^3 \in \mathbb{R}^{\ell_3 \times \ell_2}$, $b^3 \in \mathbb{R}^{\ell_3}$ \\
		GeLU\\
		\hline
		$W^4 \in \mathbb{R}^{\ell_4 \times \ell_3}$, $b^4 \in \mathbb{R}^{\ell_4}$ \\
		GeLU\\
		\hline
		$W^5 \in \mathbb{R}^{\ell_5 \times \ell_4}$, $b^5 \in \mathbb{R}^{\ell_5}$ \\
		GeLU\\
		\hline
		$W^6 \in \mathbb{R}^{\ell_6 \times \ell_5}$, $b^6 \in \mathbb{R}^{\ell_6}$ \\
		GeLU\\
		\hline
		$W^7 \in \mathbb{R}^{\ell_7 \times \ell_6}$, $b^7 \in \mathbb{R}^{\ell_7}$ \\
		GeLU\\
		\hline
		$W^8 \in \mathbb{R}^{\ell_8 \times \ell_7}$, $b^8 \in \mathbb{R}^{\ell_8}$ \\
		GeLU\\
		\hline
		$W^9 \in \mathbb{R}^{d \times \ell_8}$, $b^9 \in \mathbb{R}^d$ \\
		Linear\\
		\hline
	\end{tabular}%\subcaption{2D image data (MNIST)}
	
	\caption{Neural network architectures of time-dependent score function $s(x,t)$ with $s: \mathbb{R}^{(d+1)} \rightarrow \mathbb{R}^d$. For 2D student-t example, Keystroke example with dimension $d=1,2$,  $\ell = [32, 32, d+1, 32, 32, d+1, 32, 32]$. For Lorenz63 example with dimension $d=3$, $\ell = [64, 64, d+1, 64, 64, d+1, 64, 64]$. For heavy-tail data embedded in $\mathbb{R}^{110}$ example with dimension $d=110$, $\ell = [128, 128, d+1, 128, 128, d+1, 128, 128]$.}
	\label{table:nn architecture sgm}
\end{table}

\begin{table}[h]
	\caption{Time-dependent potential function $\phi(y)$ with $y = (x,t)$ for OT flow $\phi: \mathbb{R}^{(d+1)} \rightarrow \mathbb{R}$ consists of a nonlinear neural network part $N(y; \theta_N)$ and a quadratic part $\frac{1}{2} y^T (A^T A)y + b^T y + c$, i.e. $\phi(s; \theta) = N(s; \theta_N) + \frac{1}{2} y^T (A^T A)y + b^T y + c$ where $\theta = (\theta_N, A, b, c)$. In particular, the neural network $N(y;\theta_N)$ has a resnet structure ($y^{l+1} = y^{l} + activation(W^L y^l + b^l)$). 
		For low dimensional examples with $d=1,2$,  $\ell = 32$.\label{table:nn architecture otflow}}
	\centering
	\begin{tabular}{c}
		
		\hline
		$N(y;\theta_N)$ \\
		\hline
		$W^1 \in \mathbb{R}^{\ell \times (d+1)}$, $b^1 \in \mathbb{R}^{\ell}$ \\
		Tanh\\
		\hline
		Resnet with $W^2 \in \mathbb{R}^{\ell \times \ell}$, $b^2 \in \mathbb{R}^{\ell_2}$ \\
		Tanh\\
		\hline
		$W^3 \in \mathbb{R}^{1 \times \ell}$, $b^2 \in \mathbb{R}$ \\
		Linear\\
		\hline
	\end{tabular}%\subcaption{2D image data (MNIST)}
\end{table}

\section{Additional experiments}
\label{appendix:sec:additional:experiments}

\begin{figure}[h]
	\centering
	\begin{subfigure}{\linewidth}
		\centering
		\includegraphics[width=.23\linewidth]{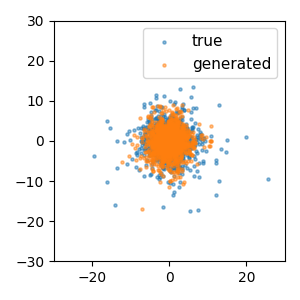}
		\includegraphics[width=.23\linewidth]{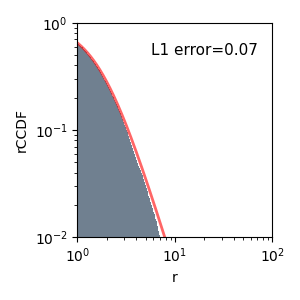}
		\includegraphics[width=.23\linewidth]{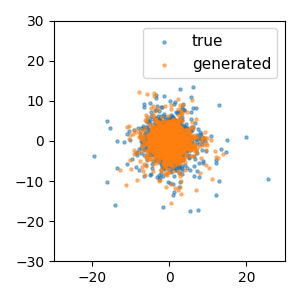}
		\includegraphics[width=.23\linewidth]{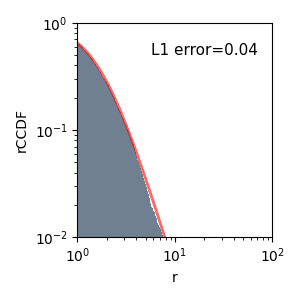}
		\caption{$\alpha$-GAN (left) and its counterpart with Lip-proximal regularization, Lip-$\alpha$-GAN  (right) }
	\end{subfigure}
	
	\begin{subfigure}{\linewidth}
		\centering
		\includegraphics[width=.23\linewidth]{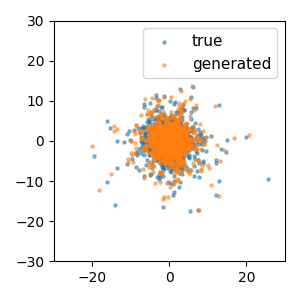}
		\includegraphics[width=.23\linewidth]{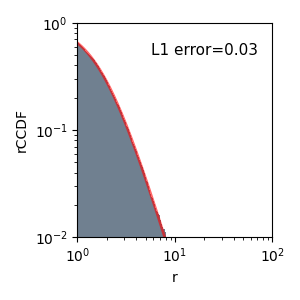}
		\includegraphics[width=.23\linewidth]{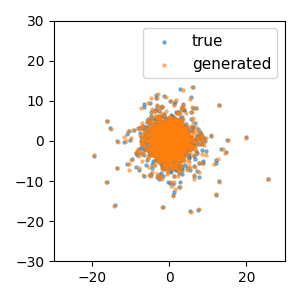}
		\includegraphics[width=.23\linewidth]{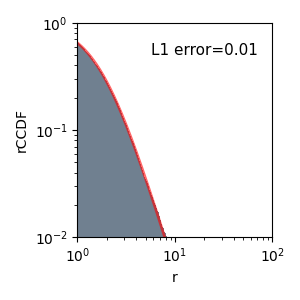}
		\caption{$\alpha$-GPA (left) and its counterpart with $W_1$-proximal regularization, Lip-$\alpha$-GPA (right) }
	\end{subfigure}

	\caption{Learning a 2D isotropic Student-t with degree of freedom $\nu=3$ (tail index $\beta=5.0$) using generative models based on Lipschitz-regularized $\alpha$-divergences with $\alpha=2$.
		Models with $W_1$-proximal regularizations (right) learn the heavy-tailed distribution significantly better than those without (left). See \Cref{subsec:explanation:generative:models} for detailed explanations of the models.}
	\label{fig:student-t:3.0:alpha:proximal:divergence}
\end{figure}

\begin{figure}[h]
	\centering
	\begin{subfigure}{.48\linewidth}
		\centering
		\includegraphics[width=.48\linewidth]{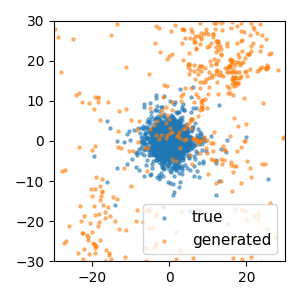}
		\includegraphics[width=.48\linewidth]{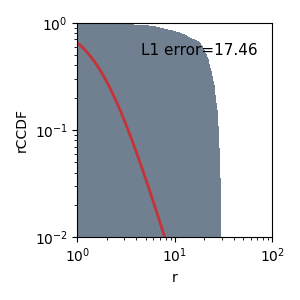}
		\caption{CNF}
	\end{subfigure}
	
	\begin{subfigure}{.48\linewidth}
		\centering
		\includegraphics[width=.48\linewidth]{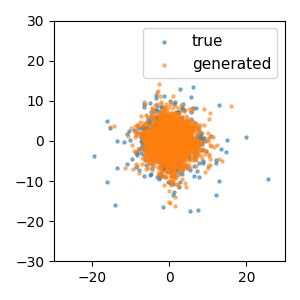}
		\includegraphics[width=.48\linewidth]{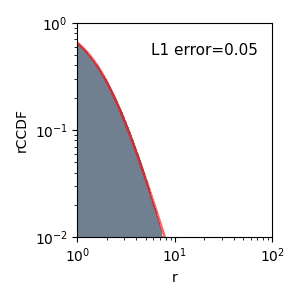}
		\caption{OT flow}
	\end{subfigure}
	\begin{subfigure}{.48\linewidth}
		\centering
		\includegraphics[width=.48\linewidth]{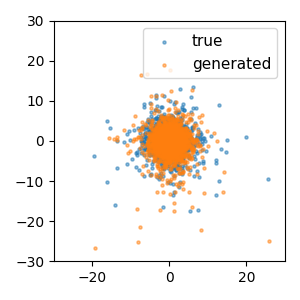}
		\includegraphics[width=.48\linewidth]{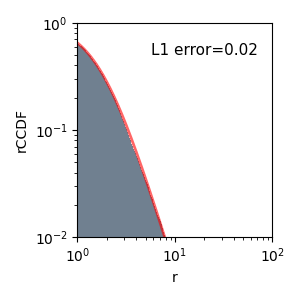}
		\caption{VE-SGM}
	\end{subfigure}
	\caption{Learning a 2D isotropic Student-t with degree of freedom $\nu=3$ (tail index $\beta=5.0$) using generative models based on $W_2$-$\alpha$-divergences with  $\alpha=1$.
		Models with $W_2$-proximal regularizations, (b) and (c), learn the heavy-tailed distribution significantly better than that without, (a). See \Cref{subsec:explanation:generative:models} for detailed explanations of the models.
	}
	\label{fig:student-t:3.0:proximal:losses}
\end{figure}

\begin{figure}[h]
	\centering
	\begin{subfigure}{\linewidth}
		\includegraphics[width=0.24\linewidth]{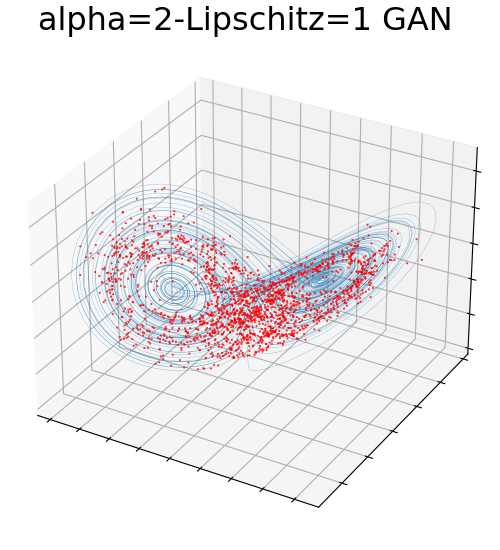}
		\includegraphics[width=0.24\linewidth]{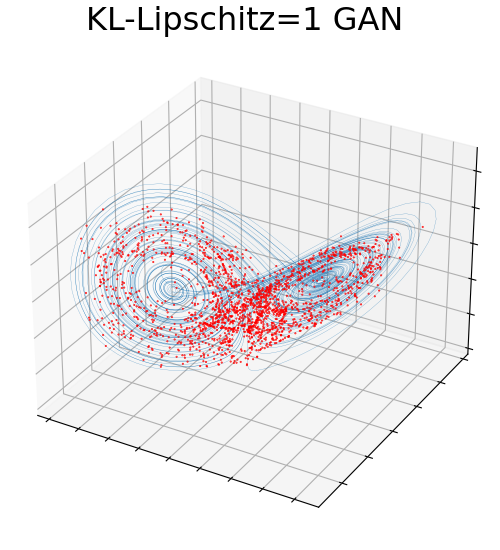}
		\includegraphics[width=0.24\linewidth]{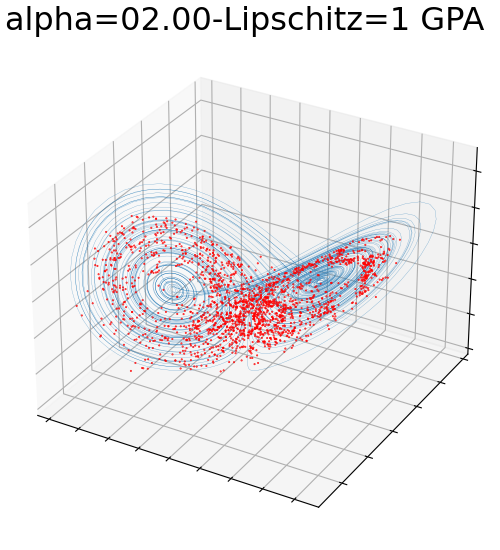}
		\includegraphics[width=0.24\linewidth]{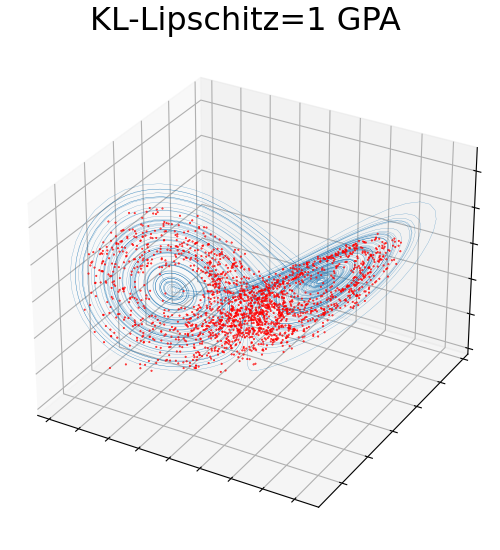}
		\caption{Generated samples from generative models with Lipschitz-regularized $\alpha$-divergences as learning objectives.
			$\alpha=2$-Lipschitz-1 GAN  (first), KL-Lipschitz-1 GAN (second), $\alpha=2$-Lipschitz-1 GPA  (third), KL-Lipschitz-1 GPA (fourth) }
	\end{subfigure}
	\begin{subfigure}{\linewidth}
		
		\includegraphics[width=0.24\linewidth]{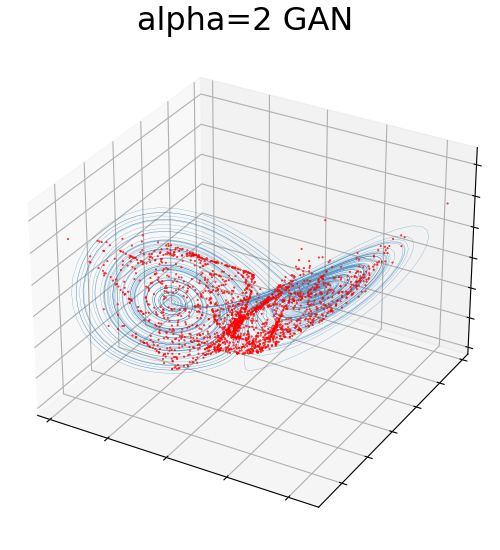}
		\includegraphics[width=0.24\linewidth]{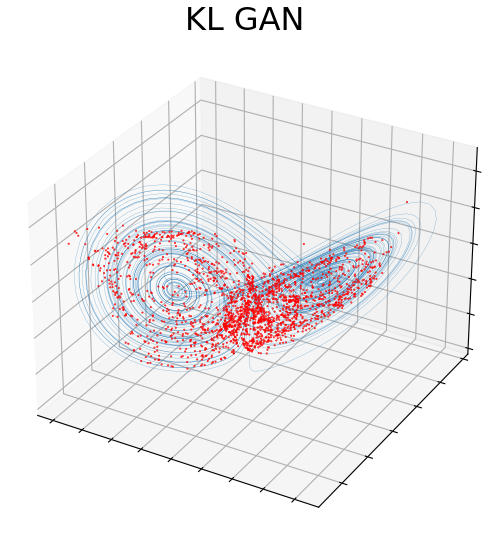}
		\includegraphics[width=0.24\linewidth]{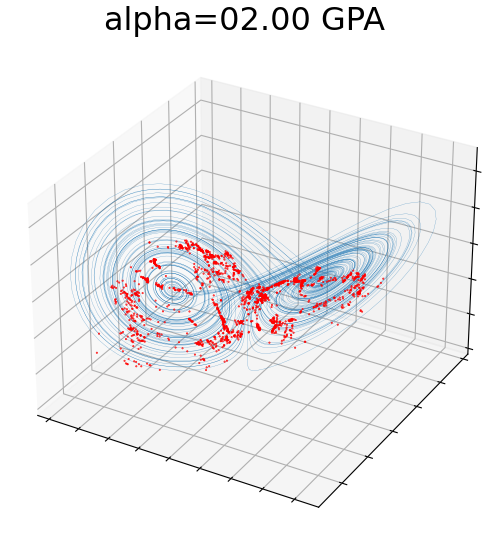}
		\includegraphics[width=0.24\linewidth]{figures/KL-Lipschitz=1_GPA__1000_samples__Lorenz63-figure.png}
		\caption{Generated samples from generative models with un-regularized $\alpha$-divergences as learning objectives. $\alpha=2$ GAN (first), KL GAN (second), $\alpha=2$ GPA (third), KL GPA (fourth)}
	\end{subfigure}
	\begin{subfigure}{\linewidth}
		\centering
		\includegraphics[width=0.24\linewidth]{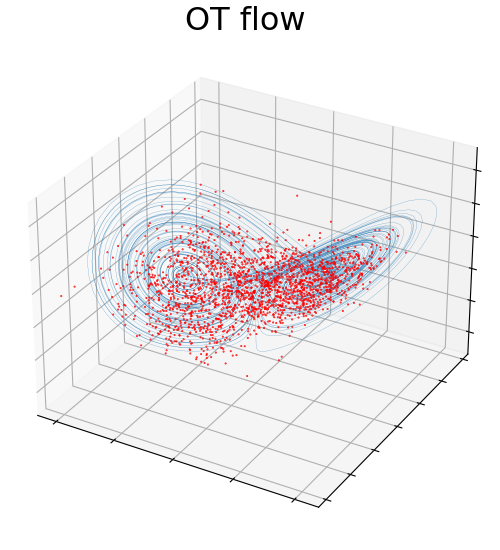}
		\includegraphics[width=0.24\linewidth]{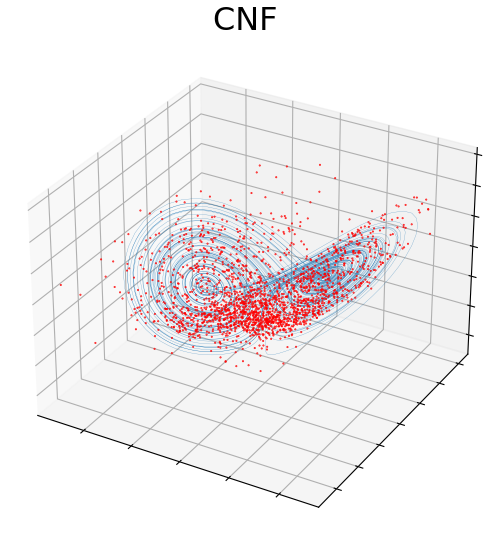}
		\includegraphics[width=0.24\linewidth]{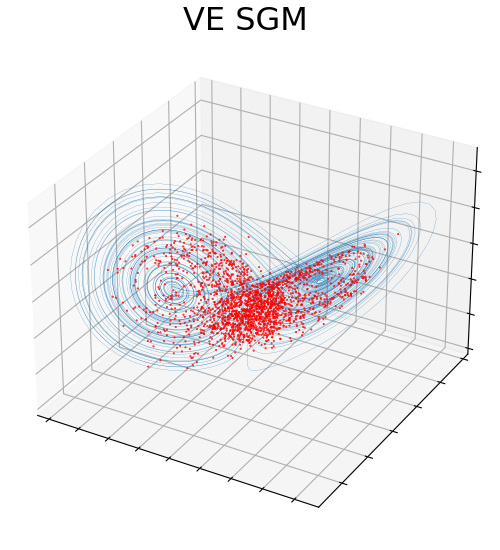}
		
		\caption{Generated samples from generative models with different learning objectives. 
			OT-flow: $W_2$-reverse KL divergence (left), CNF: reverse KL divergence (center), VE-SGM: $W_2$-proximal regularized cross-entropy (right)}
	\end{subfigure}
	
	\caption{Sample generation from Lorenz 63 strange attractor. 1000 training data are used and 2000 samples are generated.}
	\label{fig:lorenz63:1000samples}
	
\end{figure}
\end{document}